%% file: camera_ready.tex
\documentclass{article}

\PassOptionsToPackage{numbers, compress}{natbib}
\usepackage[main, final]{neurips_2026} 

\usepackage[utf8]{inputenc} 
\usepackage[T1]{fontenc}    
\usepackage{hyperref}       
\usepackage{url}            
\usepackage{booktabs}       
\usepackage{amsfonts}       
\usepackage{nicefrac}       
\usepackage{microtype}      
\usepackage{multirow}       
\usepackage{xcolor}         
\usepackage{xurl}
\usepackage{amsmath}
\usepackage{amssymb}
\usepackage{amsthm}         
\usepackage{dsfont}
\usepackage[most]{tcolorbox}   
\usepackage{enumitem}
\usepackage{graphicx}
\usepackage{caption}
\usepackage{subcaption}
\usepackage{algorithm}
\usepackage{algpseudocode}

\newtheorem{theorem}{Theorem}
\newtheorem{proposition}{Proposition}

\newtheorem{lemma}{Lemma}
\newtheorem{corollary}{Corollary}
\newtheorem{remark}{Remark}

\newcounter{restatedtheorem}
\newenvironment{restatethm}[2][]{
  \begingroup
  \stepcounter{restatedtheorem}
  \renewcommand{\thetheorem}{\ref{#2}}
  
  \begin{theorem}[#1]
  \addtocounter{theorem}{-1}
}{
  \end{theorem}
  \endgroup
}

\title{Budgeted Multi-Source Counterfactual Annotation \\ for Off-Policy Evaluation}

\author{
\textbf{Biao Xiang}$^{1}$\thanks{Corresponding author.
Email: \texttt{bqx5053@psu.edu}} \qquad
\textbf{Ali Eshragh}$^{2,3}$\thanks{Equal supervision.} \qquad
\textbf{Yuexing Li}$^{2 \dagger}$ \qquad
\textbf{Kai Wang}$^{4\dagger}$ \\
[0.1em]
$^1$Pennsylvania State University, State College, PA \\
$^2$Johns Hopkins Carey Business School, Washington, DC \\
$^3$International Computer Science Institute, Berkeley, CA \\
$^4$Georgia Institute of Technology, Atlanta, GA \\
}

\begin{document}

\maketitle

\begin{abstract}
Off-policy evaluation (OPE) estimates the value of a target policy from logged data, but limited behavior-policy coverage can force high-variance reweighting or reward-model extrapolation.
Counterfactual annotations can add evidence about unobserved actions, yet practical sources, including domain experts and large language models (LLMs), may be costly, biased, or noisy.
We study budgeted acquisition of such annotations for contextual-bandit OPE.
Given source-specific costs and error profiles, we formulate an integer allocation problem over context-action pairs and annotation sources to minimize the component of estimator variance that depends on the annotation plan.
We characterize when annotations are valuable through a first-annotation threshold and local annotation-value regimes.
For the coupled multi-source problem, we develop a majorization-minimization algorithm with dynamic-programming subroutines that monotonically improves the objective.
Experiments in synthetic clinical and LLM-annotated education bandits show that our allocation method reduces fixed-profile mean squared error (MSE) by $20.58\%$ and $10.77\%$, respectively, relative to no annotation.
\end{abstract}

\section{Introduction}
\label{sec:introduction}

Off-policy evaluation (OPE) estimates the value of a target policy from data collected by a different behavior policy.
It is especially important when online evaluation is expensive, disruptive, or risky, as in healthcare and education.
In contextual bandits, logged data reveal rewards only for actions selected by the behavior policy.
When the target policy assigns probability to actions that are poorly covered in the logged data, evaluation can depend on high-variance reweighting, reward-model extrapolation, or both.

Counterfactual annotations offer a practical way to add evidence about unobserved actions without deploying the target policy.
For example, a clinician may assess what would likely have happened under an alternative treatment, or an automated tutor may estimate a student's response under a different instructional activity.
Recent work has shown how such annotations can be incorporated into OPE estimators~\citep{tang2023counterfactual,mandyam2024}.
However, a key practical question remains: when annotations are costly and imperfect, which ones should be collected?

This question matters because annotation sources vary in cost and reliability.
Domain experts may be accurate but expensive, while large language models (LLMs) may scale cheaply yet make systematic errors on particular context-action pairs.
In the counterfactual-annotated estimator we study, such errors do not introduce bias into the policy-value estimate but affect variance through the reward model.
Annotation value therefore depends jointly on source quality, cost, factual coverage, and how strongly the target policy weights the annotated action.

We formulate counterfactual-annotation acquisition as a budget-constrained allocation problem: given estimates of each source's cost, bias, and variability, choose which context-action pairs to annotate and which sources to query so as to minimize the part of the estimation error that depends on the annotation plan.
The problem is discrete because annotations are acquired one at a time, and coupled because annotations for different actions in the same context jointly shape the reward model used for evaluation.
These challenges require both a statistical understanding of when annotations are valuable and an optimization method for allocating them under budget.
To address them, we make three contributions:
\begin{enumerate}[label=(\Roman*),leftmargin=1.5em,itemsep=2pt,topsep=2pt]
    \item \textbf{Annotation-value regimes.}
    We derive a first-annotation threshold and a four-regime characterization of when annotations help immediately, help only up to a finite amount, help only after acquiring a sufficiently large batch, or do not help locally.

    \item \textbf{Budgeted allocation algorithm.}
    We show that the allocation problem is coupled, nonconvex, and not generally submodular, and develop a majorization-minimization (MM) algorithm with dynamic-programming subroutines that monotonically improves the original objective.

    \item \textbf{Empirical validation.}
    Across synthetic and LLM-annotated bandit environments, we compare against greedy allocation and certified solver benchmarks, quantify profile-estimation sensitivity, and validate the final OPE estimator using finite pilots and independent evaluation samples.
\end{enumerate}

\section{Related Work}
\label{sec:related_work}

\looseness=1
Off-policy evaluation (OPE) estimates the value of a target policy from data generated by a behavior policy.
Core estimators include importance sampling and per-decision variants~\citep{precup2000}, direct-method estimators that learn reward, value, or transition models~\citep{beygelzimer2009offset,le2019batch}, and doubly robust (DR) estimators that combine model estimates with importance-weighted corrections~\citep{dudik2011doubly,dudik2014doubly,jiang2016doubly}.
Variance reduction and reliability have been further studied through data-efficient DR estimators~\citep{thomas2016data,farajtabar2018more}, representation-based OPE~\citep{liu2018representation}, empirical benchmarking~\citep{voloshin2021empirical}, and confidence or estimator-selection methods~\citep{thomas2015hcope,taufiq2022conformal,foffano2023conformal,nie2024opera}.
Our work builds on this OPE foundation, but addresses the acquisition stage: which counterfactual annotations to collect before \mbox{applying the estimator}.

A related stream augments logged data with counterfactual or auxiliary information.
Structural-causal methods construct counterfactual trajectories~\citep{buesing2019woulda,oberst2019counterfactual}; semi-offline OPE reweights human counterfactual annotations to preserve the logged context distribution~\citep{tang2023counterfactual}; and CANDOR incorporates imperfect annotations into DR contextual-bandit estimators, particularly through the reward model~\citep{mandyam2024}.
Auxiliary-data methods such as trajectory augmentation, PERRY, and APRIL use generated trajectories or LLM-derived predictions for inference and policy evaluation~\citep{angelopoulos2023prediction,gao2024trajectory,mandyam2025perry,mandyam2025april}.
These works study how to use auxiliary information once available; we study how to allocate a finite annotation budget across heterogeneous sources.

Our problem is also related to budgeted data acquisition: active learning and optimal design choose measurements under resource constraints~\citep{settles2009active,pukelsheim2006optimal}, while multiple-logger OPE combines logged data from different behavior policies~\citep{agarwal2017effective}.
Our setting differs because the acquired data are counterfactual rewards used in a DR reward model, rather than supervised labels or additional logged actions.
Thus, annotation value is not standalone informativeness; it is defined by how the annotation changes the resulting OPE estimator's variance under the target and behavior policies.

\section{Problem Formulation}
\label{sec:problem}

We study off-policy evaluation with counterfactual annotations in a contextual bandit with finite context space $\mathcal{S}$ and action space $\mathcal{A}$.
We use separate factual datasets for reward-model fitting and OPE: $\mathcal D_{\mathrm{fit}}=\{(\tilde s_j,\tilde a_j,\tilde r_j)\}_{j=1}^{M}$ and $\mathcal D_{\mathrm{ope}}=\{(s_i,a_i,r_i)\}_{i=1}^{N}$.
For each record in either dataset, the context $s\in\mathcal{S}$ is drawn from a distribution $d_0$, the action $a\in\mathcal{A}$ is chosen by a behavior policy $\pi_b(\cdot|s)$, and the reward $r$ is drawn from the reward distribution $R(s,a)$.
We write $\mu_r(s,a)=\mathbb{E}[r|s,a]$ and $\sigma_r^2(s,a)=\operatorname{Var}(r|s,a)$ for the true mean and variance of rewards for context-action pair $(s,a)$.

In addition to the logged data, annotations may be obtained from $K$ independent sources.
For each pair $(s,a)$ and source $k\in\mathcal{K}=\{1,\ldots,K\}$, let $n_k(s,a)$ be the number of collected annotations, and let $g_{k,j}\sim G_k(s,a)$, $j=1,\ldots,n_k(s,a)$, denote independent annotation draws from that source.
Denote $\mathbb{E}[g_{k,j}|s,a]=\mu_r(s,a)+\varepsilon_{g_k}(s,a)$ and $\operatorname{Var}(g_{k,j}|s,a)=\sigma_r^2(s,a)+\Delta_{g_k}(s,a)$, where $\varepsilon_{g_k}$ denotes annotation bias and $\Delta_{g_k}$ denotes excess annotation variance for source $k$.
The fitting data and acquired annotations are independent of the OPE evaluation sample.

Because annotations may be biased, we leverage the idea of the $\mathrm{DM}^+\text{-}\mathrm{IS}$ estimator \citep{mandyam2024} by incorporating annotations only in the reward estimate, rather than treating them as factual rewards in the importance-sampling correction.
For each pair $(s,a)$, let $n_f(s,a)=\sum_{j=1}^{M}\mathds{1}\{\tilde s_j=s,\tilde a_j=a\}$ denote the number of fitting observations for that pair in $\mathcal D_{\mathrm{fit}}$, and let $N(s,a)=n_f(s,a)+\sum_{k=1}^K n_k(s,a)$ denote the total number of factual and annotated reward observations available for that pair.
We define the annotation-augmented reward estimate, for $N(s,a) > 0$, as
\[
    \hat r^+(s,a)
    =
    \frac{
        \sum_{j=1}^{M}\mathds{1}\{\tilde s_j=s,\tilde a_j=a\}\tilde r_j
        +
        \sum_{k=1}^{K}\sum_{j=1}^{n_k(s,a)}g_{k,j}
    }{
        N(s,a)
    } .
\]

For $N(s,a)=0$, use a fixed default $r_0(s,a)$, this case is treated separately in Appendix~\ref{sec:nf_zero}.

Thus, for a target policy $\pi_e$ with standard common support, i.e., $\pi_b(a|s)>0$ whenever $\pi_e(a|s)>0$ for all $s\in\mathcal{S}, a\in\mathcal{A}$, the $\mathrm{DM}^+\text{-}\mathrm{IS}$ estimator is:
\begin{align}
\label{eq:multi-source_OPE_estimator}
    \hat v^{\mathrm{DM}^+\text{-}\mathrm{IS}}(\pi_e)
    =
    \frac{1}{N}\sum_{i=1}^{N}
    \left[
        \hat r^+(s_i,\pi_e)
        +
        \frac{\pi_e(a_i| s_i)}{\pi_b(a_i| s_i)}\bigl(r_i-\hat r^+(s_i,a_i)\bigr)
    \right],
\end{align}
where $\hat r^+(s,\pi_e)=\sum_{a\in\mathcal{A}}\pi_e(a|s)\hat r^+(s,a)$ is the estimated value of $\pi_e$ at context $s$.
Our goal is to minimize the mean squared error (MSE) of $\hat v^{\mathrm{DM}^+\text{-}\mathrm{IS}}(\pi_e)$ by choosing $\mathbf n$, i.e., how many annotations should be queried from each source for each context-action pair.

\section{Objective Function and Budget-Constrained Allocation}
To express how allocation affects MSE, we define reward-model bias and variance for $N(s,a)>0$ as
\[
\begin{aligned}
B(s,a;\mathbf n)
&=\frac{\sum_k n_k(s,a)\varepsilon_{g_k}(s,a)}{N(s,a)},\\
W(s,a;\mathbf n)
&=\frac{n_f(s,a)\sigma_r^2(s,a)+\sum_k n_k(s,a)\bigl[\sigma_r^2(s,a)+\Delta_{g_k}(s,a)\bigr]}{N(s,a)^2}.
\end{aligned}
\]
For $N(s,a)=0$, set $B(s,a;\mathbf n)=r_0(s,a)-\mu_r(s,a)$ and $W(s,a;\mathbf n)=0$.
Write $\mathcal B_s(\mathbf n)=\sum_a\pi_e(a\mid s)B(s,a;\mathbf n)$, $\eta(s,a)=1-\pi_b(a\mid s)$, $\alpha_1(s,a)=d_0(s)\pi_e(a\mid s)^2/\pi_b(a\mid s)$, and $\alpha_3(s,a)=\eta(s,a)\alpha_1(s,a)$, with $\alpha_1=\alpha_3=0$ when $\pi_e(a\mid s)=0$.
The following result isolates the part of OPE error controlled by the annotation plan.
\begin{theorem}[Unbiasedness and allocation-dependent MSE]
\label{thm:objective}
Under the independent-sample design above, known behavior propensities, common support, and finite second moments, fix the factual counts and a feasible allocation $\mathbf n$.
The estimator in~\eqref{eq:multi-source_OPE_estimator} is unbiased, and
\[
\operatorname{MSE}(\mathbf n)=C_0+\frac{J(\mathbf n)}{N},
\]
where $C_0$ is allocation-invariant and
\begin{equation}
\label{eq:final_objective}
\begin{aligned}
J(\mathbf n)&=J_1(\mathbf n)-J_2(\mathbf n)+J_3(\mathbf n),\\
J_1(\mathbf n)&=\sum_{s,a}\alpha_1(s,a)B(s,a;\mathbf n)^2,\\
J_2(\mathbf n)&=\sum_s d_0(s)\mathcal B_s(\mathbf n)^2,\\
J_3(\mathbf n)&=\sum_{s,a}\alpha_3(s,a)W(s,a;\mathbf n).
\end{aligned}
\end{equation}
\end{theorem}

Although the OPE estimator remains unbiased, annotation bias affects its variance through $J_1-J_2$, alongside reward-model variance $J_3$.
Because $C_0$ is allocation-invariant, minimizing $J$ exactly minimizes fixed-profile MSE.
For reporting, we distinguish $J$ from scaled total MSE, $\mathcal V(\mathbf n):=NC_0+J(\mathbf n) =N\operatorname{MSE}(\mathbf n)$.
The unbiasedness result extends \citet{mandyam2024} to multiple sources; Appendix~\ref{sec:appendix_derivations} gives the proof and $C_0$.

Let $c_k>0$ be the cost of obtaining one annotation from source $k$, and let $\kappa>0$ be the total annotation budget.
To focus on how the budget should be distributed across sources and context-action pairs, we consider the setting where each factual sample receives at most one counterfactual annotation.
Let $N_s:=\sum_{a\in\mathcal A}n_f(s,a)$ denote the number of factual samples with context $s$.
The resulting budget-constrained allocation problem is

\begin{equation}
\label{eq:optimization_problem}
\begin{aligned}
\min_{\mathbf n}\quad
& J(\mathbf n) \\
\text{s.t.}\quad
& \sum_{s\in\mathcal S}\sum_{a\in\mathcal A}\sum_{k\in\mathcal K} c_k n_k(s,a) \le \kappa, \\
& \sum_{a\in\mathcal A}\sum_{k\in\mathcal K} n_k(s,a) \le N_s,
\quad \forall s\in\mathcal S, \\
& \sum_{k\in\mathcal K} n_k(s,a) \le N_s-n_f(s,a),
\quad \forall s\in\mathcal S,\;a\in\mathcal A, \\
& n_k(s,a)\in\mathbb Z_{\ge0},
\quad \forall s\in\mathcal S,\;a\in\mathcal A,\;k\in\mathcal K .
\end{aligned}
\end{equation}

Write $\mathcal F$ for the feasible set in~\eqref{eq:optimization_problem}.
The first constraint enforces the budget; the second and third encode the one-annotation protocol and can be relaxed if repeated annotations are allowed.

In practice, source profiles are estimated from an independent pilot, so we compare allocations selected under estimated and true profiles.
Let $\theta$ and $\hat\theta$ denote the true and estimated source profiles, and let $\mathbf n^\star\in\arg\min_{\mathbf n\in\mathcal F}J(\mathbf n;\theta)$.
For a feasible allocation $\widehat{\mathbf n}$ computed using $\hat\theta$, write $J(\widehat{\mathbf n};\hat\theta)\le\min_{\mathbf n\in\mathcal F}J(\mathbf n;\hat\theta)+\xi_{\mathrm{alg}}$, where $\xi_{\mathrm{alg}}\ge0$ is its optimization error.
Let $\delta_\varepsilon,\delta_\Delta$ denote the uniform errors in estimated biases and excess variances, $\beta_{\max}$ bound annotation and default-model bias magnitudes, and $A_1,A_2$ be the sensitivity constants defined in Appendix~\ref{app:profile_sensitivity}.
The following result translates profile-estimation and optimization errors into excess OPE MSE.

\begin{theorem}[Effect of source-profile estimation on MSE]
\label{thm:estimated-source-regret}
Under Theorem~\ref{thm:objective}'s assumptions, conditional on the independent pilot and with fixed factual reward moments and policy probabilities,
\begin{equation}
\label{eq:source-profile-mse-bound}
\operatorname{MSE}(\widehat{\mathbf n})
-
\operatorname{MSE}(\mathbf n^\star)
\le
\frac{
2A_1\delta_\varepsilon
(2\beta_{\max}+\delta_\varepsilon)
+
2A_2\delta_\Delta
+
\xi_{\mathrm{alg}}
}{N}.
\end{equation}
\end{theorem}
The guarantee allows estimated profiles to change which sources or pairs receive annotations while controlling the resulting MSE loss.
As profile errors vanish, the bound reduces to $\xi_{\mathrm{alg}}/N$, distinguishing the gains from better source-profile estimates from those obtained by improving the allocation solver.
We examine these two effects in \S\ref{sec:experiments}; Appendix~\ref{app:profile_sensitivity} gives the proof.

\section{Annotation Value: Thresholds, Regimes, and Multi-Source Benefits}
\label{sec:annotation_value}

Before solving problem~\eqref{eq:optimization_problem}, we examine when counterfactual annotations are statistically valuable.
We first study a single context-action pair and source to derive the first-annotation threshold (\S\ref{sec:first_annotation_threshold}) and four regimes of annotation value (\S\ref{sec:four_regimes}).
These results motivate source selection and bias cancellation (\S\ref{sec:multi_source_help}), followed by an analysis of cross-action coupling in the full allocation problem (\S\ref{sec:local_to_full}).

\subsection{Local Single-Source Objective and First-Annotation Threshold}
\label{sec:first_annotation_threshold}

To gain actionable insight, we first focus on a single context-action pair $(s,a)$ with $n_f(s,a)>0$ and $\alpha_1(s,a)>0$, and a single annotation source $g$.
The case $n_f(s,a)=0$ is discussed in Appendix~\ref{sec:nf_zero}.
For ease of notation, we suppress the dependence on $(s,a)$ throughout and write $n_f:=n_f(s,a)$, $\varepsilon:=\varepsilon_g(s,a)$, $\Delta:=\Delta_g(s,a)$, $\sigma^2:=\sigma_r^2(s,a)$, $\sigma_g^2:=\sigma^2+\Delta$, $\eta:=\eta(s,a)$, and $\alpha_i:=\alpha_i(s,a)$ for $i=1,3$.
For $n$ annotations, the $(s,a)$ terms of $J_1+J_3$ give the local objective
\begin{equation}
    \phi(n)
    :=
    \frac{
        \alpha_1 n^2 \varepsilon^2
        +
        \alpha_3\bigl(n_f\sigma^2+n\sigma_g^2\bigr)
    }
    {(n_f+n)^2}.
    \label{eq:phi_main}
\end{equation}
We omit $J_2$ here because it couples actions within the same context, and revisit its effect in \S\ref{sec:local_to_full}.
By isolating this local scenario, we derive a threshold condition under which adding the first annotation does not increase the local objective \eqref{eq:phi_main}.

\begin{theorem}[First Annotation Threshold]
\label{thm:first_threshold_main}
The first annotation, corresponding to the transition $n=0$ to $n=1$, does not increase the local objective \eqref{eq:phi_main} if and only if
\begin{equation}
    \varepsilon^2(s,a) \leq \eta(s,a)\!\left(\frac{(n_f(s,a)+1)\sigma^2(s,a)}{n_f(s,a)} - \Delta(s,a)\right), \label{eq:first_thresh_main}
\end{equation}
where $\eta(s,a)=1-\pi_b(a\mid s)$.
In particular, if $\Delta(s,a) \geq (n_f(s,a)+1)\sigma^2/n_f(s,a)$, then $\delta\phi(0)>0$ for every (s,a) with $\varepsilon(s,a)\ne 0$.
\end{theorem}

The threshold explains why the same source can be useful at one pair but not another: its bias is judged against the variance reduction available at that pair.
Greater factual coverage lowers the factor $(n_f+1)/n_f$, while greater excess variance reduces the remaining variance benefit; both make the first annotation harder to justify.
Appendix~\ref{pf:t1} gives the proof.

\subsection{Four Regimes of Annotation Value}
\label{sec:four_regimes}

The first-annotation threshold compares only $\phi(1)$ with $\phi(0)$.
To understand how annotation value evolves as more annotations are added, we now characterize the behavior of the local objective $\phi(n)$ over all annotation counts.
Although $n$ is integer-valued in the allocation problem, treating it as continuous reveals the marginal structure of $\phi(n)$.
Differentiating \eqref{eq:phi_main} gives $\phi'(n) = (\mathfrak{a}\,n + \mathfrak{b})/(n_f+n)^3$, where $\mathfrak{a} := 2\alpha_1\varepsilon^2 n_f - \alpha_3\sigma_g^2$ and $\mathfrak{b} := \alpha_3 n_f(\Delta - \sigma^2)$.
Thus, the sign of $\phi'(n)$ is determined by the affine term $\mathfrak a n+\mathfrak b$, yielding four possible regimes.

\begin{theorem}[Four Regimes of Annotation Value]
\label{thm:four_regimes_main}
In all cases, $\phi(n)\to\alpha_1\varepsilon^2$ as $n\to\infty$.
When $\mathfrak a=\mathfrak b=0$, $\phi$ is constant; otherwise, it falls into four regimes:
\begin{enumerate}[
    label=\textbf{(\Roman*)},
    leftmargin=2.5em,
    labelsep=0.5em,
    itemsep=2pt,
    topsep=2pt
]
    \item If $\mathfrak a\geq 0$ and $\mathfrak b\geq 0$, then $\phi$ is increasing, so annotations are locally harmful.
    
    \item If $\mathfrak a>0$ and $\mathfrak b<0$, then $\phi$ decreases and then increases, with unique continuous minimizer $n^* = \eta n_f(\sigma^2-\Delta) / (2\varepsilon^2 n_f-\eta\sigma_g^2)$.
    
    \item If $\mathfrak a\leq 0$ and $\mathfrak b\leq 0$, then $\phi$ is decreasing, so each annotation locally reduces the objective.

    \item If $\mathfrak a<0$ and $\mathfrak b>0$, then $\phi$ first increases and then decreases, with unique continuous maximizer $\bar n = \eta n_f(\Delta-\sigma^2) / (\eta\sigma_g^2-2\varepsilon^2 n_f)$.
    In this regime, a strict improvement over no annotation is possible if and only if $\varepsilon^2<\eta\sigma^2/n_f$.
\end{enumerate}
\end{theorem}

Regime~(II) favors a finite count, so spending the entire available budget on one pair can be counterproductive.
In Regime~(IV), a sufficiently large batch improves on no annotation only if the limiting bias contribution satisfies $\alpha_1\varepsilon^2<\phi(0)$ (IVa); otherwise, $n=0$ remains best (IVb).
Table~\ref{tab:regimes} translates these distinctions into local allocation decisions, subject to the available budget and capacities; Appendix~\ref{pf:t2} gives the proof.

\begin{table}[!htbp]
\centering
\caption{Conditions and local allocation decisions for $\eta>0$, excluding the constant-objective case.}
\label{tab:regimes}
\begin{tabular}{clll}
\toprule
\textbf{Regime} & \textbf{Excess variance} & \textbf{Bias condition} & \textbf{Best local action} \\
\midrule
I
& $\Delta \geq \sigma^2$
& $\varepsilon^2 \geq \frac{\eta\sigma_g^2}{2n_f}$
& no annotations $(n=0)$ \\[4pt]

II
& $\Delta < \sigma^2$
& $\varepsilon^2 > \frac{\eta\sigma_g^2}{2n_f}$
& compare $\lfloor n^*\rfloor,\lceil n^*\rceil$, capped by feasibility \\[4pt]

III
& $\Delta \leq \sigma^2$
& $\varepsilon^2 \leq \frac{\eta\sigma_g^2}{2n_f}$
& add as many as feasible \\[4pt]

\multirow{2}{*}{IV}
& $\Delta > \sigma^2$
& a: $\varepsilon^2 < \frac{\eta\sigma^2}{n_f}$
& large batch needed to improve over $n=0$ \\[4pt]

& $\Delta > \sigma^2$
& b: $\frac{\eta\sigma^2}{n_f} \leq \varepsilon^2 < \frac{\eta\sigma_g^2}{2n_f}$
& no annotations $(n=0)$ \\
\bottomrule
\end{tabular}
\end{table}

\subsection{Source Selection and Bias Cancellation}
\label{sec:multi_source_help}

The local regimes reveal how access to multiple sources can improve allocation.
Because every single-source plan remains feasible under the same budget and capacities, the multi-source optimum is no worse than the best single-source optimum.
To identify where source selection can help, we apply the first-annotation threshold~\eqref{eq:first_thresh_main} to source $k$:
\begin{equation}
    \varepsilon_{g_k}(s,a)^2 \;\leq\; \eta(s,a)\!\left(\frac{(n_f(s,a)+1)\,\sigma_r^2(s,a)}{n_f(s,a)} - \Delta_{g_k}(s,a)\right).
    \label{eq:source_selection}
\end{equation}
Because factual coverage and source quality vary across pairs, condition~\eqref{eq:source_selection} can hold for different sources at different pairs.

Sources can also complement one another within a pair.
At a fixed total annotation count, mixing two sources with opposite-signed biases can reduce the magnitude of the reward-model bias $B(s,a;\mathbf n)=\sum_{k}n_k(s,a)\varepsilon_{g_k}(s,a)/N(s,a)$ below that obtained from either source alone.
The allocation must balance this reduction against source variances and costs.
When one source performs well across pairs and mixing offers little bias cancellation, gains over the best single source can be modest.

\subsection{From Local Thresholds to the Full Allocation Problem}
\label{sec:local_to_full}

The benefits of source selection and mixing must be evaluated through the full objective, which includes cross-action coupling through $J_2$.
For $\alpha_3>0$ and zero aggregate reward-model bias in context $s$, including $J_2$ changes the first-annotation threshold to $\varepsilon^2 \leq (n_f+1)\sigma^2/n_f-\Delta$, removing the factor $\eta$ from Theorem~\ref{thm:first_threshold_main} and permitting larger admissible bias when the right-hand side is positive.
More generally, the marginal effect of $J_2$ depends on the effective biases of all actions in the same context, while source mixing can reduce $B(s,a;\mathbf n)$ via bias cancellation but reshapes $\mathcal B_s(\mathbf n)$.
Together with the shared budget, this coupling motivates the joint allocation method in \S\ref{sec:multi_source}.

\section{Budget-Constrained Multi-Source Allocation}
\label{sec:multi_source}

We now solve the multi-source allocation problem~\eqref{eq:optimization_problem}.
Proposition~\ref{prop:objective_structure_main} formalizes why the local rules from \S\ref{sec:annotation_value} cannot be applied independently: the problem has cross-action coupling, nonconvexity, and integer constraints.

\begin{proposition}[Structure of the allocation objective]
\label{prop:objective_structure_main}
The objective $J(\mathbf n)$ in problem~\eqref{eq:optimization_problem} is nonnegative and bounded over all feasible allocations.
Moreover, there exist valid instances for which:
\begin{enumerate}[label=(\Roman*),leftmargin=2em,itemsep=1pt,topsep=2pt]
    \item $J$ is non-separable across actions within the same context.
    \item The continuous relaxation of $J$ is neither convex nor concave.
    \item $J$ is neither submodular nor supermodular over annotation allocations.
\end{enumerate}
\end{proposition}

An annotation's full value depends on the other allocations in its context, so independent pairwise decisions need not produce a good joint allocation.
We use majorization-minimization (MM)~\citep{hunter2004tutorial} to replace the cross-action coupling by a locally tight upper surrogate while retaining the exact per-pair bias-variance structure.
Feasible surrogate-improving updates then guarantee monotone descent.
Appendices~\ref{sec:appendix_properties} and~\ref{app: mm} give the structural proofs and descent analysis, respectively.

\subsection{Majorization-Minimization Surrogate}
\label{sec:mm_surrogate}

We optimize the objective $J$ using an iterative MM scheme.
Starting from a feasible allocation $\mathbf n^{(0)}$ (zero or a greedy allocation), each MM iteration constructs an upper bound around $\mathbf n^{(t)}$ and seeks a feasible update that improves this surrogate.
The term $-J_2$ is the only part of $J$ that couples actions within the same context through $\mathcal B_s(\mathbf n)^2$.
We majorize this coupling term, turning it into an affine function of $\mathcal B_s(\mathbf n)$ while leaving $J_1$ and $J_3$ unchanged.
The following proposition shows that this replacement yields an upper bound that is exact at the current allocation.

\begin{proposition}[Tight majorization of $J$]
\label{prop:mm_majorant_main}
Fix a feasible allocation $\mathbf n^{(t)}$ at iteration $t$ and let $\mathcal B_s^{(t)}:=\mathcal B_s(\mathbf n^{(t)})$.
Define $ \tilde J_{\mathrm{full}}(\mathbf n;\mathbf n^{(t)}) = J_1(\mathbf n)+J_3(\mathbf n) + \sum_{s\in\mathcal S} d_0(s) \bigl[ (\mathcal B_s^{(t)})^2 - 2\mathcal B_s^{(t)}\mathcal B_s(\mathbf n) \bigr]. $
Then $\tilde J_{\mathrm{full}}$ is a tight upper bound on $J$: $ \tilde J_{\mathrm{full}}(\mathbf n;\mathbf n^{(t)})-J(\mathbf n) = \sum_{s\in\mathcal S} d_0(s) \bigl( \mathcal B_s(\mathbf n)-\mathcal B_s^{(t)} \bigr)^2 \ge 0. $
In particular, $\tilde J_{\mathrm{full}}(\mathbf n^{(t)};\mathbf n^{(t)}) =J(\mathbf n^{(t)})$.
\end{proposition}

Tightness at the current allocation ensures that a feasible update improving the surrogate cannot increase $J$.
For computation, we drop the allocation-independent constant $\sum_s d_0(s)(\mathcal B_s^{(t)})^2$, obtaining the equivalent surrogate
\begin{align}
\tilde J(\mathbf n;\mathbf n^{(t)})
=
J_1(\mathbf n)+J_3(\mathbf n)
-
2\sum_{s\in\mathcal S} d_0(s)
\mathcal B_s^{(t)}\mathcal B_s(\mathbf n).
\label{eq:surrogate_def}
\end{align}
Substituting $\mathcal B_s(\mathbf n)=\sum_a\pi_e(a|s)B(s,a;\mathbf n)$ into \eqref{eq:surrogate_def} shows that $\tilde J$ is additively separable in $(s,a)$.
Each iteration therefore searches for an improving allocation using these separable costs under the original constraints, preserving descent in $J$.
Appendix~\ref{app: mm} gives the majorization proof.

\subsection{Solving the Surrogate Allocation Problem}
\label{sec:surrogate_solver}

We now describe the computational step at a fixed MM iteration $t$.
We write $\tilde J(\mathbf n;\mathbf n^{(t)}) =\sum_{s,a}\tilde J_{s,a}^{(t)}(\mathbf n_{s,a})$, where $\tilde J_{s,a}^{(t)}$ denotes the summand of $\tilde{J}$ specific to the context-action pair $(s,a)$, and $\mathbf n_{s,a}:=(n_1(s,a),\ldots,n_K(s,a))$ is the source-count vector.
The remaining challenge is to handle the shared budget and capacity constraints.

Let $C(\mathbf n):=\sum_{s,a,k}c_k n_k(s,a)$ denote the total annotation cost.
For a fixed multiplier $\lambda\ge0$, we minimize the penalized surrogate $\tilde J(\mathbf n;\mathbf n^{(t)})+\lambda C(\mathbf n)$, while retaining the per-context and context-action capacity constraints.
Thus, the penalized problem decomposes by context.
For each context-action pair $(s,a)$ and each feasible annotation count $m\in\{0,\ldots,N_s-n_f(s,a)\}$, define $ H_{s,a}^{(t,\lambda)}(m) := \min_{\mathbf z\in\mathbb Z_+^K:\,\mathbf 1^\top \mathbf z=m} \bigl\{ \tilde J_{s,a}^{(t)}(\mathbf z) + \lambda\sum_{k=1}^K c_k z_k \bigr\}. $
Here, $\mathbf z=(z_1,\ldots,z_K)$ specifies how the $m$ annotations for $(s,a)$ are split across sources, and the local surrogate is evaluated with $N(s,a)=n_f(s,a)+m$.
Thus, $H_{s,a}^{(t,\lambda)}(m)$ is the best penalized surrogate value among all source mixtures that assign exactly $m$ annotations to $(s,a)$.
Exact enumeration over all source-count vectors $\mathbf z$ yields a certified solution for $H_{s,a}^{(t,\lambda)}(m)$, but the cost grows combinatorially with $K$.
We instead obtain approximate values $H^\sharp_{s,a}(m,\lambda)$ by pairwise-exchange local search, which reallocates one annotation between sources at a time to reduce the penalized surrogate while keeping $m$ fixed.
Proposition~\ref{prop:greedy_opt} in Appendix~\ref{sec:C2} bounds the resulting approximation error.

Given $H_{s,a}^{(t,\lambda)}(m)$, it remains to choose how many annotations each context-action pair receives within context $s$.
Let $m_{s,a}:=\sum_{k=1}^K n_k(s,a)$ denote the total number of annotations assigned to $(s,a)$, aggregated over sources.
For each context, the remaining problem is
\begin{align*}
\min_{\{m_{s,a}\}_{a\in\mathcal A}}
\sum_{a\in\mathcal A}H_{s,a}^{(t,\lambda)}(m_{s,a})
\quad
\text{s.t.}\quad
\sum_{a\in\mathcal A}m_{s,a}\le N_s,\quad
m_{s,a}\in\{0,\ldots,N_s-n_f(s,a)\}, ~\forall a\in\mathcal{A}.
\end{align*}
We solve the aggregation by dynamic programming (DP), using $H^\sharp$ in place of $H$ and allowing unused context capacity.
DP is exact for these tables; the source mixtures remain approximate.

We search over $\lambda$ by bisection.
Each trial value of $\lambda$ returns a candidate allocation from the penalized problem above, computed by local source search followed by DP across actions within each context.
Among the candidates satisfying $C(\mathbf n)\le\kappa$, together with the current allocation $\mathbf n^{(t)}$, we select the one with the smallest surrogate value $\tilde J(\cdot;\mathbf n^{(t)})$.
Because $\mathbf n^{(t)}$ is included as a fallback, the selected update $\mathbf n^{(t+1)}$ satisfies $\tilde J(\mathbf n^{(t+1)};\mathbf n^{(t)}) \le \tilde J(\mathbf n^{(t)};\mathbf n^{(t)})$ and hence $J(\mathbf n^{(t+1)})\le J(\mathbf n^{(t)})$.
We assess allocation quality by benchmarking against exact solvers.
The evaluation of the empirical optimality gaps is detailed in \S\ref{sec:experiments}.
We next bound the computational cost of an MM update.

\begin{proposition}[Runtime per MM update]
\label{prop:complexity_main}
Let $T_{\mathrm{Bis}}$ denote the total number of multiplier trials and let $T_{\mathrm{swap}}\ge1$ bound the number of annotation reassignments per local search.
One MM update costs \(\mathcal O\!\bigl(T_{\mathrm{Bis}}\,|\mathcal{S}|\,|\mathcal{A}|\,(K^2 N_{\max}\,T_{\mathrm{swap}}+N_{\max}^{2})\bigr)\), where $N_{\max}:=\max_s N_s\ge1$.
\end{proposition}

The terms $K^2 N_{\max}T_{\mathrm{swap}}$ and $N_{\max}^{2}$ correspond to local source search and DP aggregation, respectively.
With the other parameters fixed, the runtime bound grows linearly with the number of context-action pairs.
Appendices~\ref{app:complexity} and~\ref{sec: pseudo} give the proof and pseudocode.

\section{Numerical Experiments}
\label{sec:experiments}

We evaluate \textbf{1)} how source quality affects annotation value, \textbf{2)} how MM compares with greedy and solver benchmarks, and \textbf{3)} how profile-estimation error affects final OPE performance.
For comparisons with fixed source profiles, we report scaled total MSE of the estimator~\eqref{eq:multi-source_OPE_estimator}, $\mathcal V(\mathbf n)=N\,\mathrm{MSE}(\mathbf n)=NC_0+J(\mathbf n)$.
Full experimental details are deferred to Appendix~\ref{sec:appendix_assist}.

\subsection{Annotation Value and Synthetic Validation}
\label{sec:exp_synthetic}

We construct a synthetic clinical bandit with $100$ contexts, $5$ actions, $M=1{,}000$ fitting records, and budget $\kappa=10{,}000$.
Detailed environment settings are in Appendix~\ref{app:syn_exp}.

To illustrate the four regimes in Theorem~\ref{thm:four_regimes_main}, we construct four sources combining low or high bias with low or high excess variance (Table~\ref{tab:annotators}).
For each source $k$ and pair $(s,a)$, we independently draw bias and excess variance from normal distributions parameterized by $\varepsilon_k$ and $\Delta_k$, respectively.

\begin{table}[!htbp]
\centering
\small
\caption{Configuration of annotation sources.}
\label{tab:annotators}
\begin{tabular}{clccc}
\toprule
\textbf{Source ID} & \textbf{Type} & \textbf{Bias ($\varepsilon_k$)} & \textbf{Excess Variance ($\Delta_k$)} & \textbf{Cost ($c_k$)} \\
\midrule
1 & Low Bias Low Variance & 2.0 & 6.0 & 20.0 \\
2 & Low Bias High Variance & 2.0 & 48.0 & 5.0 \\
3 & High Bias Low Variance & 5.0 & 3.0 & 10.0 \\
4 & High Bias High Variance & 5.0 & 48.0 & 1.0 \\
\bottomrule
\end{tabular}
\end{table}

We activate one source at a time and randomly allocate increasing numbers of its annotations to examine how scaled MSE changes.
We average over $100$ independently sampled environments.

Figure~\ref{fig:four_regimes} matches the predicted regimes: Source~1 decreases monotonically (Regime~III), Source~2 first increases and then decreases (Regime~IV), Source~3 decreases and then increases (Regime~II), and Source~4 increases monotonically (Regime~I).
The decomposition shows how variance reduction ($J_3$) competes with bias accumulation ($J_1$) and correction ($J_2$).

\begin{figure}[!htbp]
\centering
\includegraphics[width=\linewidth]{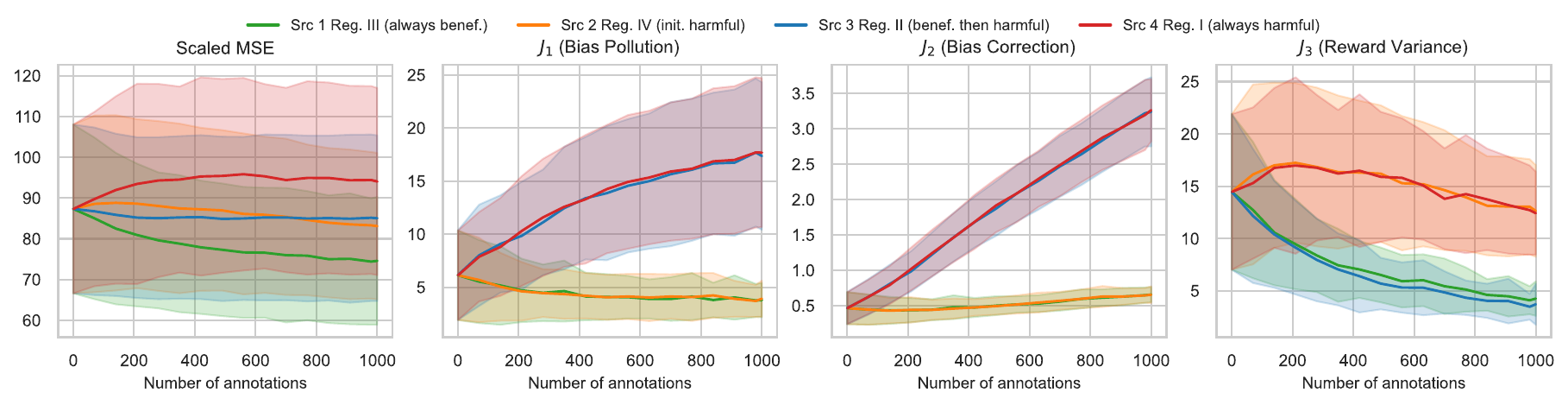}
\caption{Scaled mean squared error (MSE) $\mathcal V$ and the allocation-dependent components $J_1$, $J_2$, and $J_3$ versus the number of annotations for the four annotation sources in Table~\ref{tab:annotators}.
Shaded bands show $\pm 1$ standard deviation over $100$ environments.}
\label{fig:four_regimes}
\end{figure}

We next examine how the set of available annotation sources affects allocation quality.
We evaluate all $2^K-1=15$ nonempty source subsets $\mathcal{K}_{\mathrm{sub}}\subseteq\{1,2,3,4\}$ from Table~\ref{tab:annotators}.
For multi-source subsets, we run the proposed MM algorithm in \S\ref{sec:multi_source} restricted to the active sources.
For single-source subsets, we use a uniformly random allocation (assigning annotations uniformly across eligible logged samples and counterfactual actions until the budget or per-sample capacity is reached) as a non-optimized baseline; the corresponding optimized single-source results are reported in Appendix~\ref{app:syn_optimized}.
All runs use the same environment, factual data, and budget $\kappa=10{,}000$.

Figure~\ref{fig:source_combos} shows that the four-source setting achieves the lowest scaled MSE among the evaluated subsets, illustrating the potential benefit of source selection and mixing discussed in \S\ref{sec:multi_source_help}.
With fixed source profiles, the four-source MM allocation reduces scaled MSE from $\mathcal V(\mathbf 0)=84.23$ to $66.90$, corresponding to a $20.58\%$ MSE reduction relative to no annotation.
Full convergence traces are provided in Appendix~\ref{app:syn_convergence}.

\begin{figure}[!htbp]
\centering
\includegraphics[width=\linewidth]{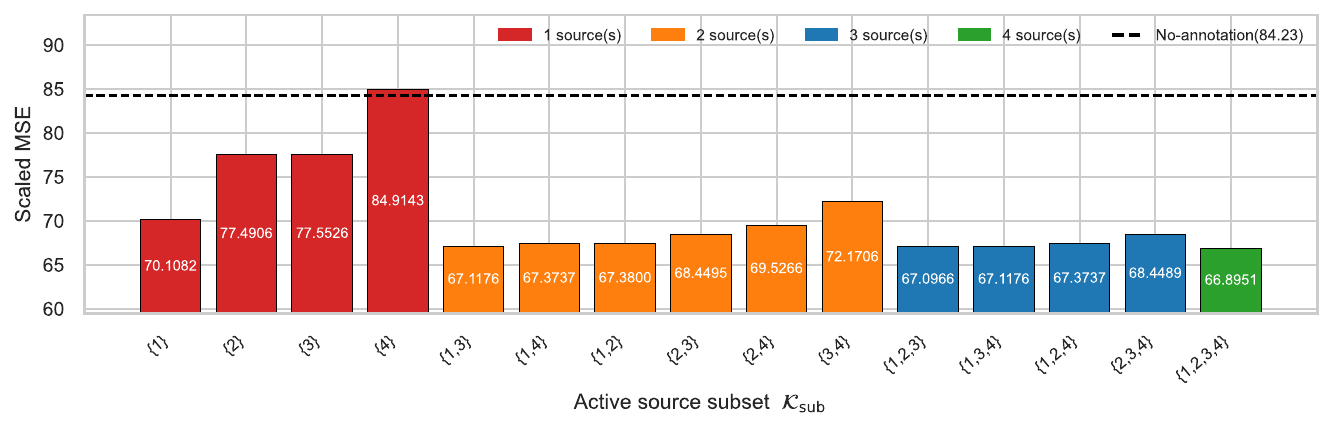}
\caption{Scaled mean squared error (MSE) $\mathcal V(\mathbf n)$ for each active source subset.
Bars are grouped by the number of available sources; the dashed line is the no-annotation baseline $\mathcal V(\mathbf 0)=84.23$.}
\label{fig:source_combos}
\end{figure}

\subsection{Baselines and Solver Benchmarks}
\label{sec:assist_experiment}

To compare allocation methods with realistic annotation sources, we also construct a semi-synthetic education bandit from $6.1$ million ASSISTments student--problem interactions~\citep{feng2009assistments}.
Contexts are $|\mathcal S|=100$ K-means clusters of student answer history and affect features, including past correct rate, hint usage, and detector confidences for frustration, confusion, concentration, and boredom.
Actions are the $|\mathcal A|=10$ \text{problem\_type}$\times$\text{original} combinations.

Following the feature-prediction approach of APRIL~\citep{mandyam2025april}, each LLM predicts the student's first-response time $t$, which is converted to reward through $R(s,a,t)=R_{\max}\cdot P_{\mathrm{IRT}}(s,a)\cdot\mathrm{TimeEff}(t,\tilde t_a)$.
Here, $R_{\max}=30$ and $P_{\mathrm{IRT}}(s,a)=\sigma(\theta_s-b_a)$ is a one-parameter logistic item response theory (IRT) model~\citep{IRT}, with context-level ability $\theta_s$ and action difficulty $b_a$.
The factor \(\mathrm{TimeEff}(t,\tilde t_a) = \max\{0.3,\,1-0.4|\log_{10}(t/\tilde t_a)|\}\) discounts responses that are much faster or slower than action $a$'s empirical median response time $\tilde t_a$.

We use Gemini-2.5-Flash, GPT-4o-mini, and Claude-Haiku-4.5 as annotation sources, collecting $20$ predictions per source and $(s,a)$ to estimate $\varepsilon_{g_k}(s,a)$ and $\Delta_{g_k}(s,a)$.
Appendix~\ref{app:assist_exp} details the data processing, reward construction, LLM prompts, and MM algorithm results.

For the baselines, Greedy (max-gain) repeatedly selects the feasible annotation with largest positive marginal gain, while Greedy (cost-aware) uses gain per unit cost.
We compare these methods with zero-initialized MM and MM initialized from the better greedy allocation on the ASSISTments, original synthetic, and I/IV-focused synthetic environments.
The I/IV-focused environment tests allocation with locally harmful or initially harmful sources, for which single-annotation gains can be uninformative.
Within each environment, all methods use the same budget, capacities, source costs, and source profiles.

\begin{table}[!htbp]
\begin{center}
\small
\setlength{\tabcolsep}{3pt}
\caption{Scaled MSE of allocations from MM and greedy baselines.}
\label{tab:greedy_stress}
\begin{tabular}{l|ccccc}
\hline
Setting & No ann. & Greedy (cost-aware) & Greedy (max-gain) & MM & MM (greedy init)\\
\hline
ASSISTments & 88.608 & 79.198 & 79.707 & \textbf{79.062} & \textbf{79.062}\\
Original synthetic & 84.234 & 69.275 & 67.078 & 66.901 & \textbf{66.8075}\\
I/IV-focused synthetic & 84.234 & 72.644 & 71.718 & \textbf{70.663} & 70.720\\
\hline
\end{tabular}
\end{center}
\end{table}

Table~\ref{tab:greedy_stress} shows that MM or greedy-initialized MM achieves the lowest MSE in each environment, with reductions of $10.77\%$, $20.68\%$, and $16.11\%$ relative to no annotation.
On the I/IV-focused instance, MM improves on both greedy baselines, illustrating the value of jointly allocating annotations beyond single-step gains.

To assess how close these allocations are to a global optimum, we construct $30$ small instances with $5$ contexts, $3$ actions, and $25$ logged samples.
For each instance, we enumerate the feasible context-level choices and use Gurobi to solve the resulting multiple-choice knapsack mixed-integer linear program to certified global optimality.
Table~\ref{tab:opt gap} shows that greedy-initialized MM has the lowest mean optimality gap.
On the full Synthetic and ASSISTments instances, it reaches within $0.2807\%$ and $0.0050\%$ of the Gurobi lower bounds, respectively (Appendix~\ref{app:solver_benchmarks}).

\begin{table}[!htbp]
\begin{center}
\small
\caption{Relative MSE gap to the certified small-instance optima.}
\label{tab:opt gap}
\begin{tabular}{l|cccc}
\hline
Method
& Mean gap (\%)
& Std. gap (\%)
& Median gap (\%)
& Max gap (\%)\\
\hline
Greedy (max-gain)
& 0.30 & 0.30 & 0.23 & 0.98\\
Greedy (cost-aware)
& 3.67 & 3.12 & 3.37 & 17.38\\
MM
& 0.31 & 0.27 & 0.26 & 1.08\\
MM (greedy init.)
& \textbf{0.05} & \textbf{0.07} & \textbf{0.02} & \textbf{0.26}\\
\hline
\end{tabular}
\end{center}
\end{table}

\subsection{Robustness to Finite-Pilot Profile Estimation}
\label{sec:exp_robustness}

The preceding comparisons use fixed source profiles; we now test whether the gains persist when profiles are estimated from finite pilots.
Across $30$ pilot replications, factual moments are estimated from $20$ pilot rewards per pair and annotation moments from $20$ pilot draws per source--pair.
The pilot, reward-model fitting data, acquired annotations, and OPE evaluation data are mutually independent.
After selecting an allocation, we evaluate DM$^{+}$-IS using $50$ independent final estimates per pilot replication.
Greedy-initialized MM performs best, reducing empirical OPE MSE by $19.13\%$ on Synthetic and $8.47\%$ on ASSISTments relative to no annotation.

As a complementary sensitivity check, we perform $30$ within-pair bootstraps of the $20$ observed LLM predictions and rerun MM after re-estimating every source profile.
On ASSISTments, the resulting mean reference-profile scaled-MSE reduction is $10.765\%$, compared with $10.774\%$ using the original profiles, with a standard deviation of only $0.0038$ percentage points.
The bootstrap indicates stable allocation quality under resampling of the observed LLM outputs, complementing the gains in final OPE MSE measured with independent evaluation data.
Full protocols and baseline results are reported in Appendices~\ref{app:end_to_end} and~\ref{app:bootsrap}.

\section{Conclusion}

We formulated budgeted counterfactual annotation for contextual-bandit off-policy evaluation as an integer allocation problem minimizing the MSE of the $\mathrm{DM}^+\text{-}\mathrm{IS}$ estimator.
The objective separates effective annotation bias, cross-action bias correction, and reward/annotation variance, explaining why annotation sources are not interchangeable.
For a single context-action pair and source, we derived a first-annotation threshold and four annotation-value regimes.
For multiple sources, we established cross-action coupling, nonconvexity, and lack of submodularity, and developed a majorization-minimization algorithm with dynamic-programming subroutines that monotonically improves the objective.
These results guide the purchase and allocation of counterfactual annotations under a limited budget to improve policy evaluation.

\textbf{Limitations and future work.}
Our framework is limited to contextual bandits and uses known or estimated source profiles $(\varepsilon_{g_k},\Delta_{g_k})$.
The ASSISTments experiment uses real LLM annotations in a semi-synthetic environment.
Future work includes sequential Markov decision processes, online estimation of source profiles, and human annotations in fully real-world experiments.

\begin{ack}
The authors thank Aishwarya Mandyam for introducing the original problem that motivated this line of work.
This work was supported in part by the Johns Hopkins Carey Business School.
This work was supported by NSF IIS-2403240, NSF IIS-2552007, NIH R01HL184139, and Schmidt Sciences AI2050 Fellowship.
The authors declare no competing interests.
\end{ack}


\newpage
\appendix

\section{Derivations for the Objective Function}
\label{sec:appendix_derivations}

This appendix provides the full, step-by-step derivations for the explicit forms of the variance decomposition described in the main text.
Throughout the appendices, we interpret $\pi_e(a|s)/\pi_b(a|s)$ and $\pi_e(a|s)^2/\pi_b(a|s)$ as zero whenever $\pi_e(a|s)=0$, consistent with the coefficient convention preceding~\eqref{eq:final_objective}.

Recall that variance of the $\text{DM}^+\text{-IS}$ estimator is decomposed as:
\begin{align*}
N \operatorname{Var} \left( \hat{v}^{\text{DM}^+\text{-IS}}(\pi_e) \right) &= \operatorname{Var}_{s \sim d_0} (v(\pi_e | s)) \\
&\quad + \mathbb{E}_{s \sim d_0} \left[ \mathbb{E}_{a \sim \pi_b(\cdot|s)} \left[ \rho(s, a)^2 \sigma_r^2(s, a) \right] \right] \\
&\quad + \mathbb{E}_{s \sim d_0} \left[ \mathbb{E}_{a \sim \pi_b(\cdot|s)} \left[ \rho(s, a)^2 \hat{\varepsilon}_{\hat{r}^+}(s, a)^2 \right] - \hat{\varepsilon}_{\hat{r}^+}^{\pi_e}(s)^2 \right] \\
&\quad + \mathbb{E}_{s \sim d_0} \Bigg[ \mathbb{E}_{a \sim \pi_b(\cdot|s)} \Bigg[ \left( \rho(s, a)^2 - \frac{\pi_e(a|s)^2}{\pi_b(a | s)} \right) \times \operatorname{Var}_{\mathcal{D}} \left( \hat{r}^+(s, a) \right) \Bigg] \Bigg]
\end{align*}
The total variance is decomposed as follows:
\begin{equation*}
N\operatorname{Var} \left( \hat{v}^{\text{DM}^+\text{-IS}}(\pi_e) \right) = T_1 + T_2 + T_3 + T_4
\end{equation*}
where $T_1$ and $T_2$ can be seen as constant with respect to $\mathbf{n}$.
We focus on deriving the explicit forms of $T_3$ and $T_4$.

In addition, as proved in ~\citep{mandyam2024}, $\hat{v}^{\text{DM}^+\text{-IS}}(\pi_e)$ remain unbiased even the annotations are biased.
Thus,
\[
\operatorname{MSE}(\mathbf n)= 0 +  \operatorname{Var} \left( \hat{v}^{\text{DM}^+\text{-IS}}(\pi_e) \right) = C_0+\frac{J(\mathbf n)}{N},
\]

where $C_0 = \frac{\operatorname{Var}_{s \sim d_0} (v(\pi_e | s)) + \mathbb{E}_{s \sim d_0} \left[ \mathbb{E}_{a \sim \pi_b(\cdot|s)} \left[ \rho(s, a)^2 \sigma_r^2(s, a) \right] \right]}{N}$.

\subsection{Term 3}

The third term is given by:
\begin{equation}
  T_3 = \mathbb{E}_{s \sim d_0} \left[ \mathbb{E}_{a \sim \pi_b(\cdot|s)} \left[ \rho(s, a)^2 \hat{\varepsilon}_{\hat{r}^+}(s, a)^2 \right] - \hat{\varepsilon}_{\hat{r}^+}^{\pi_e}(s)^2 \right]
  \label{eq:t3_origin}
\end{equation}

First, we derive the bias $\hat{\varepsilon}_{\hat{r}^+}(s, a)$ for a pair with $N(s,a)>0$.
We begin by calculating the expectation $\mathbb{E}_{\mathcal{D}}[\hat{r}^+(s, a)]$:
\begin{align*}
    \mathbb{E}_{\mathcal{D}}[\hat{r}^+(s, a)] &= \frac{1}{N(s, a)} \left( \sum_{i=1}^{n_f(s,a)} \mathbb{E}[r_i] + \sum_{k=1}^{K} \sum_{j=1}^{n_k(s,a)} \mathbb{E}[g_{k,j}] \right) \\
    &= \frac{1}{N(s, a)} \left( n_f(s, a)\mu_r(s, a) + \sum_{k=1}^{K} n_k(s, a)(\mu_r(s, a) + \varepsilon_{g_k}(s, a)) \right) \\
    &= \frac{1}{N(s, a)} \left( n_f(s, a)\mu_r(s, a) + \sum_{k=1}^{K} n_k(s, a)\mu_r(s, a) + \sum_{k=1}^{K} n_k(s, a)\varepsilon_{g_k}(s, a) \right) \\
    &= \frac{1}{N(s, a)} \left( \left(n_f(s, a) + \sum_{k=1}^{K} n_k(s, a)\right)\mu_r(s, a) + \sum_{k=1}^{K} n_k(s, a)\varepsilon_{g_k}(s, a) \right) \\
    &= \frac{1}{N(s, a)} \left( N(s,a)\mu_r(s, a) + \sum_{k=1}^{K} n_k(s, a)\varepsilon_{g_k}(s, a) \right) \\
    &= \mu_r(s, a) + \frac{\sum_{k=1}^{K} n_k(s, a)\varepsilon_{g_k}(s, a)}{N(s, a)}.
\end{align*}
The bias term $\hat{\varepsilon}_{\hat{r}^+}(s, a)$ is defined as:
\begin{align}
    \hat{\varepsilon}_{\hat{r}^+}(s, a) &:= \mathbb{E}_{\mathcal{D}}[\hat{r}^+(s, a)] - \mu_r(s, a) \notag\\
    &= \left(\mu_r(s, a) + \frac{\sum_{k=1}^{K} n_k(s, a)\varepsilon_{g_k}(s, a)}{N(s, a)} \right) - \mu_r(s, a) \notag\\
    &= \frac{\sum_{k=1}^{K} n_k(s, a)\varepsilon_{g_k}(s, a)}{N(s, a)} 
    \label{eq:epsilon_r}
\end{align}
If $N(s,a)=0$, the fixed default instead gives $\hat{\varepsilon}_{\hat r^+}(s,a)=r_0(s,a)-\mu_r(s,a)$.
Thus $\hat{\varepsilon}_{\hat r^+}(s,a)=B(s,a;\mathbf n)$ in both cases, using the definition preceding~\eqref{eq:final_objective}.
Next, we derive $\hat{\varepsilon}_{\hat{r}^+}^{\pi_e}(s)$ by expanding the expectation over the target policy $\pi_e$:
\begin{align}
    \hat{\varepsilon}_{\hat{r}^+}^{\pi_e}(s) &:= \mathbb{E}_{a \sim \pi_e(\cdot|s)} [\hat{\varepsilon}_{\hat{r}^+}(s, a)] \notag\\
    &= \sum_{a \in \mathcal{A}} \pi_e(a|s) \hat{\varepsilon}_{\hat{r}^+}(s, a) \notag\\
    &= \sum_{a \in \mathcal{A}} \pi_e(a|s) B(s,a;\mathbf n)
    \label{eq:epsilon_r_pi_e}
\end{align}
Substituting these bias expressions into \eqref{eq:t3_origin} gives
\begin{align}
    T_3 &= \sum_{s \in \mathcal{S}} d_0(s) \left( \mathbb{E}_{a \sim \pi_b(\cdot|s)} \left[ \rho(s, a)^2 \hat{\varepsilon}_{\hat{r}^+}(s, a)^2 \right] - \left( \hat{\varepsilon}_{\hat{r}^+}^{\pi_e}(s) \right)^2 \right) \notag\\
    &= \sum_{s \in \mathcal{S}} d_0(s) \left( \sum_{a \in \mathcal{A}} \pi_b(a|s) \rho(s, a)^2 B(s,a;\mathbf n)^2 - \left( \sum_{a \in \mathcal{A}} \pi_e(a|s) B(s,a;\mathbf n) \right)^2 \right)
    \label{eq:t3}
\end{align}
We simplify the importance sampling weight term $\pi_b(a|s) \rho(s, a)^2$:
\begin{align*}
    \pi_b(a|s) \rho(s, a)^2 &= \pi_b(a|s) \left( \frac{\pi_e(a|s)}{\pi_b(a|s)} \right)^2= \frac{\pi_e(a|s)^2}{\pi_b(a|s)}
\end{align*}
Substituting this back into \eqref{eq:t3}:
\begin{align}
    T_3 &= \sum_{s \in \mathcal{S}} d_0(s) \Bigg[ \sum_{a \in \mathcal{A}} \frac{\pi_e(a|s)^2}{\pi_b(a|s)} B(s,a;\mathbf n)^2 \nonumber \notag\\
    & \quad - \left( \sum_{a \in \mathcal{A}} \pi_e(a|s) B(s,a;\mathbf n) \right)^2 \Bigg]
    \label{eq:t3_final}
\end{align}

\subsection{ Term 4}

The fourth term is given by:
\begin{equation}
T_4 = \mathbb{E}_{s \sim d_0} \Bigg[ \mathbb{E}_{a \sim \pi_b(\cdot|s)} \Bigg[ \left( \rho(s, a)^2 - \frac{\pi_e(a|s)^2}{\pi_b(a | s)} \right) \operatorname{Var}_{\mathcal{D}} \left( \hat{r}^+(s, a) \right) \Bigg] \Bigg]
\label{eq:t4_origin}
\end{equation}
First, for $N(s,a)>0$, we calculate the reward-model variance $\operatorname{Var}_{\mathcal{D}} \left( \hat{r}^+(s, a) \right)$.
Using the independence of samples:
\begingroup
\allowdisplaybreaks[1]
\begin{align}
    \operatorname{Var}_{\mathcal{D}} \left( \hat{r}^+(s, a) \right) &= \operatorname{Var}_{\mathcal{D}} \left( \frac{1}{N(s, a)} \left( \sum_{i=1}^{n_f(s,a)} r_i + \sum_{k=1}^{K} \sum_{j=1}^{n_k(s,a)}g_{k,j} \right) \right) \notag\\
    &= \frac{1}{N(s, a)^2} \left( \sum_{i=1}^{n_f(s,a)} \operatorname{Var}(r_i) + \sum_{k=1}^{K} \sum_{j=1}^{n_k(s,a)} \operatorname{Var}(g_{k,j})  \right) \notag\\
    &= \frac{1}{N(s, a)^2} \left( n_f(s, a)\sigma_r^2(s, a) + \sum_{k=1}^{K}n_k(s, a)(\sigma_r^2(s, a) + \Delta_{g_k}(s, a)) \right) \notag\\
    &= \frac{1}{N(s, a)^2} \left( n_f(s, a)\sigma_r^2(s, a) + \sum_{k=1}^{K}n_k(s, a)\sigma_r^2(s, a) + \sum_{k=1}^{K}n_k(s, a)\Delta_{g_k}(s, a) \right) \notag\\
    &= \frac{1}{N(s, a)^2} \left( \left(n_f(s, a) + \sum_{k=1}^{K}n_k(s, a)\right)\sigma_r^2(s, a) + \sum_{k=1}^{K}n_k(s, a)\Delta_{g_k}(s, a) \right) \notag\\
    &= \frac{1}{N(s, a)^2} \left( N(s, a)\sigma_r^2(s, a) +\sum_{k=1}^{K}n_k(s, a)\Delta_{g_k}(s, a) \right) \notag\\
    &= \frac{\sigma_r^2(s, a)}{N(s, a)} + \frac{\sum_{k=1}^{K}n_k(s, a)\Delta_{g_k}(s, a)}{N(s, a)^2} \label{eq:var_rhat}
\end{align}
\endgroup

For $N(s,a)=0$, the fixed default has variance zero, so $\operatorname{Var}_{\mathcal D}(\hat r^+(s,a))=W(s,a;\mathbf n)$ in both cases.
Next, we substitute this variance into \eqref{eq:t4_origin} and expand the expectations, evaluating divisions by $\pi_b(a|s)$ only on its positive support:
\begin{align}
    T_4 &= \sum_{s \in \mathcal{S}} d_0(s) \sum_{a:\pi_b(a|s)>0} \pi_b(a|s) \left[ \left( \rho(s, a)^2 - \frac{\pi_e(a|s)^2}{\pi_b(a | s)} \right) \operatorname{Var}_{\mathcal{D}} \left( \hat{r}^+(s, a) \right) \right] \notag\\
    &= \sum_{s \in \mathcal{S}} d_0(s) \sum_{a:\pi_b(a|s)>0} \left[ \left( \pi_b(a|s) \rho(s, a)^2 - \frac{\pi_b(a|s)}{\pi_b(a | s)}\pi_e(a|s)^2 \right) \operatorname{Var}_{\mathcal{D}} \left( \hat{r}^+(s, a) \right) \right] \notag\\
    &= \sum_{s \in \mathcal{S}} d_0(s) \sum_{a:\pi_b(a|s)>0} \left[ \left( \pi_b(a|s) \frac{\pi_e(a|s)^2}{\pi_b(a|s)^2} - \pi_e(a|s)^2 \right) \operatorname{Var}_{\mathcal{D}} \left( \hat{r}^+(s, a) \right) \right] \notag\\
    &= \sum_{s \in \mathcal{S}} d_0(s) \sum_{a \in \mathcal{A}} \left( \frac{\pi_e(a|s)^2}{\pi_b(a|s)} - \pi_e(a|s)^2 \right)W(s,a;\mathbf n)
    \label{eq:t4_final}
\end{align}

\subsection{Objective Function}
Let $\mathbf{n} = \{n_j(s, a)\}_{s \in \mathcal{S}, a \in \mathcal{A}, j \in \mathcal{K}}$ denotes the set of decision variables.
Combining \eqref{eq:t3_final} and \eqref{eq:t4_final}, the objective function is:
\begin{align*}
    J(\mathbf{n}) &= \sum_{s \in \mathcal{S}} d_0(s) \Bigg[ \sum_{a \in \mathcal{A}} \frac{\pi_e(a|s)^2}{\pi_b(a|s)} B(s,a;\mathbf n)^2 \\
    &\quad - \left( \sum_{a \in \mathcal{A}} \pi_e(a|s) B(s,a;\mathbf n) \right)^2 \\
    &\quad + \sum_{a \in \mathcal{A}} \left( \frac{\pi_e(a|s)^2}{\pi_b(a|s)} - \pi_e(a|s)^2 \right) W(s,a;\mathbf n) \Bigg].
\end{align*}

\subsection{Theoretical bound for source-profile estimation error}
\label{app:profile_sensitivity}

The MSE decomposition above depends on source biases $\varepsilon_{g_k}$ and excess variances $\Delta_{g_k}$, which must often be estimated from pilot annotations.
We first bound the resulting perturbation of $J$ uniformly over feasible allocations, then translate this bound into the MSE guarantee in Theorem~\ref{thm:estimated-source-regret}.
Appendices~\ref{app:end_to_end} and~\ref{app:bootsrap} complement this analysis with finite-pilot evaluation and within-pair resampling.

Write \(x=(s,a)\in\mathcal X:=\mathcal S\times\mathcal A\), and define
\[
f_x:=n_f(s,a),\qquad
m_x(\mathbf n):=\sum_{k=1}^K n_k(s,a),\qquad
N_x(\mathbf n):=f_x+m_x(\mathbf n).
\]
Let
\[
\theta
:=
\left\{
\bigl(
\varepsilon_{g_k}(x),\Delta_{g_k}(x)
\bigr)
\right\}_{x,k}
\]
and
\[
\hat{\theta}
:=
\left\{
\bigl(
\widehat\varepsilon_{g_k}(x),
\widehat\Delta_{g_k}(x)
\bigr)
\right\}_{x,k}
\]
denote the true and estimated source profiles, respectively.
Define the uniform estimation errors
\[
\delta_\varepsilon
:=
\max_{x,k}
\left|
\widehat\varepsilon_{g_k}(x)
-
\varepsilon_{g_k}(x)
\right|,
\qquad
\delta_\Delta
:=
\max_{x,k}
\left|
\widehat\Delta_{g_k}(x)
-
\Delta_{g_k}(x)
\right|.
\]

To include the zero-factual-coverage case of Appendix~\ref{sec:nf_zero}, define
\[
B_x(\mathbf n;\theta)
:=
\begin{cases}
\displaystyle
\frac{
\sum_{k=1}^K n_k(x)\varepsilon_{g_k}(x)
}{
N_x(\mathbf n)
},
&
N_x(\mathbf n)>0,
\\[12pt]
r_0(x)-\mu_r(x),
&
N_x(\mathbf n)=0.
\end{cases}
\]
The same default \(r_0(x)\) is used under \(\theta\) and \(\hat{\theta}\).
Let
\[
\beta_{\max}
:=
\max
\left\{
\max_{x,k}|\varepsilon_{g_k}(x)|,\,
\max_{x:f_x=0}|r_0(x)-\mu_r(x)|
\right\}.
\]
If factual rewards, source annotations, and the default value all lie in \([0,R]\), then \(\beta_{\max}\le R\).

We next define two sensitivity constants.
First,
\[
A_1
:=
\sum_{s\in\mathcal S}
\sum_{a\in\mathcal A}
\alpha_1(s,a).
\]

For the excess-variance term, let \(M_x\) denote the maximum number of annotations that can feasibly be assigned to pair \(x\).
Define
\begin{equation}
\label{eq:profile-gamma}
\gamma_x
:=
\max_{\substack{0\le m\le M_x\\m\in\mathbb Z}}
\begin{cases}
\displaystyle
\frac{m}{(f_x+m)^2},
&
f_x+m>0,
\\[8pt]
0,
&
f_x=m=0.
\end{cases}
\end{equation}
Define
\[
A_2
:=
\sum_{x\in\mathcal X}
\alpha_3(x)\gamma_x.
\]

\begin{lemma}[Uniform perturbation of the allocation objective]
\label{lem:estimated-source-perturbation}
For every feasible allocation \(\mathbf n\in\mathcal F\),
\begin{equation}
\label{eq:uniform-profile-perturbation}
\left|
J(\mathbf n;\hat{\theta})
-
J(\mathbf n;\theta)
\right|
\le
\Phi(\delta_\varepsilon,\delta_\Delta),
\end{equation}
where
\[
\Phi(\delta_\varepsilon,\delta_\Delta)
:=
A_1\delta_\varepsilon
\left(
2\beta_{\max}+\delta_\varepsilon
\right)
+
A_2\delta_\Delta.
\]
\end{lemma}

\begin{proof}
Fix an allocation \(\mathbf n\), and define
\[
U_x(\mathbf n)
:=
B_x(\mathbf n;\hat{\theta})
-
B_x(\mathbf n;\theta).
\]
If \(N_x(\mathbf n)>0\), then
\[
U_x(\mathbf n)
=
\frac{
\sum_k n_k(x)
\left(
\widehat\varepsilon_{g_k}(x)
-
\varepsilon_{g_k}(x)
\right)
}{
N_x(\mathbf n)
},
\]
and hence
\begin{equation}
\label{eq:profile-bias-error}
|U_x(\mathbf n)|
\le
\frac{m_x(\mathbf n)}{N_x(\mathbf n)}
\delta_\varepsilon
\le
\delta_\varepsilon.
\end{equation}
If \(N_x(\mathbf n)=0\), the same default is used under both parameter sets, so \(U_x(\mathbf n)=0\).
Moreover,
\[
|B_x(\mathbf n;\theta)|\le\beta_{\max}.
\]

For each context \(s\), define
\[
M_s
:=
d_0(s)
\left[
\operatorname{diag}
\left(
\frac{\pi_e(a\mid s)^2}{\pi_b(a\mid s)}
\right)_{a\in\mathcal A}
-
\boldsymbol\pi_e(s)\boldsymbol\pi_e(s)^\top
\right],
\]
where \(\boldsymbol\pi_e(s):=(\pi_e(a\mid s))_{a\in\mathcal A}\).
For any \(z\in\mathbb R^{|\mathcal A|}\),
\[
\begin{aligned}
z^\top M_s z
&=
d_0(s)
\left[
\sum_a
\frac{\pi_e(a\mid s)^2}{\pi_b(a\mid s)}z_a^2
-
\left(
\sum_a\pi_e(a\mid s)z_a
\right)^2
\right]
\\
&=
d_0(s)
\operatorname{Var}_{a\sim\pi_b(\cdot\mid s)}
\left[
\frac{\pi_e(a\mid s)}{\pi_b(a\mid s)}z_a
\right]
\ge0.
\end{aligned}
\]
Therefore, \(M_s\succeq0\).

Let
\[
\boldsymbol B_s
:=
\bigl(B_{s,a}(\mathbf n;\theta)\bigr)_{a\in\mathcal A},
\qquad
\boldsymbol U_s
:=
\bigl(U_{s,a}(\mathbf n)\bigr)_{a\in\mathcal A}.
\]
The combined bias term satisfies
\[
J_1(\mathbf n;\theta)-J_2(\mathbf n;\theta)
=
\sum_s\boldsymbol B_s^\top M_s\boldsymbol B_s.
\]
Using the seminorm induced by \(\{M_s\}_s\),
\begin{equation}
\label{eq:bias-seminorm-perturbation}
\begin{aligned}
&
\left|
\bigl[J_1(\mathbf n;\hat{\theta})
      -J_2(\mathbf n;\hat{\theta})\bigr]
-
\bigl[J_1(\mathbf n;\theta)
      -J_2(\mathbf n;\theta)\bigr]
\right|
\\
&\le
2
\left(
\sum_s\boldsymbol B_s^\top M_s\boldsymbol B_s
\right)^{1/2}
\left(
\sum_s\boldsymbol U_s^\top M_s\boldsymbol U_s
\right)^{1/2}
+
\sum_s\boldsymbol U_s^\top M_s\boldsymbol U_s.
\end{aligned}
\end{equation}
Because the rank-one term subtracted in \(M_s\) is nonnegative,
\begin{equation}
\label{eq:true-bias-seminorm-bound}
\sum_s\boldsymbol B_s^\top M_s\boldsymbol B_s
\le
\sum_{s,a}\alpha_1(s,a)B_{s,a}^2
\le
A_1\beta_{\max}^2,
\end{equation}
and, using \eqref{eq:profile-bias-error},
\begin{equation}
\label{eq:estimated-bias-seminorm-bound}
\sum_s\boldsymbol U_s^\top M_s\boldsymbol U_s
\le
\sum_{s,a}\alpha_1(s,a)U_{s,a}^2
\le
A_1\delta_\varepsilon^2.
\end{equation}
Substituting \eqref{eq:true-bias-seminorm-bound}--\eqref{eq:estimated-bias-seminorm-bound} into \eqref{eq:bias-seminorm-perturbation} gives
\begin{equation}
\label{eq:bias-profile-perturbation}
\left|
\bigl[J_1(\mathbf n;\hat{\theta})
      -J_2(\mathbf n;\hat{\theta})\bigr]
-
\bigl[J_1(\mathbf n;\theta)
      -J_2(\mathbf n;\theta)\bigr]
\right|
\le
A_1\delta_\varepsilon
\left(
2\beta_{\max}+\delta_\varepsilon
\right).
\end{equation}

For the excess-variance term,
\begin{equation}
\label{eq:variance-profile-perturbation}
\begin{aligned}
\left|
J_3(\mathbf n;\hat{\theta})
-
J_3(\mathbf n;\theta)
\right|
&\le
\sum_x
\alpha_3(x)
\frac{
\sum_k n_k(x)
\left|
\widehat\Delta_{g_k}(x)-\Delta_{g_k}(x)
\right|
}{
N_x(\mathbf n)^2
}
\\
&\le
\delta_\Delta
\sum_x
\alpha_3(x)
\frac{m_x(\mathbf n)}{(f_x+m_x(\mathbf n))^2}
\\
&\le
A_2\delta_\Delta,
\end{aligned}
\end{equation}
where the ratio is defined as zero when \(f_x=m_x=0\), and the last inequality follows from \eqref{eq:profile-gamma}.
Combining \eqref{eq:bias-profile-perturbation} and \eqref{eq:variance-profile-perturbation} proves \eqref{eq:uniform-profile-perturbation}.
\end{proof}

Let
\[
\mathbf n^\star
\in
\arg\min_{\mathbf n\in\mathcal F}
J(\mathbf n;\theta)
\]
be an optimal allocation under the population profiles.
Let \(\widehat{\mathbf n}\) be the allocation returned by Algorithm~\ref{alg:main} using \(\hat\theta\), and suppose that
\begin{equation}
\label{eq:estimated-objective-gap}
J(\widehat{\mathbf n};\hat{\theta})
\le
\min_{\mathbf n\in\mathcal F}
J(\mathbf n;\hat{\theta})
+
\xi_{\mathrm{alg}},
\end{equation}
where \(\xi_{\mathrm{alg}}\ge0\) denotes the algorithmic suboptimality on the estimated objective.

\begin{restatethm}[Effect of source-profile estimation on MSE]{thm:estimated-source-regret}
Under the assumptions of Theorem~\ref{thm:objective}, assume that the pilot/allocation sample, the reward-model and annotation sample, and the OPE evaluation sample are independent.
Conditional on the pilot, with factual counts, factual reward moments, and policy probabilities fixed, source-profile estimation introduces no additional OPE bias, and
\begin{equation*}
\operatorname{MSE}(\widehat{\mathbf n})
-
\operatorname{MSE}(\mathbf n^\star)
\le
\frac{
2A_1\delta_\varepsilon
(2\beta_{\max}+\delta_\varepsilon)
+
2A_2\delta_\Delta
+
\xi_{\mathrm{alg}}
}{N}.
\end{equation*}
\end{restatethm}

\begin{proof}
Let
\[
\widehat{\mathbf n}^{\,\star}
\in
\arg\min_{\mathbf n\in\mathcal F}
J(\mathbf n;\hat{\theta}).
\]
By adding and subtracting the estimated-objective values,
\[
\begin{aligned}
&J(\widehat{\mathbf n};\theta)
-J(\mathbf n^\star;\theta)
\\
&=
\left[
J(\widehat{\mathbf n};\theta)
-J(\widehat{\mathbf n};\hat{\theta})
\right]
+
\left[
J(\widehat{\mathbf n};\hat{\theta})
-J(\widehat{\mathbf n}^{\,\star};\hat{\theta})
\right]
\\
&\quad+
\left[
J(\widehat{\mathbf n}^{\,\star};\hat{\theta})
-J(\mathbf n^\star;\hat{\theta})
\right]
+
\left[
J(\mathbf n^\star;\hat{\theta})
-J(\mathbf n^\star;\theta)
\right].
\end{aligned}
\]
By Lemma~\ref{lem:estimated-source-perturbation}, the first and fourth terms are each at most \(\Phi(\delta_\varepsilon,\delta_\Delta)\).
By \eqref{eq:estimated-objective-gap}, the second term is at most \(\xi_{\mathrm{alg}}\), and the third term is nonpositive by the optimality of \(\widehat{\mathbf n}^{\,\star}\) for the estimated objective.
Hence,
\begin{equation}
\label{eq:source-profile-regret}
J(\widehat{\mathbf n};\theta)
-
J(\mathbf n^\star;\theta)
\le
2\Phi(\delta_\varepsilon,\delta_\Delta)
+
\xi_{\mathrm{alg}}.
\end{equation}
Using
\[
\Phi(\delta_\varepsilon,\delta_\Delta)
=
A_1\delta_\varepsilon
\left(2\beta_{\max}+\delta_\varepsilon\right)
+
A_2\delta_\Delta,
\]
this becomes
\begin{equation}
\label{eq:source-profile-regret-expanded}
J(\widehat{\mathbf n};\theta)
-
J(\mathbf n^\star;\theta)
\le
2A_1\delta_\varepsilon
\left(2\beta_{\max}+\delta_\varepsilon\right)
+
2A_2\delta_\Delta
+
\xi_{\mathrm{alg}}.
\end{equation}

Under the stated data separation, DM\(^{+}\)-IS is unbiased conditional on the pilot-selected allocation.
Theorem~\ref{thm:objective} therefore gives
\[
\operatorname{MSE}(\mathbf n)=C_0+\frac{J(\mathbf n;\theta)}{N}.
\]
Subtracting the MSE at \(\mathbf n^\star\) and applying \eqref{eq:source-profile-regret-expanded} proves \eqref{eq:source-profile-mse-bound}.
\end{proof}

\section{Threshold for Valuable Annotation for a single annotation source}
\label{app:threshold}

In this appendix, we analyze, for a single annotation source, under what conditions adding counterfactual annotations improves or degrades the estimator variance.
We begin with the simplest case (within a single context-action pair), then extend to actions within a single context, and finally consider arbitrary context-action pairs.

\subsection{Setup and Notation}

We consider a single annotation source ($K=1$) and suppress the source index.
For a fixed context-action pair $(s, a)$, let $n_f := n_f(s, a)$, $\varepsilon := \varepsilon_g(s, a)$, $\Delta := \Delta_g(s, a)$, $\sigma^2 := \sigma_r^2(s, a)$, and $n$ denote the number of annotations.
We assume $n_f>0$ and $\alpha_1>0$ for the analysis below; the case $n_f=0$ is treated in Appendix~\ref{sec:nf_zero}.
The total sample size is $N := n_f + n$.
Define the total annotation variance as $\sigma_g^2 := \sigma^2 + \Delta$, and the policy-dependent ratio:
\begin{equation}
    \eta(s,a) := \frac{\alpha_3(s,a)}{\alpha_1(s,a)} = \frac{\frac{\pi_e(a|s)^2}{\pi_b(a|s)} - \pi_e(a|s)^2}{\frac{\pi_e(a|s)^2}{\pi_b(a|s)}} = 1 - \pi_b(a|s). \label{eq:eta_def}
\end{equation}
This ratio satisfies $0 \leq \eta(s,a) < 1$  under the common support assumption $\pi_b(a|s) > 0$, and measures how much the behavior policy under-covers action $a$: small $\pi_b(a|s)$ yields $\eta$ close to 1.

\subsection{Bias Threshold for a Single Context-Action Pair}

Ignoring the cross-action bias correction term ($J_2$) (since $J_2$ couples all actions together, we will analysis it later in \ref{sec: bias_in_states}), the contribution of a single $(s, a)$ to the objective is:
\begin{equation}
    \phi(n) := \alpha_1 \frac{n^2 \varepsilon^2}{N^2} + \alpha_3 \left( \frac{\sigma^2}{N} + \frac{n \Delta}{N^2} \right) = \frac{\alpha_1 n^2 \varepsilon^2 + \alpha_3 \left( n_f \sigma^2 + n \sigma_g^2 \right)}{N^2}. \label{eq:phi_def}
\end{equation}

\begin{lemma}[Marginal Change for a single (s,a)]
\label{lem:marginal}
The marginal change in the local objective when increasing the annotation count from $n$ to $n+1$ is:
\begin{equation}
\begin{aligned}
    \delta\phi(n) &:= \phi(n+1) - \phi(n) \\
    &= \frac{\alpha_1 \varepsilon^2 n_f \bigl( (2n+1)N + n \bigr) + \alpha_3 \sigma_g^2 \bigl( n_f^2 - n(n+1) \bigr) - \alpha_3 n_f \sigma^2 (2N+1)}{N^2 (N+1)^2},
\end{aligned}
    \label{eq:delta_phi}
\end{equation}
where $N = n_f + n$.
\end{lemma}

\begin{proof}
Write $\phi(n) = \frac{P(n)}{N^2}$ where $P(n) = \alpha_1 n^2 \varepsilon^2 + \alpha_3 n_f \sigma^2 + \alpha_3 n \sigma_g^2$.
Then:
\begin{equation}
    \delta\phi(n) = \frac{P(n+1) \cdot N^2 - P(n) \cdot (N+1)^2}{N^2(N+1)^2}. \label{eq:delta_raw}
\end{equation}
We decompose $P(n+1) = P(n) + \alpha_1(2n+1)\varepsilon^2 + \alpha_3\sigma_g^2$ and expand the numerator:
\begin{align}
    &P(n+1) N^2 - P(n)(N+1)^2 \nonumber \\
    &= \bigl[P(n) + \alpha_1(2n+1)\varepsilon^2 + \alpha_3\sigma_g^2\bigr] N^2 - P(n)(N+1)^2 \nonumber \\
    &= P(n)\bigl[N^2 - (N+1)^2\bigr] + \bigl[\alpha_1(2n+1)\varepsilon^2 + \alpha_3\sigma_g^2\bigr] N^2 \nonumber \\
    &= -P(n)(2N+1) + \bigl[\alpha_1(2n+1)\varepsilon^2 + \alpha_3\sigma_g^2\bigr] N^2. \label{eq:num_expand}
\end{align}
We evaluate each component in \eqref{eq:num_expand} separately.
For the first term:
\begin{equation}
    -P(n)(2N+1) = -\alpha_1 n^2 \varepsilon^2(2N+1) - \alpha_3 n_f\sigma^2(2N+1) - \alpha_3 n\sigma_g^2(2N+1). \label{eq:first_comp}
\end{equation}
For the second term:
\begin{equation}
    \bigl[\alpha_1(2n+1)\varepsilon^2 + \alpha_3\sigma_g^2\bigr] N^2 = \alpha_1(2n+1)\varepsilon^2 N^2 + \alpha_3\sigma_g^2 N^2. \label{eq:second_comp}
\end{equation}
Collecting $\alpha_1\varepsilon^2$ terms we get $\alpha_1\varepsilon^2\bigl[(2n+1)N^2 - n^2(2N+1)\bigr]$.

We use the factorization $(2n+1)N^2 - n^2(2N+1) = \bigl[(2n+1)N + n\bigr]\bigl[N - n\bigr] = n_f\bigl[(2n+1)N + n\bigr]$, which follows from the difference-of-products identity with $N - n = n_f$.
Collecting $\alpha_3\sigma_g^2$ terms:
\begin{equation}
    \alpha_3\sigma_g^2\bigl[N^2 - n(2N+1)\bigr] = \alpha_3\sigma_g^2\bigl[n_f^2 - n(n+1)\bigr], \label{eq:sigma_g_term}
\end{equation}
where we expanded $N^2 - n(2N+1) = (n_f+n)^2 - 2n(n_f+n) - n = n_f^2 + 2nn_f + n^2 - 2nn_f - 2n^2 - n = n_f^2 - n^2 - n$.
The remaining term is $-\alpha_3 n_f\sigma^2(2N+1)$.
Substitute \eqref{eq:first_comp} -- \eqref{eq:sigma_g_term} back into \eqref{eq:delta_raw} yields Eq.~\eqref{eq:delta_phi}.
\end{proof}

\subsection{Proof of Theorem~\ref{thm:first_threshold_main} (First Annotation Threshold)}\label{pf:t1}

\begin{restatethm}{thm:first_threshold_main}
With $\eta(s,a)$ in \eqref{eq:eta_def}, the first annotation ($n=0\to 1$) does not increase the local objective \eqref{eq:phi_def}, i.e., $\delta\phi(0)\le 0$, if and only if
\begin{equation}
    \varepsilon^2(s,a) \leq \eta(s,a)\!\left(\frac{(n_f(s,a)+1)\sigma^2(s,a)}{n_f(s,a)} - \Delta(s,a)\right). \label{eq:first_threshold}
\end{equation}
Equality in \eqref{eq:first_threshold} leaves the local objective unchanged, whereas strict inequality gives a strict decrease.
In particular, if $\Delta(s,a) \geq (n_f(s,a)+1)\sigma^2/n_f(s,a)$, then $\delta\phi(0)>0$ for every (s,a) with $\varepsilon(s,a)\ne 0$.
\end{restatethm}

\begin{proof}
Substituting $n = 0$ and $N = n_f$ into Eq.~\eqref{eq:delta_phi}:
\begin{align}
    \delta\phi(0) &= \frac{\alpha_1 \varepsilon^2 n_f \cdot n_f + \alpha_3\sigma_g^2 \cdot n_f^2 - \alpha_3 n_f\sigma^2(2n_f + 1)}{n_f^2(n_f+1)^2} \nonumber \\
    &= \frac{n_f\bigl[\alpha_1\varepsilon^2 n_f + \alpha_3 n_f\sigma_g^2 - \alpha_3(2n_f+1)\sigma^2\bigr]}{n_f^2(n_f+1)^2} \nonumber \\
    &= \frac{\alpha_1\varepsilon^2 n_f + \alpha_3\bigl[n_f(\sigma^2 + \Delta) - (2n_f+1)\sigma^2\bigr]}{n_f(n_f+1)^2} \nonumber \\
    &= \frac{\alpha_1\varepsilon^2 n_f + \alpha_3\bigl[n_f\Delta - (n_f+1)\sigma^2\bigr]}{n_f(n_f+1)^2}. \label{eq:delta0_simplified}
\end{align}
Because $n_f>0$ and $\alpha_1>0$, the condition $\delta\phi(0)\leq0$ is equivalent to
\begin{align*}
    \alpha_1\varepsilon^2 n_f &\leq \alpha_3\bigl[(n_f+1)\sigma^2 - n_f\Delta\bigr], \nonumber \\
    \varepsilon^2 &\leq \frac{\alpha_3}{\alpha_1}\left(\frac{(n_f+1)\sigma^2}{n_f} - \Delta\right) = \eta(s,a)\left(\frac{(n_f+1)\sigma^2}{n_f} - \Delta\right). \nonumber
\end{align*}
The same equivalence holds with either equality or strict inequality throughout.
When $\Delta\geq(n_f+1)\sigma^2/n_f$ and $\varepsilon\ne0$, the right-hand side is nonpositive while $\varepsilon^2>0$, so \eqref{eq:delta0_simplified} gives $\delta\phi(0)>0$.
\end{proof}

\subsection{Proof of Theorem~\ref{thm:four_regimes_main} (Four Regimes of Annotation Value)}\label{pf:t2}

\begin{lemma}[Continuous Critical Point $n^*$]
\label{lem:critical}
Treating $n$ as a continuous variable, the derivative of $\phi$ is:
\begin{equation}
    \phi'(n) = \frac{(2\alpha_1\varepsilon^2 n_f - \alpha_3\sigma_g^2)\, n + \alpha_3 n_f(\Delta - \sigma^2)}{(n_f + n)^3}. \label{eq:phi_prime}
\end{equation}
If $2\alpha_1\varepsilon^2 n_f\ne\alpha_3\sigma_g^2$, the only possible critical point is
\begin{equation}
    n^* = \frac{\alpha_3 n_f(\sigma^2 - \Delta)}{2\alpha_1\varepsilon^2 n_f - \alpha_3\sigma_g^2} = \frac{\eta \, n_f(\sigma^2 - \Delta)}{2\varepsilon^2 n_f - \eta\,\sigma_g^2}. \label{eq:n_star}
\end{equation}
It is an interior critical point exactly when $n^*>0$.
If $2\alpha_1\varepsilon^2 n_f=\alpha_3\sigma_g^2$, there is no interior critical point unless $\alpha_3(\Delta-\sigma^2)=0$, in which case $\phi$ is constant.
Furthermore, $\phi(n) \to \alpha_1\varepsilon^2$ as $n \to \infty$.
\end{lemma}

\begin{proof}
From Eq.~\eqref{eq:phi_def}, $\phi(n) = P(n)/N^2$ where $P(n) = \alpha_1 n^2\varepsilon^2 + \alpha_3 n_f\sigma^2 + \alpha_3 n\sigma_g^2$ and $N = n_f + n$.
By the quotient rule:
\begin{align}
    \phi'(n) &= \frac{P'(n) N^2 - P(n) \cdot 2N}{N^4} = \frac{P'(n)N - 2P(n)}{N^3}, \label{eq:phi_prime_raw}
\end{align}
where $P'(n) = 2\alpha_1 n\varepsilon^2 + \alpha_3\sigma_g^2$.
Expanding the numerator:
\begin{align}
    &(2\alpha_1 n\varepsilon^2 + \alpha_3\sigma_g^2)(n_f + n) - 2(\alpha_1 n^2\varepsilon^2 + \alpha_3 n_f\sigma^2 + \alpha_3 n\sigma_g^2) \nonumber \\
    &= 2\alpha_1 n\varepsilon^2 n_f + 2\alpha_1 n^2\varepsilon^2 + \alpha_3\sigma_g^2 n_f + \alpha_3\sigma_g^2 n - 2\alpha_1 n^2\varepsilon^2 - 2\alpha_3 n_f\sigma^2 - 2\alpha_3 n\sigma_g^2 \nonumber \\
    &= 2\alpha_1\varepsilon^2 n_f\, n + \alpha_3 n_f(\sigma_g^2 - 2\sigma^2) - \alpha_3\sigma_g^2 n \nonumber \\
    &= (2\alpha_1\varepsilon^2 n_f - \alpha_3\sigma_g^2)\,n + \alpha_3 n_f(\Delta - \sigma^2), \label{eq:num_phi_prime}
\end{align}
where we used $\sigma_g^2 - 2\sigma^2 = \Delta - \sigma^2$.
When its slope is nonzero, setting the numerator to zero yields Eq.~\eqref{eq:n_star}, with an interior solution exactly when $n^*>0$.
When the slope is zero, the numerator is the constant $\alpha_3 n_f(\Delta-\sigma^2)$, so the derivative either never vanishes or is identically zero.

For the limit, divide numerator and denominator of $\phi(n)$ by $n^2$:
\begin{equation*}
    \phi(n) = \frac{\alpha_1\varepsilon^2 + \alpha_3 n_f\sigma^2/n^2 + \alpha_3\sigma_g^2/n}{(n_f/n + 1)^2} \xrightarrow{n \to \infty} \alpha_1\varepsilon^2. \qedhere
\end{equation*}
\end{proof}

Define the slope and intercept of the numerator in Eq.~\eqref{eq:phi_prime} as:
\begin{equation}
    \mathfrak{a} := 2\alpha_1\varepsilon^2 n_f - \alpha_3\sigma_g^2, \qquad \mathfrak{b} := \alpha_3 n_f(\Delta - \sigma^2). \label{eq:ab_def}
\end{equation}
\begin{restatethm}[Four Regimes of Annotation Value]{thm:four_regimes_main}
In all cases, $\phi(n)\to\alpha_1\varepsilon^2$ as $n\to\infty$.
When $\mathfrak a=\mathfrak b=0$, $\phi$ is constant; otherwise, it falls into four regimes:
\begin{enumerate}[
    label=\textbf{(\Roman*)},
    leftmargin=2.5em,
    labelsep=0.5em,
    itemsep=2pt,
    topsep=2pt
]
    \item If $\mathfrak a\geq 0$ and $\mathfrak b\geq 0$, then $\phi$ is increasing, so annotations are locally harmful.
    
    \item If $\mathfrak a>0$ and $\mathfrak b<0$, then $\phi$ decreases and then increases, with unique continuous minimizer $n^* = \eta n_f(\sigma^2-\Delta) / (2\varepsilon^2 n_f-\eta\sigma_g^2)$.
    
    \item If $\mathfrak a\leq 0$ and $\mathfrak b\leq 0$, then $\phi$ is decreasing, so each additional annotation locally reduces the objective.

    \item If $\mathfrak a<0$ and $\mathfrak b>0$, then $\phi$ first increases and then decreases, with unique continuous maximizer $\bar n = \eta n_f(\Delta-\sigma^2) / (\eta\sigma_g^2-2\varepsilon^2 n_f)$.
    In this regime, a strict improvement over no annotation is possible if and only if $\varepsilon^2<\eta\sigma^2/n_f$.
\end{enumerate}
\end{restatethm}

\begin{proof}
The numerator of $\phi'(n)$ is $\mathfrak{a}n + \mathfrak{b}$, and the denominator $(n_f+n)^3 > 0$ for all $n \geq 0$.
Hence $\operatorname{sign}(\phi'(n)) = \operatorname{sign}(\mathfrak{a}n + \mathfrak{b})$ for $n > 0$.
If $\mathfrak a=\mathfrak b=0$, then $\phi'(n)=0$ for every $n\ge0$, so $\phi(n)=\phi(0)=\alpha_1\varepsilon^2$ throughout.
When $\eta=0$, we have $\alpha_3=0$ and $\phi(n)=\alpha_1\varepsilon^2n^2/(n_f+n)^2$; thus $\varepsilon\ne0$ gives Regime~(I), while $\varepsilon=0$ gives the constant zero objective, regardless of $\Delta$.
For $\eta>0$, equivalently $\alpha_3>0$, we translate the four sign patterns into the source conditions in Table~\ref{tab:regimes}, excluding the constant case.

\textbf{Regime (I):} $\mathfrak{a} \geq 0$ and $\mathfrak{b} \geq 0$ (not both zero).
Then $\mathfrak{a}n + \mathfrak{b} > 0$ for all $n > 0$, so $\phi'(n) > 0$.
The condition $\mathfrak{b} \geq 0$ is $\Delta \geq \sigma^2$.
The condition $\mathfrak{a} \geq 0$ is $2\alpha_1\varepsilon^2 n_f \geq \alpha_3\sigma_g^2$, which gives $\varepsilon^2 \geq \frac{\alpha_3\sigma_g^2}{2\alpha_1 n_f} = \frac{\eta\sigma_g^2}{2n_f}$.

\textbf{Regime (II):} $\mathfrak{a} > 0$ and $\mathfrak{b} < 0$.
The affine function $\mathfrak{a}n + \mathfrak{b}$ starts negative at $n = 0$ and increases, crossing zero at $n^* = -\mathfrak{b}/\mathfrak{a} > 0$.
For $n < n^*$, $\phi'(n) < 0$ (decreasing); for $n > n^*$, $\phi'(n) > 0$ (increasing).
The condition $\mathfrak{b} < 0$ is $\Delta < \sigma^2$.
The condition $\mathfrak{a} > 0$ gives $\varepsilon^2 > \frac{\eta\sigma_g^2}{2n_f}$.

\textbf{Regime (III):} $\mathfrak a\le 0,\mathfrak b\le 0$ (not both zero) gives $\mathfrak a n+\mathfrak b\le 0$ for all $n>0$, with strict inequality whenever at least one of $\mathfrak a,\mathfrak b$ is strictly negative.
The conditions are $\Delta\le\sigma^2$ and $\varepsilon^2\le\eta\sigma_g^2/(2n_f)$, with at least one inequality strict.
In particular, $\Delta=\sigma^2$ is included when $\varepsilon^2<\eta\sigma^2/n_f$; equality in both conditions gives the constant case already treated.

\textbf{Regime (IV):} The sign transition is reversed: $\mathfrak{b} > 0$ gives $\phi'(0^+) > 0$, and $\mathfrak{a} < 0$ forces a sign change at the maximizer $\bar n=-\mathfrak b/\mathfrak a>0$ stated in the theorem.
For the sub-case analysis, we examine $\phi(n) - \phi(0)$.
The function increases up to $\bar n$, then decreases toward $\alpha_1\varepsilon^2$.

\textbf{Sub-case (IVa):} $\alpha_1\varepsilon^2 < \phi(0)$.
Because $\lim_{n\to\infty}\phi(n)<\phi(0)$, there is a finite $T>\bar n$ with $\phi(T)<\phi(0)$.
The intermediate value theorem on $[\bar n,T]$, together with $\phi(\bar n)>\phi(0)$ and strict decrease after $\bar n$, gives a unique break-even count $n_{\mathrm{be}}>\bar n$ satisfying $\phi(n_{\mathrm{be}})=\phi(0)$.
Consequently, $\phi(n)<\phi(0)$ exactly when $n>n_{\mathrm{be}}$; under a finite capacity, this improvement is attainable only if such a count is feasible.

\textbf{Sub-case (IVb):} $\alpha_1\varepsilon^2 \geq \phi(0)$.
We show $\phi(n) \geq \phi(0)$ for all $n \geq 0$.
Since $\phi$ is increasing on $(0,\bar n)$, we have $\phi(n)>\phi(0)$ for $n\in(0,\bar n]$.
On $(\bar n,\infty)$, $\phi$ is strictly decreasing toward $\alpha_1\varepsilon^2\ge\phi(0)$.
Thus $\phi(n)>\alpha_1\varepsilon^2\ge\phi(0)$ for every $n>\bar n$, proving the claim, including the boundary case $\alpha_1\varepsilon^2=\phi(0)$.

The sub-case condition $\alpha_1\varepsilon^2 \lessgtr \phi(0) = \alpha_3\sigma^2/n_f$ translates to $\varepsilon^2 \lessgtr \eta\sigma^2/n_f$.
\end{proof}

\paragraph{Integer counts in Regime~(II).}
For the local problem with feasible counts $n\in\{0,\ldots,U\}$, where $U$ is a nonnegative integer, it suffices to compare $\phi$ at $\min\{\lfloor n^*\rfloor,U\}$ and $\min\{\lceil n^*\rceil,U\}$.
If $n^*\le U$, monotonicity on each side of $n^*$ implies that every feasible integer has an objective value at least as large as that at $\lfloor n^*\rfloor$ or $\lceil n^*\rceil$.
If $n^*>U$, $\phi$ decreases throughout $[0,U]$, so both capped candidates equal the optimum $U$.
When the candidates coincide, including $U=0$ or integer $n^*$, only one evaluation is needed; if their objective values tie, either is optimal.

\subsection{Asymptotic Behavior and Bias Floor}
\begin{corollary}[Diminishing Marginal Returns]
\label{cor:diminishing}
In all regimes, the marginal effect of annotation vanishes asymptotically:
\begin{equation}
    |\phi'(n)| = \frac{|\mathfrak{a}\,n + \mathfrak{b}|}{(n_f + n)^3} = O\!\left(\frac{1}{n^2}\right) \quad \text{as } n \to \infty. \label{eq:diminishing}
\end{equation}
In Regime~(II), annotations beyond $n^*$ increase $\phi$, with marginal cost of magnitude $O(1/n^2)$.
In Regime~(III), and in Regime~(IV) beyond the maximizer $\bar n$, annotations decrease $\phi$, with marginal benefit of magnitude $O(1/n^2)$.
The discrete marginal change $\delta\phi(n)=\phi(n+1)-\phi(n)$ also has magnitude $O(1/n^2)$.
\end{corollary}

\begin{proof}
For $n\ge1$, $|\mathfrak a n+\mathfrak b|\le(|\mathfrak a|+|\mathfrak b|)n$ and $(n_f+n)^3\ge n^3$, so $|\phi'(n)|\le(|\mathfrak a|+|\mathfrak b|)/n^2$.
The signs follow from Theorem~\ref{thm:four_regimes_main}: $\phi'>0$ after the minimum in Regime~(II), whereas $\phi'<0$ throughout Regime~(III) and after the maximum in Regime~(IV).
For integer $n\ge1$, the mean value theorem gives $\delta\phi(n)=\phi'(\xi_n)$ for some $\xi_n\in(n,n+1)$, yielding the same bound on its magnitude.
\end{proof}

\begin{corollary}[Irreducible Bias Floor]
\label{cor:floor}
Regardless of regime, $\lim_{n \to \infty}\phi(n) = \alpha_1\varepsilon^2$.
The asymptotic reduction relative to no annotation is
\begin{equation}
    \phi(0) - \alpha_1\varepsilon^2 = \frac{\alpha_1}{n_f}\bigl[\eta\,\sigma^2 - n_f\varepsilon^2\bigr]. \label{eq:max_benefit}
\end{equation}
This quantity is positive if and only if $\varepsilon^2 < \frac{\eta\,\sigma^2}{n_f}$, which is exactly the condition separating sub-cases~(IVa) from~(IVb).
The condition holds throughout the nonconstant Regime~(III), including the boundary $\Delta=\sigma^2$.
In Regime~(II), the asymptotic floor may exceed $\phi(0)$ when $\varepsilon$ is large; however, the local optimum $\phi(n^*)$ always satisfies $\phi(n^*) < \phi(0)$ since $\phi'(0^+) < 0$ guarantees initial descent.
\end{corollary}

\begin{proof}
The limit and the formula follow from Lemma~\ref{lem:critical}.
Regime~(III) requires $\eta>0$, because $\eta=0$ gives either Regime~(I) or the constant case.
If $\Delta<\sigma^2$, then $\varepsilon^2\le\eta\sigma_g^2/(2n_f)<\eta\sigma^2/n_f$.
If $\Delta=\sigma^2$, then $\mathfrak b=0$ and the nonconstant Regime~(III) requires $\mathfrak a<0$, which gives $\varepsilon^2<\eta\sigma^2/n_f$ directly.
In the excluded constant case, $\phi(0)=\alpha_1\varepsilon^2$, so the asymptotic reduction is zero.
For Regime~(II): $\mathfrak{b} < 0$ implies $\phi'(0^+) = \mathfrak{b}/n_f^3 < 0$, so $\phi$ is initially decreasing.
Since $n^*$ is the unique minimum of $\phi$ on $(0,\infty)$, continuity gives $\phi(n^*) < \phi(\epsilon) < \phi(0)$ for sufficiently small $\epsilon > 0$.
\end{proof}

\subsection{Cross-Action Coupling within a Context}\label{sec: bias_in_states}

When multiple actions are annotated within the same context, the marginal value of an annotation also depends on how it changes the aggregate bias in $J_2$.
We compute this correction to identify when it strengthens or weakens the local benefit derived above.

\begin{lemma}[Marginal change including $J_2$]
\label{lem:multi_action}
Consider adding one annotation (from $n$ to $n+1$) for action $a$ at context $s$.
The total marginal change in the objective, including the bias correction term $J_2$, is:
\begin{equation}
    \delta J(n) = \delta\phi(n) - \delta J_2(n), \label{eq:delta_total}
\end{equation}
where $\delta\phi(n)$ is the local change from Lemma~\ref{lem:marginal} and:
\begin{equation}
    \delta J_2(n) = d_0(s) \left[ 2\mathcal{B}_s \cdot \pi_e(a|s) \frac{n_f \varepsilon}{N(N+1)} + \left( \pi_e(a|s) \frac{n_f \varepsilon}{N(N+1)} \right)^2 \right], \label{eq:delta_J2}
\end{equation}
with $\mathcal{B}_s := \sum_{a' \in \mathcal{A}} \pi_e(a'|s) B(s,a';\mathbf n)$ being the current weighted bias accumulation, including default-model biases at pairs with $N(s,a')=0$.
\end{lemma}

\begin{proof}
The bias accumulation for context $s$ changes from $\mathcal{B}_s$ to $\mathcal{B}_s' = \mathcal{B}_s + \pi_e(a|s)\,\delta B(s,a)$, where:
\begin{align}
    \delta B(s,a) &:= \frac{(n+1)\varepsilon}{N+1} - \frac{n\varepsilon}{N} = \varepsilon \frac{(n+1)N - n(N+1)}{N(N+1)} = \frac{n_f \varepsilon}{N(N+1)}.  \label{eq:delta_B}
\end{align}
Therefore:
\begin{align}
    \delta J_2(n) &= d_0(s)\bigl[(\mathcal{B}_s')^2 - \mathcal{B}_s^2\bigr] = d_0(s)\bigl[2\mathcal{B}_s \cdot \pi_e(a|s)\,\delta B + \bigl(\pi_e(a|s)\,\delta B\bigr)^2\bigr]. \label{eq:dJ2_expand}
\end{align}
Since $-J_2$ appears in the objective, the total change is $\delta J = \delta\phi - \delta J_2$.
\end{proof}

\begin{remark}[Sign of $\delta J_2$]
The quadratic term in~\eqref{eq:delta_J2} is non-negative, but the linear term $2d_0(s)\mathcal{B}_s\,\pi_e(a|s)\,n_f\varepsilon/[N(N+1)]$ has the sign of $\mathcal{B}_s\cdot\varepsilon$.
Therefore $\delta J(n)$ may be either smaller or larger than $\delta\phi(n)$ depending on whether the new bias correlates with or opposes the existing aggregate bias $\mathcal{B}_s$.
This is exactly why allocation decisions cannot be made independently across actions.
\end{remark}

\begin{theorem}[Threshold including bias correction]
\label{thm:corrected_threshold}
Assume $\alpha_3>0$.
If the aggregate bias before the first annotation for action $a$ is $\mathcal{B}_s^{(0)}=0$, then this annotation does not increase the objective ($\delta J(0)\le 0$) if and only if
\begin{equation}
    \varepsilon^2 \;\leq\; \frac{(n_f+1)\sigma^2}{n_f} - \Delta. \label{eq:corrected_simple}
\end{equation}
Equality in \eqref{eq:corrected_simple} gives $\delta J(0)=0$, whereas strict inequality gives $\delta J(0)<0$.
When the right-hand side of~\eqref{eq:corrected_simple} is positive, including $J_2$ multiplies the allowable squared bias from Theorem~\ref{thm:first_threshold_main} by $1/\eta>1$.
For general $\mathcal{B}_s^{(0)}\neq 0$, the threshold becomes
\begin{equation}
    \alpha_3\,\varepsilon^2 \;\leq\; \frac{\alpha_3\bigl[(n_f+1)\sigma^2-n_f\Delta\bigr]}{n_f}
    + 2 d_0(s)\,\mathcal{B}_s^{(0)}\,\pi_e(a|s)\,(n_f+1)\,\varepsilon, \label{eq:corrected_thresh_general}
\end{equation}
which couples the allowable bias of action $a$ to the existing aggregate bias of context $s$.
\end{theorem}

\begin{proof}
The condition $\delta J(0) \leq 0$ is equivalent to $\delta\phi(0) \leq \delta J_2(0)$.
At $n=0$, $N = n_f$ and the increment $\delta B(s,a) = \varepsilon/(n_f+1)$ (Eq.~\ref{eq:delta_B}).
Substituting $\delta\phi(0)$ from Eq.~\eqref{eq:delta0_simplified} and $\delta J_2(0)$ from Eq.~\eqref{eq:dJ2_expand}:
\begin{equation}
    \frac{\alpha_1\varepsilon^2 n_f + \alpha_3(n_f\Delta - (n_f+1)\sigma^2)}{n_f(n_f+1)^2} \leq \frac{d_0(s)\,2\mathcal{B}_s^{(0)}\pi_e(a|s)\,\varepsilon}{n_f+1} + \frac{d_0(s)\,\pi_e(a|s)^2\,\varepsilon^2}{(n_f+1)^2}. \label{eq:thresh_start}
\end{equation}
Multiplying both sides by $(n_f+1)^2$ and rearranging, we collect $\varepsilon^2$ terms:
\begin{equation}
    \bigl(\alpha_1 - d_0(s)\pi_e(a|s)^2\bigr)\varepsilon^2 \leq \frac{\alpha_3\bigl((n_f+1)\sigma^2 - n_f\Delta\bigr)}{n_f} + 2d_0(s)\,\mathcal{B}_s^{(0)}\,\pi_e(a|s)\,(n_f+1)\,\varepsilon. \label{eq:thresh_collected}
\end{equation}
The key identity is:
\begin{equation}
    \alpha_1 - d_0(s)\pi_e(a|s)^2 = d_0(s)\pi_e(a|s)^2\!\left(\frac{1}{\pi_b(a|s)} - 1\right) = \alpha_3. \label{eq:key_identity}
\end{equation}
Substitute Eq.~\eqref{eq:key_identity} into LHS of Eq.~\eqref{eq:thresh_collected} yields Eq.~\eqref{eq:corrected_thresh_general}.
Setting $\mathcal{B}_s^{(0)}=0$ and dividing by $\alpha_3>0$ gives Eq.~\eqref{eq:corrected_simple}.
All transformations preserve equality and strict inequality, proving the corresponding statements for $\delta J(0)=0$ and $\delta J(0)<0$.
\end{proof}

\subsection{Extension to Arbitrary Context-Action Pairs}

In~\eqref{eq:final_objective}, each aggregate bias $\mathcal B_s$ depends only on allocations within context $s$.
Thus the objective is additive across contexts: $J(\mathbf n)=\sum_{s\in\mathcal S}J_s(\mathbf n_s)$, where $\mathbf n_s:=\{n_k(s,a):a\in\mathcal A,\ k\in\mathcal K\}$ and $J_s$ contains that context's terms.\label{eq:state_sep}
The shared budget in~\eqref{eq:optimization_problem} nevertheless couples the feasible choices across contexts.

The four regimes in Theorem~\ref{thm:four_regimes_main} and the diminishing marginal effects in Corollary~\ref{cor:diminishing} describe the local objective $\phi$, which omits $J_2$.
They do not directly characterize the coupled context objective $J_s$.
For a single annotation with other allocations held fixed, Theorem~\ref{thm:corrected_threshold} gives the full first-annotation condition through the current aggregate bias.
Joint allocation must therefore account for both within-context bias coupling and the shared budget, as in \S\ref{sec:multi_source}.

\subsection{The Case \texorpdfstring{$n_f(s,a) = 0$}{n\_f(s,a) = 0}}
\label{sec:nf_zero}
When $n_f(s,a)=0$, the fitting data contain no observations for the pair $(s,a)$.
This can occur even when $\pi_b(a|s)>0$, so an empty fitting cell does not by itself violate common support.
The formulas in Theorem~\ref{thm:four_regimes_main} require $n_f > 0$.
We treat this case separately.

\subsubsection{\texorpdfstring{$\pi_b(a|s) > 0$ but $n_f(s,a) = 0$}{pi\_b(a|s) > 0 but n\_f(s,a) = 0}}

We retain the single-source setting and consider pairs with $\alpha_1>0$; if $\alpha_1=0$, then $\alpha_3=0$ and the local objective is identically zero.
When no annotations have been acquired, $N(s,a)=0$ and we use the fixed default $\hat r^+(s,a)=r_0(s,a)$.
Then
\begin{equation}
    \hat{\varepsilon}_{\hat{r}^+}(s,a)\big|_{N=0} = r_0(s,a) - \mu_r(s,a),
    \qquad
    \operatorname{Var}_{\mathcal{D}}(\hat{r}^+(s,a))\big|_{N=0} = 0.
    \label{eq:default_model}
\end{equation}
The baseline local objective is:
\begin{equation}
    \phi_0 := \alpha_1 \bigl(\mu_r(s,a) - r_0(s,a)\bigr)^2.
    \label{eq:phi_baseline}
\end{equation}
For $n\ge1$, $N=n$ and the local objective (Eq.~\ref{eq:phi_def}) becomes:
\begin{equation}
    \phi(n) = \frac{\alpha_1 n^2 \varepsilon^2 + \alpha_3 n \sigma_g^2}{n^2}
            = \alpha_1\varepsilon^2 + \frac{\alpha_3\sigma_g^2}{n}, \qquad n \geq 1.
    \label{eq:phi_nf0}
\end{equation}
Substituting $n_f=0$ into Eq.~\eqref{eq:ab_def} gives $\mathfrak a=-\alpha_3\sigma_g^2\le0$ and $\mathfrak b=0$.
The derivative $\phi'(n)=-\alpha_3\sigma_g^2/n^2\le0$ shows that the local objective is nonincreasing for $n\ge1$, with strict decrease exactly when $\alpha_3\sigma_g^2>0$.

\begin{proposition}[Properties When $n_f = 0$ and $\pi_b(a|s) > 0$]
\label{prop:nf_zero_properties}
For $n_f(s,a)=0$ and $\alpha_1>0$, the local objective has the following properties:
\begin{enumerate}
     \item \textbf{First Annotation Threshold} The transition from $n = 0$ to $n = 1$ satisfies:
    \begin{equation}
        \delta\phi(0 \to 1) := \phi(1) - \phi_0
        = \alpha_1\bigl[\varepsilon^2 - (\mu_r - r_0)^2\bigr]
          + \alpha_3\sigma_g^2.
        \label{eq:delta_01}
    \end{equation}
    This is positive (i.e., the first annotation is locally harmful) if and only if
    \begin{equation}
        \varepsilon^2 > (\mu_r - r_0)^2 - \eta\,\sigma_g^2.
        \label{eq:first_harmful_nf0}
    \end{equation}
    Equality in \eqref{eq:first_harmful_nf0} leaves the local objective unchanged, and reversing the strict inequality gives a strict decrease.
    For an exact default $r_0=\mu_r$, the first annotation is locally harmful if and only if $\varepsilon^2+\eta\sigma_g^2>0$.
    Conversely, if the default is poor ($(\mu_r - r_0)^2$ large), the first annotation can be beneficial even with high bias.
    
    \item \textbf{Nonincreasing objective for $n\ge1$:} The marginal benefit of the $(n+1)$-th annotation is
    \begin{equation}
        -\delta\phi(n) = \frac{\alpha_3\sigma_g^2}{n(n+1)}
                       = O\!\left(\frac{1}{n^2}\right).
        \label{eq:marginal_nf0}
    \end{equation}
    The decrease is strict if $\alpha_3\sigma_g^2>0$; otherwise, $\phi(n)=\alpha_1\varepsilon^2$ for every $n\ge1$.

    \item \textbf{Bias is never diluted:} The bias of the reward estimator is $\hat{\varepsilon}_{\hat{r}^+}(s,a) = \varepsilon$ for all $n \geq 1$, independent of the annotation count.
    Without factual data, there is no unbiased component to attenuate the annotation bias.

    \item \textbf{Irreducible floor:} For $n\ge1$, the local objective is bounded below by $\alpha_1\varepsilon^2$, which cannot be reduced by increasing $n$.
    A net improvement over the baseline ($\phi(n) < \phi_0$ for some $n \geq 1$) requires:
    \begin{equation}
        \varepsilon^2 < (\mu_r - r_0)^2.
        \label{eq:net_improvement_nf0}
    \end{equation}
    That is, the annotation bias must be strictly smaller in magnitude than the default model's bias.
\end{enumerate}
\end{proposition}

\begin{proof}
\textit{(1)}  From Eqs.~\eqref{eq:phi_nf0} and~\eqref{eq:phi_baseline}:
\begin{align}
    \delta\phi(0\to 1) 
    &= \phi(1) - \phi_0 \nonumber\\
    &= \bigl[\alpha_1\varepsilon^2 + \alpha_3\sigma_g^2\bigr]
       - \alpha_1(\mu_r - r_0)^2 \nonumber \\
    &= \alpha_1\bigl[\varepsilon^2 - (\mu_r - r_0)^2\bigr]
       + \alpha_3\sigma_g^2. \nonumber
\end{align}
Because $\alpha_1>0$, the condition $\delta\phi(0\to1)>0$ is equivalent to Eq.~\eqref{eq:first_harmful_nf0}, with equality or the reversed strict inequality corresponding to zero or negative change, respectively.
Substituting $r_0=\mu_r$ gives the exact-default statement.

\textit{(2)} From Eq.~\eqref{eq:phi_nf0}, $\delta\phi(n)=\alpha_3\sigma_g^2(1/(n+1)-1/n)=-\alpha_3\sigma_g^2/(n(n+1))\le0$.
The denominator is positive for $n\ge1$, so the decrease is strict exactly when $\alpha_3\sigma_g^2>0$, and its magnitude is at most $\alpha_3\sigma_g^2/n^2$.
If $\alpha_3\sigma_g^2=0$, Eq.~\eqref{eq:phi_nf0} gives the stated constant value.

\textit{(3)} When $n_f = 0$, $B(s,a) = \frac{n\varepsilon}{n} = \varepsilon$ for all $n \geq 1$.

\textit{(4)} Equation~\eqref{eq:phi_nf0} gives $\phi(n)\ge\alpha_1\varepsilon^2$ for every $n\ge1$, with $\lim_{n\to\infty}\phi(n)=\alpha_1\varepsilon^2$.
If $\phi(n)<\phi_0$ for a feasible $n\ge1$, then $\alpha_1\varepsilon^2<\phi_0=\alpha_1(\mu_r-r_0)^2$; dividing by $\alpha_1>0$ proves the necessary condition~\eqref{eq:net_improvement_nf0}.

\end{proof}

\subsubsection{\texorpdfstring{$\pi_b(a|s) = 0$ and $n_f(s,a) = 0$}{pi\_b(a|s) = 0 and n\_f(s,a) = 0}}
When $\pi_b(a|s)=0$, neither the fitting sample nor the evaluation sample contains observations of action $a$ at context $s$.
If also $\pi_e(a|s)=0$, the pair contributes nothing to the estimator or $J$, and the zero-weight convention gives $\alpha_1(s,a)=\alpha_3(s,a)=0$.
If $\pi_e(a|s)>0$, common support fails; at a context with $d_0(s)>0$, the importance-weighted correction cannot remove the reward-model error for that action.
The unbiasedness guarantee and MSE decomposition in Theorem~\ref{thm:objective} therefore do not generally apply.
This does not imply infinite variance: with one context, $\pi_b(a_1|s)=1$, and $\pi_e(a_2|s)=1$ for $a_1\ne a_2$, the estimator in~\eqref{eq:multi-source_OPE_estimator} reduces to $\hat r^+(s,a_2)$, which has finite variance under the stated moment assumptions but may be biased.

\section{Properties of the Objective Function}
\label{sec:appendix_properties}

This appendix establishes the four structural properties of $J(\mathbf{n})$ claimed in Proposition~\ref{prop:objective_structure_main} of the main text.
Nonnegativity and boundedness hold for every valid instance; the remaining properties are established by counterexamples.
We organize the proof into four parts, summarized below:
\begin{enumerate}[label=(\Roman*),leftmargin=2em]
    \item $J$ is non-negative and bounded over $\mathcal{F}$ \hfill(Proposition~\ref{prop:nonneg});
    \item $J$ is non-separable across actions within the same context \hfill(Proposition~\ref{prop:nonsep});
    \item the continuous relaxation of $J$ is neither convex nor concave \hfill(Proposition~\ref{prop:nonconvex});
    \item $J$ is neither submodular nor supermodular over annotation allocations \hfill(Corollary~\ref{cor:nonsubmod}).
\end{enumerate}

\subsection{Decomposition and Notation}
\label{sec:B1_decomp}

We use $B(s,a):=B(s,a;\mathbf n)$ and $W(s,a):=W(s,a;\mathbf n)$ as defined before~\eqref{eq:final_objective}, including the default-model case $N(s,a)=0$.
The objective decomposes as
\begin{equation*}
    J(\mathbf{n}) \;=\; J_1(\mathbf{n}) \;-\; J_2(\mathbf{n}) \;+\; J_3(\mathbf{n}),
\end{equation*}
where the three components are:
\begin{align*}
    J_1(\mathbf{n}) &= \sum_{s} \sum_{a} \alpha_1(s,a)\, B(s,a)^2, \\
    J_2(\mathbf{n}) &= \sum_{s} d_0(s) \,\mathcal{B}_s(\mathbf{n})^2, \\
    J_3(\mathbf{n}) &= \sum_{s} \sum_{a} \alpha_3(s,a) W(s,a).
\end{align*}

For a fixed context-action pair $(s,a)$, we define:
\begin{alignat*}{3}
    &N \;:=\; N(s,a) = n_f + \textstyle\sum_{k} n_k, &\qquad& n_f \;:=\; n_f(s,a), \\
    &P \;:=\; \textstyle\sum_{k=1}^K n_k\, \varepsilon_{g_k}(s,a), &\qquad& Q_D \;:=\; \textstyle\sum_{k=1}^K n_k\, \Delta_{g_k}(s,a), \\
    &B(s,a) = P / N \quad (N>0), &\qquad& \sigma^2 \;:=\; \sigma_r^2(s,a), \\
    &\alpha_1(s,a) \;:=\; d_0(s)\,\frac{\pi_e(a|s)^2}{\pi_b(a|s)}, &\qquad& \alpha_3(s,a) = \bigl(1-\pi_b(a|s)\bigr)\alpha_1(s,a).
\end{alignat*}
For $N>0$, $W(s,a)=\sigma^2/N+Q_D/N^2$; for $N=0$, we use $B(s,a)=r_0(s,a)-\mu_r(s,a)$ and $W(s,a)=0$.
The coefficients use the zero-weight convention $\alpha_1=\alpha_3=0$ when $\pi_e(a|s)=0$.
The context-level bias accumulation is
\[
    \mathcal{B}_s(\mathbf{n}) \;:=\; \sum_{a \in \mathcal{A}} \pi_e(a|s)\, B(s,a).
\]
We write $E_k := \varepsilon_{g_k}(s,a)$ and $D_k := \Delta_{g_k}(s,a)$ when $(s,a)$ is clear from context.
The shorthand $n_j$ denotes the annotation count $n_j(s,a)$ from source~$j$ for the pair under consideration.

We note the following elementary sign properties:
\begin{itemize}
    \item $\alpha_1(s,a)\ge0$, with strict positivity exactly when $d_0(s)>0$ and $\pi_e(a|s)>0$ under common support;
    \item $\alpha_3(s,a)=\alpha_1(s,a)-d_0(s)\pi_e(a|s)^2=\bigl(1-\pi_b(a|s)\bigr)\alpha_1(s,a)\ge0$, with equality exactly when $\alpha_1(s,a)=0$ or $\pi_b(a|s)=1$.
\end{itemize}

\subsection{First and Second-Order Derivatives}
\label{sec:B2_derivatives}

We compute all partial derivatives with respect to the annotation counts $n_k := n_k(s,a)$.
Since $J_1$ and $J_3$ decompose over individual $(s,a)$ pairs while $J_2$ decomposes over contexts, the analysis is organized accordingly.
At $N(s,a)=0$, finite annotation changes are evaluated using the default definitions rather than the rational derivative formulas below.

\begin{lemma}[First-order partial derivatives]
\label{lem:gradients}
For a fixed $(s,a)$ pair with $N(s,a)>0$ and annotation source $k\in\mathcal K$:
\begin{align}
    \frac{\partial J_1^{(s,a)}}{\partial n_k} &= \frac{2\alpha_1\, P\,(E_k N - P)}{N^3}, \label{eq:grad_J1}\\[4pt]
    \frac{\partial J_2^{(s)}}{\partial n_k(s,a)} &= 2\,d_0(s)\,\mathcal{B}_s\,\pi_e(a|s)\,\frac{E_k N - P}{N^2}, \label{eq:grad_J2}\\[4pt]
    \frac{\partial J_3^{(s,a)}}{\partial n_k} &= \frac{\alpha_3}{N^3}\Big((D_k - \sigma^2)\,N - 2\,Q_D\Big). \label{eq:grad_J3}
\end{align}
\end{lemma}
\begin{proof}
\textit{Eq.~\eqref{eq:grad_J1}.}
Write $J_1^{(s,a)} = \alpha_1\,(P/N)^2$.
Noting $\frac{\partial P}{\partial n_k}=E_k$ and $\frac{\partial N}{\partial n_k}=1$, the chain rule gives
\[
    \frac{\partial}{\partial n_k}\!\left(\frac{P}{N}\right) = \frac{E_k\,N - P}{N^2},
\]
hence $\frac{\partial J_1^{(s,a)}}{\partial n_k} = 2\alpha_1 \cdot \frac{P}{N} \cdot \frac{E_k N - P}{N^2}$, yielding \eqref{eq:grad_J1}.

\textit{Eq.~\eqref{eq:grad_J2}.}
Since $\mathcal{B}_s = \sum_{a'}\pi_e(a'|s)\,B(s,a')$ and only $B(s,a)$ depends on $n_k(s,a)$:
\[
    \frac{\partial \mathcal{B}_s}{\partial n_k(s,a)} = \pi_e(a|s)\,\frac{E_k N - P}{N^2}.
\]
Then $\frac{\partial J_2^{(s)}}{\partial n_k(s,a)} = 2\,d_0(s)\,\mathcal{B}_s \cdot \frac{\partial \mathcal{B}_s}{\partial n_k(s,a)}$ gives \eqref{eq:grad_J2}.

\textit{Eq.~\eqref{eq:grad_J3}.}
Writing $J_3^{(s,a)} = \alpha_3\,\sigma^2 N^{-1} + \alpha_3\,Q_D\,N^{-2}$ and differentiating:
\[
    \frac{\partial J_3^{(s,a)}}{\partial n_k} = \alpha_3\!\left(-\frac{\sigma^2}{N^2} + \frac{D_k\,N - 2\,Q_D}{N^3}\right) = \frac{\alpha_3}{N^3}\Big((D_k - \sigma^2)N - 2Q_D\Big). \qedhere
\]
\end{proof}

\begin{lemma}[Hessian entries]
\label{lem:hessian}
For a fixed context $s$, the following second-order partial derivatives hold when every pair being differentiated has positive total count.

\paragraph{(i) $J_1^{(s,a)}$: same pair, sources $k,l$.}
\begin{equation}
\label{eq:hess_J1}
    \frac{\partial^2 J_1^{(s,a)}}{\partial n_k\,\partial n_l} = \frac{2\alpha_1}{N^4}\Big(E_k E_l\,N^2 - 2(E_k+E_l)\,P\,N + 3P^2\Big).
\end{equation}

\paragraph{(ii) $J_2^{(s)}$: same action $a$, sources $k,l$.}
\begin{equation}
\label{eq:hess_J2_same}
    \frac{\partial^2 J_2^{(s)}}{\partial n_k(s,a)\,\partial n_l(s,a)} = 2\,d_0(s)\!\left[\frac{\pi_e(a|s)^2\,(E_k N{-}P)(E_l N{-}P)}{N^4} + \frac{\mathcal{B}_s\,\pi_e(a|s)\,(2P{-}(E_k{+}E_l)N)}{N^3}\right]\!.
\end{equation}

\paragraph{(iii) $J_2^{(s)}$: different actions $a \neq a'$, sources $k,l$.}
\begin{equation}
\label{eq:hess_J2_cross}
    \frac{\partial^2 J_2^{(s)}}{\partial n_k(s,a)\,\partial n_l(s,a')} = \frac{2\,d_0(s)\,\pi_e(a|s)\,\pi_e(a'|s)\,(E_k N_a - P_a)(E_l N_{a'} - P_{a'})}{N_a^2\,N_{a'}^2},
\end{equation}
where $N_a := N(s,a)$, $P_a := P(s,a)$, etc.

\paragraph{(iv) $J_3^{(s,a)}$: same pair, sources $k,l$.}
\begin{equation}
\label{eq:hess_J3}
    \frac{\partial^2 J_3^{(s,a)}}{\partial n_k\,\partial n_l} = \frac{2\alpha_3}{N^4}\Big((\sigma^2 - D_k - D_l)\,N + 3\,Q_D\Big).
\end{equation}
\end{lemma}
\begin{proof}
All four identities follow from applying the quotient rule to the first-order expressions in Lemma~\ref{lem:gradients}.
We detail each computation below.

\textit{Part (i).}
Let $u := P/N$ so that $J_1^{(s,a)} = \alpha_1 u^2$.
From Lemma~\ref{lem:gradients}, $\frac{\partial u}{\partial n_k} = (E_k - u)/N$.
Differentiating:
\begin{align*}
    \frac{\partial^2 (\alpha_1 u^2)}{\partial n_k\,\partial n_l} &= 2\alpha_1\!\left(\frac{\partial u}{\partial n_k}\,\frac{\partial u}{\partial n_l} + u\,\frac{\partial^2 u}{\partial n_k\,\partial n_l}\right).
\end{align*}
We compute
\[
    \frac{\partial^2 u}{\partial n_k\,\partial n_l} = \frac{\partial}{\partial n_l}\!\left(\frac{E_k - u}{N}\right) = \frac{-\frac{\partial u}{\partial n_l}\,N - (E_k-u)}{N^2} = \frac{2u - E_k - E_l}{N^2}.
\]
Substituting and collecting terms with $u = P/N$:
\begin{align*}
    \frac{\partial^2 J_1^{(s,a)}}{\partial n_k\,\partial n_l} &= \frac{2\alpha_1}{N^2}\!\left[(E_k-u)(E_l-u) + u(2u - E_k - E_l)\right] \\
    &= \frac{2\alpha_1}{N^2}\!\left[E_kE_l - 2(E_k+E_l)u + 3u^2\right],
\end{align*}
which yields \eqref{eq:hess_J1} after substituting $u = P/N$.

\textit{Part (ii).}
Differentiating \eqref{eq:grad_J2} with respect to $n_l(s,a)$:
\[
    \frac{\partial^2 J_2^{(s)}}{\partial n_k(s,a)\,\partial n_l(s,a)} = 2\,d_0(s)\!\left[\frac{\partial\mathcal{B}_s}{\partial n_l}\!\cdot\!\pi_e(a|s)\frac{E_kN{-}P}{N^2} + \mathcal{B}_s\!\cdot\!\pi_e(a|s)\frac{\partial}{\partial n_l}\!\left(\frac{E_kN{-}P}{N^2}\right)\right].
\]
Using $\frac{\partial \mathcal{B}_s}{\partial n_l(s,a)} = \pi_e(a|s)\frac{E_lN-P}{N^2}$ from Lemma~\ref{lem:gradients} and
\[
    \frac{\partial}{\partial n_l}\!\left(\frac{E_kN - P}{N^2}\right) = \frac{2P - (E_k+E_l)N}{N^3},
\]
(computed via the quotient rule) yields \eqref{eq:hess_J2_same}.

\textit{Part (iii).}
For $a \neq a'$, since $B(s,a)$ depends only on $\{n_k(s,a)\}_k$ and $B(s,a')$ depends only on $\{n_k(s,a')\}_k$, the term $\mathcal{B}_s\cdot\frac{\partial}{\partial n_l(s,a')}(\cdots)$ does not contribute a cross-action second derivative from the same factor.
Only the product of first derivatives remains:
\[
    \frac{\partial^2 J_2^{(s)}}{\partial n_k(s,a)\,\partial n_l(s,a')} = 2\,d_0(s)\,\frac{\partial \mathcal{B}_s}{\partial n_k(s,a)}\,\frac{\partial \mathcal{B}_s}{\partial n_l(s,a')},
\]
giving \eqref{eq:hess_J2_cross}.

\textit{Part (iv).}
Let $R_k := (D_k - \sigma^2)N - 2Q_D$ so that $\frac{\partial J_3^{(s,a)}}{\partial n_k} = \alpha_3 R_k / N^3$.
Then $\frac{\partial R_k}{\partial n_l} = (D_k - \sigma^2) - 2D_l$, and the quotient rule gives:
\begin{align*}
    \frac{\partial}{\partial n_l}\!\left(\frac{R_k}{N^3}\right) &= \frac{(D_k - \sigma^2 - 2D_l)\,N - 3R_k}{N^4}.
\end{align*}
Expanding $3R_k = 3(D_k-\sigma^2)N - 6Q_D$ and simplifying:
\[
    (D_k-\sigma^2-2D_l)N - 3(D_k-\sigma^2)N + 6Q_D = 2(\sigma^2 - D_k - D_l)N + 6Q_D. \qedhere
\]
\end{proof}

\paragraph{Hessian matrices in compact form.}
For a fixed $(s,a)$, define the vectors $\mathbf{e} = (E_1,\ldots,E_K)^\top$, $\mathbf{d} = (D_1,\ldots,D_K)^\top$, and $\mathbf{1} = (1,\ldots,1)^\top \in \mathbb{R}^K$.
The Hessian matrices $\mathbf{H}_1, \mathbf{H}_3 \in \mathbb{R}^{K\times K}$ with respect to $(n_1,\ldots,n_K)$ for the given $(s,a)$ pair are:
\begin{align*}
    \mathbf{H}_1 &= \frac{2\alpha_1}{N^4}\!\left(N^2\,\mathbf{e}\mathbf{e}^\top - 2NP\,(\mathbf{e}\mathbf{1}^\top + \mathbf{1}\mathbf{e}^\top) + 3P^2\,\mathbf{1}\mathbf{1}^\top\right),\\
    \mathbf{H}_3 &= \frac{2\alpha_3}{N^4}\!\left((\sigma^2 N + 3Q_D)\,\mathbf{1}\mathbf{1}^\top - N\,(\mathbf{d}\mathbf{1}^\top + \mathbf{1}\mathbf{d}^\top)\right).
\end{align*}
Both are sums of rank-one and rank-two matrices; their spectral properties are parameter-dependent and analyzed in \S\ref{sec:B3_properties}.

\subsection{Structural Properties}
\label{sec:B3_properties}

We now prove the structural claims in Proposition~\ref{prop:objective_structure_main}.

\begin{proposition}[Non-negativity and boundedness]
\label{prop:nonneg}
For every feasible allocation $\mathbf{n} \in \mathcal{F}$,
\[
    0 \;\leq\; J(\mathbf{n}) \;\leq\; J_{\max} < \infty.
\]
\end{proposition}
\begin{proof}
\textit{Lower bound.}
For $J_1 - J_2$: fix any context $s$ and define $x_a := B(s,a)$, $w_a := \pi_e(a|s)$, and $c_a := \pi_b(a|s)$.
In the following display, action sums run over $a$ with $c_a>0$; common support gives $w_a=0$ on the omitted actions, so their contributions are zero.
By the Cauchy--Schwarz inequality applied with weights $(\sqrt{c_a})$ and $(\frac{w_a}{\sqrt{c_a}}\,x_a)$:
\[
    \underbrace{\left(\sum_a w_a\,x_a\right)^2}_{\textstyle= \mathcal{B}_s^2}
    = \left(\sum_a \sqrt{c_a}\cdot\frac{w_a\,x_a}{\sqrt{c_a}}\right)^2
    \leq \left(\sum_a c_a\right)\!\left(\sum_a \frac{w_a^2}{c_a}\,x_a^2\right)
    = \sum_a \frac{w_a^2}{c_a}\,x_a^2,
\]
where we used $\sum_a c_a = \sum_a \pi_b(a|s) = 1$.
Multiplying both sides by $d_0(s)$ and summing over $s$ yields $J_2(\mathbf{n}) \leq J_1(\mathbf{n})$, hence $J_1 - J_2 \geq 0$.

For $J_3$, pairs with $N(s,a)=0$ have $W(s,a)=0$ and contribute zero.
For a pair with $N(s,a)>0$, combining the two terms in $W(s,a)$ gives
\[
W(s,a)=\frac{\sigma_r^2(s,a)}{N(s,a)}+\frac{Q_D(s,a)}{N(s,a)^2}
=\frac{\sigma_r^2(s,a)\,N(s,a)+Q_D(s,a)}{N(s,a)^2},
\]
and the numerator expands as
\[
\sigma_r^2(s,a)\,n_f(s,a)
+\sum_{k=1}^K n_k(s,a)\,\bigl[\sigma_r^2(s,a)+\Delta_{g_k}(s,a)\bigr]
\ge 0,
\]
since $\sigma_r^2(s,a)\ge 0$, $n_f(s,a),\,n_k(s,a)\ge 0$, and $\sigma_r^2(s,a)+\Delta_{g_k}(s,a) = \operatorname{Var}(g_{k,j}|s,a)\ge 0$ for every $k$.
The denominator is positive, so $W(s,a)\ge0$; together with $\alpha_3(s,a)\ge0$, this shows that every $J_3$ summand is nonnegative and hence $J_3\ge0$.
Therefore $J(\mathbf n)=(J_1-J_2)+J_3\ge 0$.

\textit{Upper bound.}
Each count $n_k(s,a)$ is a nonnegative integer bounded by $N_s$, so $\mathcal F$ is finite and contains the zero allocation.
Common support and the zero-weight convention make all coefficients $\alpha_1$ and $\alpha_3$ finite.
For positive total counts, $N(s,a)\ge1$, so finite reward and annotation moments give finite $B(s,a)$ and $W(s,a)$; for zero counts, the fixed finite defaults give finite $B(s,a)$ and $W(s,a)=0$.
Thus $J(\mathbf n)$ is finite for every $\mathbf n\in\mathcal F$, and $J_{\max}:=\max_{\mathbf n\in\mathcal F}J(\mathbf n)<\infty$.
\end{proof}

\begin{proposition}[Non-separability across actions]
\label{prop:nonsep}
There exist valid instances with a single context $s$ and $|\mathcal A|\ge2$ for which no functions $\{f_a\}_{a\in\mathcal A}$ satisfy
\[
    J(\mathbf{n}) = \sum_{a \in \mathcal{A}} f_a\!\big(\{n_k(s,a)\}_{k\in\mathcal{K}}\big).
\]
\end{proposition}
\begin{proof}
We first identify the cross-action coupling in the continuous relaxation, then verify non-separability on feasible integer allocations.

Since $J_1$ and $J_3$ decompose over individual $(s,a)$ pairs, they contribute zero to the cross-action derivative.
The only contribution comes from $-J_2$.
By Lemma~\ref{lem:hessian}(iii):
\begin{equation}
\label{eq:cross_action}
    \frac{\partial^2 J}{\partial n_k(s,a)\,\partial n_l(s,a')} = -\frac{2\,d_0(s)\,\pi_e(a|s)\,\pi_e(a'|s)\,(E_k\,N_a - P_a)(E_l\,N_{a'} - P_{a'})}{N_a^2\,N_{a'}^2}.
\end{equation}
For pairs with $n_f(s,a),n_f(s,a')>0$, at $\mathbf n=\mathbf0$ we have $P_a=P_{a'}=0$, giving
\[
    \frac{\partial^2 J}{\partial n_k(s,a)\,\partial n_l(s,a')}\bigg|_{\mathbf{n}=\mathbf{0}} = -\frac{2\,d_0(s)\,\pi_e(a|s)\,\pi_e(a'|s)\,E_k\,E_l}{n_f(s,a)\,n_f(s,a')},
\]
which is non-zero whenever $d_0(s)>0$, $\pi_e(a|s),\pi_e(a'|s)>0$, and $E_k,E_l\ne0$.

For integer allocations, we choose one context, two actions, one source, and $\pi_b(a_i|s)=\pi_e(a_i|s)=1/2$ for $i=1,2$.
We set $n_f(s,a_i)=1$, $\varepsilon_{g_1}(s,a_i)=1$, $\sigma_r^2(s,a_i)=1$, $\Delta_{g_1}(s,a_i)=0$, $c_1=1$, and $\kappa=2$.
Then $N_s=2$ and each action has annotation capacity one, so $(0,0)$, $(1,0)$, $(0,1)$, and $(1,1)$ are feasible annotation-count vectors.
Writing $J(x,y)$ for the objective at counts $x$ and $y$ on the two actions, substitution into~\eqref{eq:final_objective} gives
\[
J(0,0)=\frac12,\qquad J(1,0)=J(0,1)=\frac7{16},\qquad J(1,1)=\frac14.
\]
Consequently, $J(1,1)-J(1,0)-J(0,1)+J(0,0)=-1/8$.
Any additive per-action decomposition would make this difference zero, proving the claim.
\end{proof}

\begin{proposition}[Non-convexity and non-concavity]
\label{prop:nonconvex}
There exist valid instances for which the continuous relaxation of $J(\mathbf n)$ is neither convex nor concave over feasible allocations.
\end{proposition}
\begin{proof}
We show non-convexity and non-concavity by exhibiting a feasible direction along which a diagonal entry of the Hessian of the continuous relaxation takes opposite signs at two configurations.

We construct an instance with one context $s$, two actions $\mathcal A=\{a_1,a_2\}$, a single annotation source ($K=1$), and $\pi_b(a_1|s)=\pi_b(a_2|s)=1/2$.
Set $\pi_e(a_1|s)=1$ and $\pi_e(a_2|s)=0$, and fix $n_1(s,a_2)=0$, so the only decision variable is $n:=n_1(s,a_1)$.
Write
\[
E:=\varepsilon_{g_1}(s,a_1),\quad D:=\Delta_{g_1}(s,a_1),\quad \sigma^2:=\sigma_r^2(s,a_1),\quad n_f:=n_f(s,a_1),\quad N:=n_f+n.
\]
We choose a positive integer $n_f$, set $n_f(s,a_2)=2n_f$, $c_1=1$, and $\kappa=2n_f$; then $N_s=3n_f$ and the feasible interval along this direction is $0\le n\le2n_f$.
Under this configuration, only $(s,a_1)$ contributes to the bias-related terms, and \(\mathcal B_s(\mathbf n)=\pi_e(a_1|s)\,B(s,a_1)=B(s,a_1)\).
Using the identity $\alpha_1(s,a_1)-d_0(s)\pi_e(a_1|s)^2=\alpha_3(s,a_1)$ (Eq.~\eqref{eq:key_identity}), the contribution of $(s,a_1)$ to $J_1-J_2$ collapses into $\alpha_3(s,a_1)\,B(s,a_1)^2$.
Thus, along this feasible direction,
\[
J(n)=\alpha_3\!\left[\frac{n^2 E^2}{N^2}+\frac{\sigma^2}{N}+\frac{n D}{N^2}\right],
\qquad \alpha_3:=\alpha_3(s,a_1)>0.
\]
Specializing the Hessian formulas of Lemma~\ref{lem:hessian}(i) and (iv) to $K=1$ (so that the mixed partial reduces to the second derivative), the diagonal entry of the Hessian is
\begin{equation}
\label{eq:Hess_diag}
J''(n)=\frac{2\alpha_3}{N^4}\!\left[E^2\,n_f(n_f-2n)+(\sigma^2-2D)\,N+3nD\right].
\end{equation}

\textit{Non-concavity.}
Evaluating \eqref{eq:Hess_diag} at $n=0$ (so $N=n_f$):
\[
J''(0)=\frac{2\alpha_3}{n_f^4}\!\left[E^2 n_f^2+(\sigma^2-2D)n_f\right]
=\frac{2\alpha_3\bigl(E^2 n_f+\sigma^2-2D\bigr)}{n_f^3}.
\]
For $\alpha_3>0$ and parameters with $E^2 n_f+\sigma^2-2D>0$ (e.g.\ $E\ne0$ and $0\le D\le\sigma^2/2$), $J''(0)>0$.
By continuity, $J''(n)>0$ at interior points of the feasible interval sufficiently close to zero, so $J$ is not concave.

\textit{Non-convexity.}
At $n=n_f$ (so $N=2n_f$):
\[
J''(n_f)=\frac{2\alpha_3}{(2n_f)^4}\!\left[-E^2 n_f^2+2(\sigma^2-2D)n_f+3n_f D\right]
=\frac{\alpha_3\bigl(-E^2 n_f+2\sigma^2-D\bigr)}{8\,n_f^3}.
\]
Choosing the bias with $E^2>(2\sigma^2-D)/n_f$ makes $J''(n_f)<0$, so the Hessian is not positive semidefinite; $J$ is not convex.

Both regimes coexist on a single instance: pick any $n_f\ge 1$, $\sigma^2>0$, $D$ with $0\le D\le \sigma^2/2$, and an $E$ satisfying $E^2>(2\sigma^2-D)/n_f$.
Then $J''(0)>0$ and $J''(n_f)<0$ simultaneously, certifying that the continuous relaxation of $J$ is neither convex nor concave.
\end{proof}

\begin{corollary}[Neither submodular nor supermodular]
\label{cor:nonsubmod}
There exist valid instances for which $J(\mathbf n)$, viewed as a function on the integer lattice $\mathbb Z_+^{|\mathcal S||\mathcal A|K}$, is neither submodular nor supermodular, with both properties violated by feasible allocations.
\end{corollary}
\begin{proof}
We test the defining lattice inequality on feasible integer allocations.
Submodularity requires $J(\mathbf x)+J(\mathbf y)\ge J(\mathbf x\wedge\mathbf y)+J(\mathbf x\vee\mathbf y)$, where $\wedge$ and $\vee$ denote coordinate-wise minimum and maximum; supermodularity requires the reverse inequality.
For distinct coordinates $i,j$ with unit vectors $\mathbf e_i,\mathbf e_j$, we define the mixed finite difference
\[
\delta_{ij}J(\mathbf n):=J(\mathbf n+\mathbf e_i+\mathbf e_j)-J(\mathbf n+\mathbf e_i)-J(\mathbf n+\mathbf e_j)+J(\mathbf n).
\]
Taking $\mathbf x=\mathbf n+\mathbf e_i$ and $\mathbf y=\mathbf n+\mathbf e_j$ shows that submodularity requires $\delta_{ij}J(\mathbf n)\le0$, whereas supermodularity requires $\delta_{ij}J(\mathbf n)\ge0$.

We construct one instance violating both requirements, with one context $s$, three actions $a_1,a_2,a_3$, and one source.
For each action, we set $\pi_b(a_i|s)=\pi_e(a_i|s)=1/3$, $n_f(s,a_i)=1$, $\sigma_r^2(s,a_i)=1$, and $\Delta_{g_1}(s,a_i)=0$.
We write $\varepsilon_i:=\varepsilon_{g_1}(s,a_i)$ and choose $(\varepsilon_1,\varepsilon_2,\varepsilon_3)=(1,1,-1)$.
These moments can be realized by independent rewards taking values $1$ and $3$ with equal probability, with each annotation distributed as an independent reward draw plus the corresponding bias.
We set $c_1=1$ and $\kappa=2$, giving $N_s=3$ and per-action capacities $N_s-n_f(s,a_i)=2$.
Thus $\mathbf0$, $\mathbf e_i$, $\mathbf e_j$, and $\mathbf e_i+\mathbf e_j$ are feasible for every $i\ne j$, where $\mathbf e_i$ now denotes one annotation for action $a_i$.

The terms $J_1$ and $J_3$ are sums of per-action functions, so their mixed finite differences vanish.
At the zero allocation, all reward-model biases are zero; adding one annotation for action $a_i$ gives $B(s,a_i)=\varepsilon_i/2$.
Consequently, $\mathcal B_s(\mathbf0)=0$, $\mathcal B_s(\mathbf e_i)=\varepsilon_i/6$, and $\mathcal B_s(\mathbf e_i+\mathbf e_j)=(\varepsilon_i+\varepsilon_j)/6$.
Only $-J_2=-\mathcal B_s^2$ contributes to the mixed difference, yielding
\[
\delta_{ij}J(\mathbf0)
=-\frac{(\varepsilon_i+\varepsilon_j)^2-\varepsilon_i^2-\varepsilon_j^2}{36}
=-\frac{\varepsilon_i\varepsilon_j}{18}.
\]
Hence $\delta_{12}J(\mathbf0)=-1/18<0$ violates supermodularity, while $\delta_{13}J(\mathbf0)=1/18>0$ violates submodularity on the same instance.
The same violations persist with additional sources by setting their annotation counts to zero.
\end{proof}

\section{Analysis of the Majorization-Minimization Algorithm}
\label{sec:appendix_mm}

This appendix provides the theoretical analysis of the proposed MM framework.
Throughout this section, we use the notation established in \S\ref{sec:B1_decomp}.
In particular, $\mathcal{B}_s(\mathbf{n}) := \sum_{a \in \mathcal{A}} \pi_e(a|s) B(s,a)$ denotes the context-level bias accumulation, with $B$ and $W$ defined in~\eqref{eq:final_objective} and its preceding text, including the default-model case $N(s,a)=0$.
We write $\mathcal{F}$ for the full feasible set defined by all constraints in~\eqref{eq:optimization_problem}, and $\mathcal{F}_0$ for the feasible set with the budget constraint removed.

\subsection{Pseudocode of the Algorithms}
\label{sec: pseudo}

Algorithm~\ref{alg:main} retains the best feasible allocation found at each majorization-minimization (MM) iteration, while Algorithm~\ref{alg:subproblem} generates candidates by local source search and dynamic programming (DP).
Both algorithms use factual fitting counts, reward moments, source profiles, policy probabilities, default rewards, and costs from the allocation problem; these fixed inputs are suppressed in subroutine calls.
The MM coefficients are $\omega_1(s,a)=\alpha_1(s,a)$, $\omega_3(s,a)=\alpha_3(s,a)$, and $\omega_2^{(t)}(s,a)=-2d_0(s)\mathcal B_s(\mathbf n^{(t)})\pi_e(a|s)$.
The local search and DP tables use the complete penalized cost $\psi_{s,a}^{(t)}(\mathbf z;\lambda):=\tilde J_{s,a}^{(t)}(\mathbf z)+\lambda\sum_k c_kz_k$, where $\tilde J_{s,a}^{(t)}$ is the summand defined in \S\ref{sec:surrogate_solver}.
For each total count $m$, the subproblem solver records the largest remaining positive swap gain $\delta_{s,a}(m,\lambda)$, with zero assigned when no improving swap exists.
These residuals certify the approximation error even when the swap cap is reached (Proposition~\ref{prop:greedy_opt}).
All ties in source selection, swaps, DP, and candidate ordering use a fixed index order, and the outer update retains the current allocation unless the surrogate strictly improves.

\begin{algorithm}[htbp]
\caption{Majorization-minimization (MM) allocation with feasible-candidate retention}
\label{alg:main}
\small
\begin{algorithmic}[1]
\Require Fixed problem inputs above, budget $\kappa$, feasible $\mathbf n^{(0)}$ (default $\mathbf0$), iteration limits $T_{\max},T_{\mathrm{swap}}\ge1$, tolerances $\epsilon\ge0$, $\epsilon_\lambda>0$, initial multiplier $\lambda_0>0$.
\Ensure Feasible allocation with objective no larger than $J(\mathbf n^{(0)})$.
\For{$t = 0, 1, \dots, T_{\max}-1$}
    \State Compute MM coefficients $\boldsymbol{\omega}^{(t)}$ from $\mathbf{n}^{(t)}$.
    \State $\mathbf n_{\mathrm{best}}\leftarrow\mathbf n^{(t)}$; $\mathcal Q_t\leftarrow\varnothing$.
    \State $(\mathbf n_0,\boldsymbol\delta_0)\leftarrow\textsc{DP}(0,\boldsymbol\omega^{(t)},T_{\mathrm{swap}})$.
    \State If $C(\mathbf n_0)\le\kappa$, add $(0,\mathbf n_0,\boldsymbol\delta_0)$ to $\mathcal Q_t$.
    \State $\lambda_{\mathrm{low}}\leftarrow0$; $\lambda_{\mathrm{high}}\leftarrow\lambda_0$.
    \State $(\mathbf n_{\mathrm{high}},\boldsymbol\delta_{\mathrm{high}})\leftarrow\textsc{DP}(\lambda_{\mathrm{high}},\boldsymbol\omega^{(t)},T_{\mathrm{swap}})$.
    \While{$C(\mathbf n_{\mathrm{high}})>\kappa$}
        \State $\lambda_{\mathrm{low}}\leftarrow\lambda_{\mathrm{high}}$; $\lambda_{\mathrm{high}}\leftarrow2\lambda_{\mathrm{high}}$.
        \State $(\mathbf n_{\mathrm{high}},\boldsymbol\delta_{\mathrm{high}})\leftarrow\textsc{DP}(\lambda_{\mathrm{high}},\boldsymbol\omega^{(t)},T_{\mathrm{swap}})$.
    \EndWhile
    \State Add $(\lambda_{\mathrm{high}},\mathbf n_{\mathrm{high}},\boldsymbol\delta_{\mathrm{high}})$ to $\mathcal Q_t$.
    \While{$\lambda_{\mathrm{high}}-\lambda_{\mathrm{low}}>\epsilon_\lambda$}
        \State $\lambda_{\mathrm{mid}}\leftarrow(\lambda_{\mathrm{low}}+\lambda_{\mathrm{high}})/2$.
        \State $(\mathbf n',\boldsymbol\delta')\leftarrow\textsc{DP}(\lambda_{\mathrm{mid}},\boldsymbol\omega^{(t)},T_{\mathrm{swap}})$.
        \If{$C(\mathbf n')>\kappa$}
            \State $\lambda_{\mathrm{low}}\leftarrow\lambda_{\mathrm{mid}}$.
        \Else
            \State $\lambda_{\mathrm{high}}\leftarrow\lambda_{\mathrm{mid}}$.
            \State Add $(\lambda_{\mathrm{mid}},\mathbf n',\boldsymbol\delta')$ to $\mathcal Q_t$.
        \EndIf
    \EndWhile
    \For{each $(\lambda,\mathbf n',\boldsymbol\delta')\in\mathcal Q_t$}
        \If{$\tilde J(\mathbf n';\mathbf n^{(t)})<\tilde J(\mathbf n_{\mathrm{best}};\mathbf n^{(t)})$}
            \State $\mathbf n_{\mathrm{best}}\leftarrow\mathbf n'$.
        \EndIf
    \EndFor
    \State $\mathbf n^{(t+1)}\leftarrow\mathbf n_{\mathrm{best}}$.
    \If{$\mathbf n^{(t+1)}=\mathbf n^{(t)}$ \textbf{or} $J(\mathbf n^{(t)})-J(\mathbf n^{(t+1)})\le\epsilon$}
        \State \textbf{break}
    \EndIf
\EndFor
\State \Return $\mathbf{n}^{(t+1)}$
\end{algorithmic}
\end{algorithm}

\begin{algorithm}[htbp]
\caption{Capped source search and dynamic programming (\textsc{DP})}
\label{alg:subproblem}
\small
\begin{algorithmic}[1]
\Require Fixed problem inputs, multiplier $\lambda\ge0$, MM coefficients $\boldsymbol\omega^{(t)}$, swap cap $T_{\mathrm{swap}}\ge1$.
\Ensure Capacity-feasible allocation $\mathbf n\in\mathcal F_0$ and table of residual swap gains $\boldsymbol\delta$.
\For{each context $s \in \mathcal{S}$}
    \For{each action $a \in \mathcal{A}$}
        \State $M_{s,a}\leftarrow N_s-n_f(s,a)$.
        \For{$m=0,\ldots,M_{s,a}$}
            \State Initialize $\mathbf z$ with $m$ annotations from the cheapest source.
            \If{$m>0$ \textbf{and} $K>1$}
                \For{$i=1,\ldots,T_{\mathrm{swap}}$}
                    \State Choose $(u,v)$, $u\ne v$, $z_u\ge1$, maximizing $g=\psi_{s,a}^{(t)}(\mathbf z;\lambda)-\psi_{s,a}^{(t)}(\mathbf z+\mathbf e_v-\mathbf e_u;\lambda)$.
                    \State If $g\le0$, \textbf{break}.
                    \State $\mathbf z\leftarrow\mathbf z+\mathbf e_v-\mathbf e_u$.
                \EndFor
            \EndIf
            \State Store $\mathbf n^\sharp_{s,a}[m]\leftarrow\mathbf z$ and $H^\sharp_{s,a}[m]\leftarrow\psi_{s,a}^{(t)}(\mathbf z;\lambda)$.
            \State Store $\delta_{s,a}(m,\lambda)$ as the largest remaining positive swap gain (zero if none).
        \EndFor
    \EndFor
    \State Order actions as $a_1,\ldots,a_{|\mathcal A|}$; set $D_0(0)=0$ and $D_0(w)=\infty$ for $w>0$.
    \For{$j=1,\ldots,|\mathcal A|$}
        \For{$w=0,\ldots,N_s$}
            \State $D_j(w)\leftarrow\min_{0\le m\le\min\{w,M_{s,a_j}\}}\{D_{j-1}(w-m)+H^\sharp_{s,a_j}[m]\}$.
            \State Store a minimizing count $m_j(w)$.
        \EndFor
    \EndFor
    \State $w^\star\in\arg\min_{0\le w\le N_s}D_{|\mathcal A|}(w)$; set $w\leftarrow w^\star$.
    \For{$j=|\mathcal A|,\ldots,1$}
        \State $m\leftarrow m_j(w)$; $\mathbf n_{s,a_j}\leftarrow\mathbf n^\sharp_{s,a_j}[m]$; $w\leftarrow w-m$.
    \EndFor
\EndFor
\State \Return $\mathbf n,\boldsymbol\delta$.
\end{algorithmic}
\end{algorithm}

\clearpage

\subsection{Surrogate Construction and Majorization Property}
\label{app: mm}

The MM framework replaces the non-convex objective $J(\mathbf{n})$ with a locally tight upper bound.
To remove the cross-action coupling, we majorize $-J_2(\mathbf{n}) = -\sum_s d_0(s)\,\mathcal{B}_s(\mathbf{n})^2$ while retaining $J_1$ and $J_3$.
We apply the tangent majorant of the concave scalar function $h(x)=-x^2$ at $x=\mathcal B_s(\mathbf n^{(t)})$, giving an expression affine in $\mathcal B_s(\mathbf n)$.

\begin{lemma}[Valid Majorization]
\label{lem:majorization}
For any fixed feasible allocation $\mathbf n^{(t)}$, the full surrogate $\tilde{J}_{\mathrm{full}}(\mathbf{n};\,\mathbf{n}^{(t)})$ defined in Proposition~\ref{prop:mm_majorant_main} of \S\ref{sec:mm_surrogate} satisfies:
\begin{enumerate}
    \item[(i)] \textbf{(Upper bound)} $\tilde{J}_{\mathrm{full}}(\mathbf{n};\,\mathbf{n}^{(t)}) \geq J(\mathbf{n})$ for all $\mathbf{n} \in \mathcal{F}$.
    \item[(ii)] \textbf{(Tangency)} $\tilde J_{\mathrm{full}}(\mathbf n^{(t)};\mathbf n^{(t)}) = J(\mathbf{n}^{(t)})$.
\end{enumerate}
Moreover, the majorization gap admits the closed-form expression:
\begin{equation}
\label{eq:mm_gap_formula}
    \tilde{J}_{\mathrm{full}}(\mathbf{n};\,\mathbf{n}^{(t)}) - J(\mathbf{n}) = \sum_{s \in \mathcal{S}} d_0(s)\bigl(\mathcal{B}_s(\mathbf{n}) - \mathcal{B}_s(\mathbf{n}^{(t)})\bigr)^2.
\end{equation}
\end{lemma}

\begin{proof}
We fix $\mathbf n^{(t)}\in\mathcal F$ and any $\mathbf n\in\mathcal F$.
Because $h(x)=-x^2$ is concave, its tangent at any $x_0\in\mathbb R$ gives a global upper bound:
\begin{equation*}
    -x^2 \;\leq\; -x_0^2 - 2x_0(x-x_0) \;=\; x_0^2 - 2x_0\,x, \qquad \forall\, x \in \mathbb{R},
\end{equation*}
with equality if and only if $x = x_0$.
Substituting $x = \mathcal{B}_s(\mathbf{n})$ and $x_0 = \mathcal{B}_s(\mathbf{n}^{(t)})$, multiplying by the nonnegative weights $d_0(s)$, and summing over contexts gives
\[
    -\sum_s d_0(s)\,\mathcal{B}_s(\mathbf{n})^2 \;\leq\; \sum_s d_0(s)\bigl[\mathcal{B}_s(\mathbf{n}^{(t)})^2 - 2\,\mathcal{B}_s(\mathbf{n}^{(t)})\,\mathcal{B}_s(\mathbf{n})\bigr].
\]
The full surrogate (including the constant term) is:
\begin{equation}
\label{eq:Jtilde_full}
    \tilde{J}_{\mathrm{full}}(\mathbf{n};\,\mathbf{n}^{(t)}) = J_1(\mathbf{n}) + \sum_s d_0(s)\bigl[\mathcal{B}_s(\mathbf{n}^{(t)})^2 - 2\,\mathcal{B}_s(\mathbf{n}^{(t)})\,\mathcal{B}_s(\mathbf{n})\bigr] + J_3(\mathbf{n}).
\end{equation}
\emph{Property (i)} follows from $-J_2(\mathbf{n}) \leq \tilde{J}_{\mathrm{full}}(\mathbf{n};\,\mathbf{n}^{(t)}) - J_1(\mathbf{n}) - J_3(\mathbf{n})$.
At $\mathbf n=\mathbf n^{(t)}$, the summed linearized term in~\eqref{eq:Jtilde_full} equals
\[
    \sum_s d_0(s)\bigl[(\mathcal B_s^{(t)})^2-2(\mathcal B_s^{(t)})^2\bigr]
    =-\sum_s d_0(s)(\mathcal B_s^{(t)})^2
    =-J_2(\mathbf n^{(t)}).
\]
Substitution into~\eqref{eq:Jtilde_full} gives $\tilde J_{\mathrm{full}}(\mathbf n^{(t)};\mathbf n^{(t)})=J_1(\mathbf n^{(t)})-J_2(\mathbf n^{(t)})+J_3(\mathbf n^{(t)})=J(\mathbf n^{(t)})$, proving \emph{property (ii)}.

For the gap formula, we compute:
\begin{align*}
    &\tilde{J}_{\mathrm{full}}(\mathbf{n};\,\mathbf{n}^{(t)}) - J(\mathbf{n}) \\
    &= \sum_s d_0(s)\bigl[\mathcal{B}_s(\mathbf{n}^{(t)})^2 - 2\mathcal{B}_s(\mathbf{n}^{(t)})\mathcal{B}_s(\mathbf{n}) + \mathcal{B}_s(\mathbf{n})^2\bigr] \\
    &= \sum_s d_0(s)\bigl(\mathcal{B}_s(\mathbf{n}) - \mathcal{B}_s(\mathbf{n}^{(t)})\bigr)^2.
\end{align*}
The optimization surrogate in~\eqref{eq:surrogate_def} satisfies
\[
    \tilde J(\mathbf n;\mathbf n^{(t)})
    =\tilde J_{\mathrm{full}}(\mathbf n;\mathbf n^{(t)})-J_2(\mathbf n^{(t)}).
\]
Because $J_2(\mathbf n^{(t)})$ is independent of $\mathbf n$, both surrogates have the same minimizers and give the same comparisons between allocations at fixed $t$.
At the current allocation, however, $\tilde J(\mathbf n^{(t)};\mathbf n^{(t)})=J(\mathbf n^{(t)})-J_2(\mathbf n^{(t)})$.
\end{proof}

\begin{proposition}[Monotone Descent under Surrogate-Improving Updates]
\label{prop:mm_descent}
Suppose at every iteration $t\ge 0$ the update $\mathbf{n}^{(t+1)}\in\mathcal{F}$ satisfies the \emph{surrogate-improvement} condition
\begin{equation}
\label{eq:surrogate_improving}
    \tilde{J}\!\bigl(\mathbf{n}^{(t+1)};\,\mathbf{n}^{(t)}\bigr)
    \;\le\;
    \tilde{J}\!\bigl(\mathbf{n}^{(t)};\,\mathbf{n}^{(t)}\bigr).
\end{equation}
Then the original objective is monotonically non-increasing,
\begin{equation*}
    J\!\bigl(\mathbf{n}^{(t+1)}\bigr) \;\le\; J\!\bigl(\mathbf{n}^{(t)}\bigr),
    \qquad \forall\, t\ge 0,
\end{equation*}
and the sequence $\{J(\mathbf{n}^{(t)})\}$ converges.
The fallback rule $\mathbf{n}^{(t+1)} \leftarrow \mathbf{n}^{(t)}$ used in Algorithm~\ref{alg:main} ensures \eqref{eq:surrogate_improving} regardless of how accurately the inner Lagrangian subproblem is solved; in particular, exact surrogate minimization is \emph{not} required for descent of $J$.
\end{proposition}

\begin{proof}
At fixed $t$, adding $J_2(\mathbf n^{(t)})$ to both sides of~\eqref{eq:surrogate_improving} gives the same improvement inequality for $\tilde J_{\mathrm{full}}$.
Combining this inequality with the upper bound and tangency in Lemma~\ref{lem:majorization} yields
\[
\begin{aligned}
    J(\mathbf{n}^{(t+1)})
    &\le \tilde{J}_{\mathrm{full}}(\mathbf{n}^{(t+1)};\mathbf{n}^{(t)}) \\
    &\le \tilde{J}_{\mathrm{full}}(\mathbf{n}^{(t)};\mathbf{n}^{(t)}) \\
    &= J(\mathbf{n}^{(t)}).
\end{aligned}
\]
The fallback update satisfies~\eqref{eq:surrogate_improving} with equality.
Proposition~\ref{prop:nonneg} gives $J\ge0$, so the non-increasing objective sequence converges.
\end{proof}

\begin{remark}[Reconciliation with Theorem~\ref{thm:per_iter_error}]
\label{rem:reconciliation}
Proposition~\ref{prop:mm_descent} only certifies monotonic descent of $J$ along the iterates produced by Algorithm~\ref{alg:main}.
It is \emph{independent} of the suboptimality of $\mathbf{n}^{(t+1)}$ relative to the exact constrained minimizer $\hat{\mathbf{n}}^{(t+1)}=\arg\min_{\mathcal F}\tilde J(\cdot;\mathbf n^{(t)})$.
Theorem~\ref{thm:per_iter_error} subsequently bounds the gap to the \emph{global} minimizer $\mathbf{n}^*$ of $J$, and that bound naturally incorporates Inner-loop suboptimality (Proposition~\ref{prop:lag_gap}).
The two results therefore play complementary roles: Proposition~\ref{prop:mm_descent} guarantees algorithmic stability, while Theorem~\ref{thm:per_iter_error} quantifies how close each iterate may be to a global optimum.
\end{remark}

\subsection{Local Source Search and Exact Count Aggregation}
\label{sec:C2}

At a fixed multiplier $\lambda$, separability allows us to search source mixtures independently and then aggregate their costs across actions.
We bound the local-search error before propagating it through the exact DP aggregation.

\begin{lemma}[Separability of the Surrogate]
\label{lem:separability}
The surrogate objective $\tilde{J}(\mathbf{n};\,\mathbf{n}^{(t)})$ decomposes as:
\[
    \tilde{J}(\mathbf{n};\,\mathbf{n}^{(t)}) = \sum_{s \in \mathcal{S}} \sum_{a \in \mathcal{A}} \phi_{s,a}^{(t)}\!\bigl(\{n_k(s,a)\}_{k=1}^K\bigr).
\]
Here $\phi_{s,a}^{(t)}=\tilde J_{s,a}^{(t)}$ is the local summand from \S\ref{sec:surrogate_solver} and depends only on the annotation counts for $(s,a)$.
\end{lemma}

\begin{proof}
The terms $J_1$ and $J_3$ are already sums over individual $(s,a)$ pairs.
The linearized $-J_2$ contribution to the surrogate is:
\[
    -2\sum_s d_0(s)\,\mathcal{B}_s(\mathbf{n}^{(t)})\!\sum_a \pi_e(a|s)\,B(s,a) = \sum_s \sum_a \omega_2^{(t)}(s,a)\,B(s,a),
\]
where $\omega_2^{(t)}(s,a) = -2\,d_0(s)\,\mathcal{B}_s(\mathbf{n}^{(t)})\,\pi_e(a|s)$ is a constant with respect to $\mathbf{n}$.
Because each $B(s,a)$ depends only on $\{n_k(s,a)\}_k$, the optimization surrogate $\tilde J$ is separable.
\end{proof}

\begin{lemma}[Convexity of the Local Cost for Fixed Total Count]
\label{lem:convex_fixed_m}
Fix an MM iteration $t$, a pair $(s,a)$, a multiplier $\lambda\ge0$, and a total count $m$ with $N_0:=n_f(s,a)+m>0$.
Holding $N_0$ constant, the complete local cost is
\begin{equation}
\label{eq:psi_local}
\begin{aligned}
    \psi_{s,a}^{(t)}(\mathbf n_K;\lambda)
    &= \frac{\omega_1}{N_0^2}\left(\sum_k n_kE_k\right)^2
       +\frac{\omega_2^{(t)}}{N_0}\sum_k n_kE_k \\
    &\quad+\omega_3\left(\frac{\sigma_r^2(s,a)}{N_0}
       +\frac{\sum_k n_kD_k}{N_0^2}\right)
       +\lambda\sum_k c_kn_k,
\end{aligned}
\end{equation}
and is convex in $\mathbf n_K:=(n_1,\ldots,n_K)^\top$ on the simplex $\{\mathbf n_K\ge\mathbf0,\ \mathbf1^\top\mathbf n_K=m\}$.
The coefficients $\omega_1=\alpha_1(s,a)$, $\omega_3=\alpha_3(s,a)$, and $\omega_2^{(t)}=-2d_0(s)\mathcal B_s(\mathbf n^{(t)})\pi_e(a|s)$ are fixed during this search.
\end{lemma}

\begin{proof}
The Hessian of $\psi$ with respect to $\mathbf{n}_K$ is
\[
    \mathbf{H}_\psi = \frac{2\omega_1}{N_0^2}\,\mathbf{e}\,\mathbf{e}^\top, \qquad \mathbf{e} = (E_1,\ldots,E_K)^\top.
\]
For every direction $\mathbf v\in\mathbb R^K$, $\mathbf v^\top\mathbf H_\psi\mathbf v=2\omega_1(\mathbf e^\top\mathbf v)^2/N_0^2\ge0$ because $\omega_1=\alpha_1(s,a)\ge0$.
Thus the quadratic extension with fixed denominator is convex, as is its restriction to the simplex.
Strict convexity is not required and need not hold because the Hessian has rank at most one.
\end{proof}

The term $\omega_3\sigma_r^2(s,a)/N_0$ is constant across source mixtures at fixed $m$, but varies across counts and must be retained in $H^\sharp_{s,a}(m,\lambda)$.
For $m=0$, the only source vector is $\mathbf0$, so Algorithm~\ref{alg:subproblem} stores its complete local cost without a swap search.
If also $n_f(s,a)=0$, that cost is $\omega_1(r_0-\mu_r)^2+\omega_2^{(t)}(r_0-\mu_r)$, using the default bias and zero reward-model variance.

\begin{proposition}[Pairwise-Exchange Termination and Approximation Error]
\label{prop:greedy_opt}
Fix $t$, $(s,a)$, $m\ge1$, and $\lambda\ge0$.
The source search in Algorithm~\ref{alg:subproblem} accepts at most $\min\{T_{\mathrm{swap}},\binom{m+K-1}{K-1}-1\}$ swaps.
For its output $\mathbf n^\sharp$, let $H^\sharp_{s,a}(m,\lambda):=\psi_{s,a}^{(t)}(\mathbf n^\sharp;\lambda)$ and define the residual gain
\[
    \delta_{s,a}(m,\lambda):=
    \max_{\substack{u,v\in\mathcal K\\n_u^\sharp\ge1}}
    \left[\psi_{s,a}^{(t)}(\mathbf n^\sharp;\lambda)
    -\psi_{s,a}^{(t)}(\mathbf n^\sharp+\mathbf e_v-\mathbf e_u;\lambda)\right]_+,
\]
where $[x]_+:=\max\{x,0\}$ and allowing $u=v$ includes a zero-gain move when $K=1$.
Relative to the exact fixed-count minimum $H_{s,a}^{(t,\lambda)}(m)$, the output satisfies
\begin{equation}
\label{eq:greedy_bound}
\begin{aligned}
    0&\le H^\sharp_{s,a}(m,\lambda)-H_{s,a}^{(t,\lambda)}(m)
    \le\eta_{s,a}(m)+m\,\delta_{s,a}(m,\lambda),\\
    \eta_{s,a}(m)&:=\frac{\alpha_1(s,a)m\Delta E_{s,a}^{\,2}}{(n_f(s,a)+m)^2},
    \qquad
    \Delta E_{s,a}:=\max_k E_k-\min_k E_k.
\end{aligned}
\end{equation}
The output is swap-local if and only if $\delta_{s,a}(m,\lambda)=0$.
With no swap cap, strict descent reaches such an output after at most $\binom{m+K-1}{K-1}-1$ accepted swaps.
For $m=0$, the local gap is zero, and we set $\eta_{s,a}(0)=\delta_{s,a}(0,\lambda)=0$.
When $n_f(s,a)>0$, $\eta_{s,a}(m)\le\alpha_1(s,a)\Delta E_{s,a}^{\,2}/[4n_f(s,a)]$ for every $m$.
\end{proposition}

\begin{proof}
\emph{Finite termination.}
The integer simplex contains $\binom{m+K-1}{K-1}$ configurations.
Each accepted swap strictly decreases $\psi_{s,a}^{(t)}(\cdot;\lambda)$, so it cannot revisit a configuration.
Including the initial configuration gives at most $\binom{m+K-1}{K-1}-1$ accepted swaps, and the cap imposes the additional limit $T_{\mathrm{swap}}$.
Without a cap, the search stops only when there is no positive swap gain; this is precisely $\delta_{s,a}(m,\lambda)=0$.
For $K=1$ there are no swaps, and the sole source mixture is optimal.

\emph{Approximation error.}
With $m\ge1$, the denominator $N_0=n_f(s,a)+m$ is positive even if the factual count is zero.
We write $\psi(\mathbf n_K)=\omega_1(\mathbf e^\top\mathbf n_K)^2/N_0^2+\mathbf L^\top\mathbf n_K+C$, where $\mathbf e=(E_1,\ldots,E_K)^\top$ and $\mathbf L,C$ collect the fixed-count linear and constant terms in~\eqref{eq:psi_local}.
Its Hessian is $\mathbf H=2\omega_1\mathbf e\mathbf e^\top/N_0^2\succeq0$.

For any feasible swap $\mathbf d_{u,v}=\mathbf e_v-\mathbf e_u$, the residual definition and the exact quadratic expansion give
\begin{equation}
\label{eq:swap_grad_upper}
    -\nabla\psi(\mathbf n^{\sharp})^{\!\top}\mathbf d_{u,v}
    \le\delta_{s,a}(m,\lambda)+\tfrac12\mathbf d_{u,v}^\top\mathbf H\mathbf d_{u,v}
    =\delta_{s,a}(m,\lambda)+\frac{\omega_1(E_v-E_u)^2}{N_0^2}.
\end{equation}
Let $\mathbf z^*$ be a global minimizer on the same integer simplex.
Moving units from coordinates with $n_u^\sharp>z_u^*$ to coordinates with $n_v^\sharp<z_v^*$ gives $\mathbf z^*-\mathbf n^\sharp=\sum_{u,v}q_{uv}\mathbf d_{u,v}$, where $q_{uv}$ are nonnegative integers and every donor has $n_u^\sharp\ge1$.
The total moved mass is $\sum_{u,v}q_{uv}=\tfrac12\|\mathbf z^*-\mathbf n^\sharp\|_1\le m$.
Convexity and~\eqref{eq:swap_grad_upper} therefore yield
\[
\begin{aligned}
    \psi(\mathbf n^\sharp)-\psi(\mathbf z^*)
    &\le-\nabla\psi(\mathbf n^\sharp)^\top(\mathbf z^*-\mathbf n^\sharp)\\
    &\le\sum_{u,v}q_{uv}\left(\delta_{s,a}(m,\lambda)+\frac{\omega_1(E_v-E_u)^2}{N_0^2}\right)\\
    &\le m\,\delta_{s,a}(m,\lambda)+\frac{\omega_1m\Delta E_{s,a}^{\,2}}{N_0^2}.
\end{aligned}
\]
The lower bound follows because $\mathbf n^\sharp$ is a feasible source mixture.
Substituting $\omega_1=\alpha_1(s,a)$ proves~\eqref{eq:greedy_bound}.
For $m=0$, the simplex is the singleton $\{\mathbf0\}$, so the exact and stored costs coincide.
Finally, $(n_f+m)^2\ge4n_fm$ gives the uniform bound when $n_f>0$.
The count-specific bound remains valid for $n_f=0$ and $m\ge1$, where $\eta_{s,a}(m)=\alpha_1(s,a)\Delta E_{s,a}^{\,2}/m$.
\end{proof}

\subsection{Lagrangian Relaxation Analysis}
\label{sec:C3}

The global budget constraint $\sum_{s,a,k} c_k\,n_k(s,a) \leq \kappa$ couples the subproblems across contexts.
We relax it via a Lagrange multiplier $\lambda \geq 0$, yielding the dual function:
\[
    q(\lambda) = \min_{\mathbf{n} \in \mathcal{F}_0} \Bigl\{\tilde{J}(\mathbf{n};\,\mathbf{n}^{(t)}) + \lambda\bigl(\textstyle\sum_{s,a,k}c_k\,n_k(s,a) - \kappa\bigr)\Bigr\}.
\]
By Lemma~\ref{lem:separability}, the Lagrangian problem decomposes over contexts.
Algorithm~\ref{alg:subproblem} solves the count aggregation exactly for its supplied $H^\sharp$ tables; the local approximation errors determine its gap to the exact Lagrangian minimum.

The multiplier search always produces a feasible candidate in exact arithmetic.
Indeed, $\mathcal F_0$ is finite, $\mathbf0$ is represented in every set of DP tables, and $c_{\min}:=\min_k c_k>0$.
For fixed $t$, every nonzero allocation has cost at least $c_{\min}$, so any
\[
    \lambda>
    \frac{\tilde J(\mathbf0;\mathbf n^{(t)})
    -\min_{\mathbf n\in\mathcal F_0}\tilde J(\mathbf n;\mathbf n^{(t)})}{c_{\min}}
\]
makes the zero allocation strictly better than every nonzero allocation in the penalized objective.
Thus doubling eventually returns $\mathbf0$, or terminates earlier at another budget-feasible candidate.
Subsequent bisection is a finite candidate search, without a monotonicity guarantee for the costs returned by approximate local solves.

\begin{proposition}[Inner-loop suboptimality]
\label{prop:lag_gap}
Fix $t$ and let $\hat{\mathbf n}^{(t+1)}\in\arg\min_{\mathbf n\in\mathcal F}\tilde J(\mathbf n;\mathbf n^{(t)})$.
Choose any recorded feasible trial $(\lambda^{(t)},\mathbf n_{\mathrm{cand}}^{(t)},\boldsymbol\delta^{(t)})\in\mathcal Q_t$ from Algorithm~\ref{alg:main}.
To bound the aggregate local-search error, define
\[
    E_{\mathrm{Greedy}}^{(t)}
    :=\sum_{s,a}\max_{0\le m\le M_{s,a}}
    \left\{\eta_{s,a}(m)+m\,\delta_{s,a}(m,\lambda^{(t)})\right\},
\]
where the maxima are over integer counts and $M_{s,a}=N_s-n_f(s,a)$.
The actual outer update, including the current-allocation fallback, satisfies
\begin{equation}
\label{eq:lag_bound}
\begin{aligned}
    0\le\Delta_{\mathrm{Lag}}^{(t)}
    &:=\tilde J(\mathbf n^{(t+1)};\mathbf n^{(t)})
       -\tilde J(\hat{\mathbf n}^{(t+1)};\mathbf n^{(t)})\\
    &\le\lambda^{(t)}\bigl(\kappa-C(\mathbf n_{\mathrm{cand}}^{(t)})\bigr)
       +E_{\mathrm{Greedy}}^{(t)}.
\end{aligned}
\end{equation}
If all local residuals in that trial vanish and all factual counts are positive, then $E_{\mathrm{Greedy}}^{(t)}\le\sum_{s,a}\alpha_1(s,a)\Delta E_{s,a}^{\,2}/[4n_f(s,a)]$.
\end{proposition}

\begin{proof}
Feasibility of $\mathbf n^{(t+1)}$ and optimality of $\hat{\mathbf n}^{(t+1)}$ give $\Delta_{\mathrm{Lag}}^{(t)}\ge0$.
For the upper bound, fix the selected trial multiplier $\lambda=\lambda^{(t)}$ and write $L(\mathbf n;\lambda):=\tilde J(\mathbf n;\mathbf n^{(t)})+\lambda C(\mathbf n)$.

First, $D_j(w)$ in Algorithm~\ref{alg:subproblem} is the minimum stored cost for exactly $w$ annotations across the first $j$ actions.
The initialization gives this interpretation for $j=0$.
For stage $j$, any feasible assignment uses some $m\le\min\{w,M_{s,a_j}\}$ on action $a_j$, leaving $w-m$ annotations for the first $j-1$ actions.
Minimizing over these choices proves the interpretation inductively and supplies valid backtracking counts.
Selecting the best $w\le N_s$ enforces the context capacity inequality.
Assembling the independent context solutions therefore minimizes the sum of stored costs over $\mathcal F_0$.

Let $\tilde{\mathbf n}_\lambda\in\arg\min_{\mathbf n\in\mathcal F_0}L(\mathbf n;\lambda)$ and let $\tilde m_{s,a}:=\sum_k\tilde n_{\lambda,k}(s,a)$ be its counts.
By exact aggregation and Proposition~\ref{prop:greedy_opt},
\[
\begin{aligned}
    L(\mathbf n_{\mathrm{cand}}^{(t)};\lambda)
    &\le\sum_{s,a}H^\sharp_{s,a}(\tilde m_{s,a},\lambda)\\
    &\le\sum_{s,a}\left[
       H_{s,a}^{(t,\lambda)}(\tilde m_{s,a})
       +\eta_{s,a}(\tilde m_{s,a})
       +\tilde m_{s,a}\delta_{s,a}(\tilde m_{s,a},\lambda)\right]\\
    &\le L(\tilde{\mathbf n}_\lambda;\lambda)+E_{\mathrm{Greedy}}^{(t)}.
\end{aligned}
\]
The last inequality uses the fact that the exact fixed-count costs are no greater than the costs of the corresponding mixtures in $\tilde{\mathbf n}_\lambda$.
Because $\hat{\mathbf n}^{(t+1)}\in\mathcal F\subseteq\mathcal F_0$, this gives
\begin{equation}
\label{eq:greedy_lag}
    L(\mathbf n_{\mathrm{cand}}^{(t)};\lambda)
    \le L(\tilde{\mathbf n}_\lambda;\lambda)+E_{\mathrm{Greedy}}^{(t)}
    \le L(\hat{\mathbf n}^{(t+1)};\lambda)+E_{\mathrm{Greedy}}^{(t)}.
\end{equation}
The outer comparison ensures $\tilde J(\mathbf n^{(t+1)};\mathbf n^{(t)})\le\tilde J(\mathbf n_{\mathrm{cand}}^{(t)};\mathbf n^{(t)})$, even when the current allocation wins.
Rearranging~\eqref{eq:greedy_lag} and using $C(\hat{\mathbf n}^{(t+1)})\le\kappa$ yields
\[
\begin{aligned}
    \Delta_{\mathrm{Lag}}^{(t)}
    &\le\tilde J(\mathbf n_{\mathrm{cand}}^{(t)};\mathbf n^{(t)})
       -\tilde J(\hat{\mathbf n}^{(t+1)};\mathbf n^{(t)})\\
    &\le\lambda\bigl[C(\hat{\mathbf n}^{(t+1)})-C(\mathbf n_{\mathrm{cand}}^{(t)})\bigr]
       +E_{\mathrm{Greedy}}^{(t)}\\
    &\le\lambda\bigl[\kappa-C(\mathbf n_{\mathrm{cand}}^{(t)})\bigr]+E_{\mathrm{Greedy}}^{(t)}.
\end{aligned}
\]
If every residual is zero and every factual count is positive, the uniform local bound in Proposition~\ref{prop:greedy_opt} gives the stated bound on $E_{\mathrm{Greedy}}^{(t)}$.
\end{proof}

\subsection{Per-Iteration Approximation Error}
\label{sec:C4}

We now combine the majorization and Lagrangian analyses to bound the optimality gap.
\begin{theorem}[Per-Iteration Approximation Error]
\label{thm:per_iter_error}
Let $\mathbf{n}^*$ denote any global minimizer of $J(\mathbf{n})$ over $\mathcal{F}$, and use any recorded feasible trial and its error certificate from Proposition~\ref{prop:lag_gap}.
The solution $\mathbf{n}^{(t+1)}$ returned by one MM iteration satisfies:
\begin{equation}
\label{eq:per_iter_bound}
    J(\mathbf{n}^{(t+1)}) - J(\mathbf{n}^*) \;\leq\;
    \underbrace{\sum_{s \in \mathcal{S}} d_0(s)\bigl(\mathcal{B}_s(\mathbf{n}^*) - \mathcal{B}_s(\mathbf{n}^{(t)})\bigr)^2}_{\textstyle \mathcal{G}_{\mathrm{MM}}^{(t)}\;:\;\text{majorization gap}}
    \;+\;\underbrace{\lambda^{(t)}\bigl(\kappa-C(\mathbf n_{\mathrm{cand}}^{(t)})\bigr)}_{\text{Lagrangian relaxation gap}}
    \;+\;\underbrace{E_{\mathrm{Greedy}}^{(t)}}_{\text{inner greedy gap}}.
\end{equation}
The majorization gap is bounded uniformly by $\mathcal G_{\mathrm{MM}}^{(t)}\le4\beta_{\max}^2$, where $\beta_{\max}$ from Appendix~\ref{app:profile_sensitivity} bounds both annotation biases and default-model biases at zero factual counts.
The greedy error certificate includes remaining swap gains when the local search stops at its cap.
\end{theorem}

\begin{proof}
At fixed $t$, the constant difference between $\tilde J_{\mathrm{full}}$ and $\tilde J$ leaves both the constrained minimizer $\hat{\mathbf n}^{(t+1)}$ and the value difference $\Delta_{\mathrm{Lag}}^{(t)}$ unchanged.
Consequently, Lemma~\ref{lem:majorization} and the definition of $\Delta_{\mathrm{Lag}}^{(t)}$ in Proposition~\ref{prop:lag_gap} give
\begin{align*}
    J(\mathbf{n}^{(t+1)}) &\leq \tilde{J}_{\mathrm{full}}\!\bigl(\mathbf{n}^{(t+1)};\,\mathbf{n}^{(t)}\bigr) && \text{(majorization)} \\
    &= \tilde{J}_{\mathrm{full}}\!\bigl(\hat{\mathbf{n}}^{(t+1)};\,\mathbf{n}^{(t)}\bigr) + \Delta_{\mathrm{Lag}}^{(t)} && \text{(definition of $\Delta_{\mathrm{Lag}}^{(t)}$)} \\
    &\leq \tilde{J}_{\mathrm{full}}\!\bigl(\mathbf{n}^*;\,\mathbf{n}^{(t)}\bigr) + \Delta_{\mathrm{Lag}}^{(t)} && \text{(optimality of $\hat{\mathbf n}^{(t+1)}$)}.
\end{align*}
Subtracting $J(\mathbf n^*)$ and applying the gap formula~\eqref{eq:mm_gap_formula} yields
\[
    J(\mathbf n^{(t+1)})-J(\mathbf n^*)
    \le \sum_{s\in\mathcal S}d_0(s)\bigl(\mathcal B_s(\mathbf n^*)-\mathcal B_s(\mathbf n^{(t)})\bigr)^2
    +\Delta_{\mathrm{Lag}}^{(t)}.
\]
The upper bound on $\Delta_{\mathrm{Lag}}^{(t)}$ from Proposition~\ref{prop:lag_gap} then gives~\eqref{eq:per_iter_bound}.
For $N(s,a)>0$, the bias is a convex combination of zero factual bias and source biases, so $|B(s,a)|\le\beta_{\max}$.
For $N(s,a)=0$, the definition of $\beta_{\max}$ also bounds the default bias.
Because the target-policy probabilities sum to one, $|\mathcal B_s|\le\beta_{\max}$, and hence $|\mathcal B_s(\mathbf n^*)-\mathcal B_s(\mathbf n^{(t)})|\le2\beta_{\max}$.
Summing the squared differences with weights $d_0(s)$, which also sum to one, gives $\mathcal G_{\mathrm{MM}}^{(t)}\le4\beta_{\max}^2$.
\end{proof}

\subsection{Convergence Guarantees}
\label{sec:C5}

\begin{corollary}[Finite Descent and Stopping Criteria]
\label{cor:convergence}
With fixed problem inputs and deterministic tie-breaking, Algorithm~\ref{alg:main} has the following properties:
\begin{enumerate}
    \item[(i)] Every changed allocation strictly decreases $J$, and there are at most $|\mathcal F|-1$ such updates.
    \item[(ii)] If the outer iteration limit is removed and $\epsilon=0$, the algorithm reaches an unchanged allocation in finitely many iterations.
    This allocation is a fixed point of the deterministic update rule.
    A positive tolerance or finite iteration limit may terminate before a fixed point is reached.
    \item[(iii)] If an unchanged allocation $\mathbf n^{(\infty)}$ is reached at iteration $t_\star$, then
    \begin{equation}
    \label{eq:fixed_point_quality}
    \begin{aligned}
        J(\mathbf n^{(\infty)})-J(\mathbf n^*)
        &\le\sum_s d_0(s)\bigl(\mathcal B_s(\mathbf n^*)-\mathcal B_s(\mathbf n^{(\infty)})\bigr)^2\\
        &\quad+\lambda^{(t_\star)}\bigl(\kappa-C(\mathbf n_{\mathrm{cand}}^{(t_\star)})\bigr)
        +E_{\mathrm{Greedy}}^{(t_\star)},
    \end{aligned}
    \end{equation}
    using any feasible trial recorded at $t_\star$.
\end{enumerate}
\end{corollary}

\begin{proof}
The best-candidate variable starts at $\mathbf n^{(t)}$ and changes only on a strict surrogate improvement.
Thus any changed allocation satisfies $\tilde J_{\mathrm{full}}(\mathbf n^{(t+1)};\mathbf n^{(t)})<\tilde J_{\mathrm{full}}(\mathbf n^{(t)};\mathbf n^{(t)})$, and the majorization chain in Proposition~\ref{prop:mm_descent} gives strict descent of $J$.
No changed allocation can therefore revisit a previous allocation in the finite set $\mathcal F$, proving~(i).
After at most $|\mathcal F|-1$ changes, an update must leave the allocation unchanged.
With fixed inputs, search parameters, and tie-breaking, repeating the update at that allocation returns the same result, proving~(ii).
The practical stopping conditions can stop after a small strict improvement or at the iteration limit, neither of which implies an unchanged next update.
At an unchanged update, $\mathbf n^{(t_\star+1)}=\mathbf n^{(t_\star)}=\mathbf n^{(\infty)}$, so Theorem~\ref{thm:per_iter_error} gives~\eqref{eq:fixed_point_quality}, proving~(iii).
\end{proof}

\subsection{Computational Complexity}
\label{app:complexity}

\begin{proposition}[Total Computational Complexity]
\label{prop:complexity}
Let $S=|\mathcal S|$, $A=|\mathcal A|$, $K=|\mathcal K|$, $N_{\max}=\max_sN_s\ge1$, and let $T_{\mathrm{swap}}\ge1$ cap accepted swaps per local call.
Let $T_{\mathrm{MM}}$ be the number of outer iterations and $T_{\mathrm{Bis}}$ an upper bound on all subproblem calls per outer iteration, including the zero-multiplier trial, doubling, and bisection.
The total complexity of Algorithm~\ref{alg:main} is
\begin{equation}
\label{eq:total_complexity}
    \mathcal{O}\!\Bigl( T_{\mathrm{MM}}\cdot T_{\mathrm{Bis}} \cdot S\cdot A\cdot \bigl(K^{2}\,N_{\max}\,T_{\mathrm{swap}}+N_{\max}^{2}\bigr) \Bigr).
\end{equation}
For an iteration with $d$ multiplier doublings and bracket width $W>0$ before bisection, the number of subproblem calls is $2+d+\max\{0,\lceil\log_2(W/\epsilon_\lambda)\rceil\}$.
The finite-state bound in Proposition~\ref{prop:greedy_opt} also limits the actual number of accepted swaps, regardless of the chosen cap.
\end{proposition}

\begin{proof}
For a fixed pair and total count, source search accepts at most $T_{\mathrm{swap}}$ swaps and needs at most one additional scan to certify the final residual.
Each scan evaluates $K(K-1)$ swaps in $\mathcal O(K^2)$ by maintaining the sums $\sum_k n_kE_k$, $\sum_k n_kD_k$, and $\sum_k c_kn_k$; each accepted swap updates these sums in constant time.
Initialization and storing the source vector cost $\mathcal O(K)$, so the full fixed-count cost is $\mathcal O(K^2T_{\mathrm{swap}})$ for $T_{\mathrm{swap}}\ge1$.
There are at most $N_{\max}+1$ counts per pair, giving $\mathcal O(AN_{\max}K^2T_{\mathrm{swap}})$ per context when $N_{\max}\ge1$.
The action-indexed DP has $A(N_s+1)$ entries, each scanning at most $N_s+1$ feasible counts, so aggregation costs $\mathcal O(AN_{\max}^2)$ per context.
Backtracking costs $\mathcal O(AK)$, and scanning the residual tables costs $\mathcal O(AN_{\max})$, both within the preceding bounds.
Candidate objective comparisons cost $\mathcal O(SAK)$ per trial and are also dominated by the subproblem cost.
Summing over contexts, all multiplier trials, and outer iterations proves~\eqref{eq:total_complexity}.
There are two initial trials, one at zero and one at $\lambda_0$, followed by $d$ doubling trials.
Each bisection trial halves the interval width, giving the stated remaining trial count.
\end{proof}

\section{Experimental Details and Additional Results}
\label{sec:appendix_assist}

Code and reproduction instructions are available at \begin{center}
\url{https://github.com/Lygist/Budgeted-Multi-Source-Counterfactual-Annotation-for-Off-Policy-Evaluation}.
\end{center}
The accompanying \texttt{README} provides detailed instructions on data preparation, configuration, and the execution pipeline.

\paragraph{Computational resources.}
Local preprocessing, allocation, and OPE calculations used only the central processing unit (CPU) of a MacBook with an Apple M3 Pro chip; the codebase is also compatible with the standard Google Colab CPU environment.
LLM annotations were obtained through hosted application programming interfaces (APIs) from Google, OpenAI, and Anthropic.
Table~\ref{tab:benchmarks} reports full-instance solver runtimes, with a one-hour limit for each Gurobi run.
Theoretical runtime bounds are given in Proposition~\ref{prop:complexity_main} and Appendix~\ref{app:complexity}.

\subsection{Synthetic Clinical Bandit}
\label{app:syn_exp}

\subsubsection{Environment and source profiles}

We use $100$ contexts and $5$ actions.
For each context, $\pi_b(\cdot\mid s)\sim\operatorname{Dirichlet}(0.2\mathbf1)$.
The target policy retains a behavior-selected action with probability $0.9$ and otherwise replaces it by a uniformly sampled alternative.
The context distribution is sampled from a symmetric Dirichlet distribution.

Pairwise reward means are independently drawn from
\[
0.2\,\operatorname{Normal}(-2,3^2)
+
0.8\,\operatorname{SkewNormal}(\xi=5,\omega=6,\alpha=-2),
\]
and reward variances from $\operatorname{Uniform}(20,50)$.
The fitting dataset contains $M=1{,}000$ records.
Source $k$ has mean $\mu_r(s,a)+\varepsilon_k$, variance $\sigma_r^2(s,a)+\Delta_k$, and cost $c_k$.
The allocation budget is $\kappa=10{,}000$.

\subsubsection{Source Subset Comparison Under Consistent Optimization}
\label{app:syn_optimized}
To isolate the effect of source availability from allocation strategy, we evaluate all nonempty subsets $\mathcal{K}_{\mathrm{sub}}\subseteq\{1,2,3,4\}$ with MM, including the single-source subsets evaluated under random allocation in Figure~\ref{fig:source_combos}.
Figure~\ref{fig:source_combos_optimized} shows the resulting scaled MSE values.
Under consistent optimization, the qualitative conclusions of \S\ref{sec:experiments} are unchanged.

\begin{figure}[!htbp]
\centering
\includegraphics[width=\linewidth]{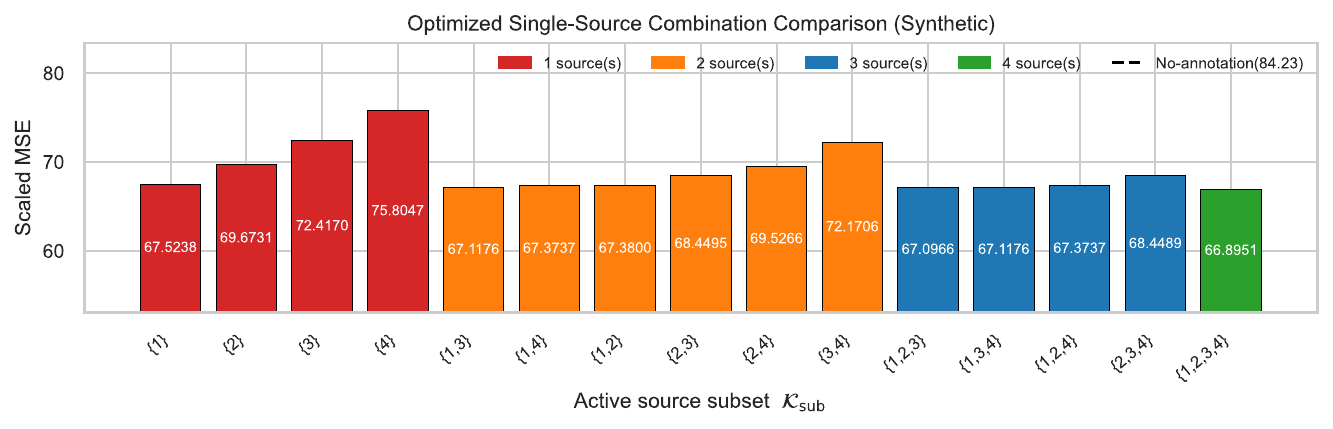}
\caption{Scaled mean squared error (MSE) $\mathcal V(\mathbf n)$ for each active source subset \emph{under the proposed majorization-minimization (MM) algorithm}.
Compared with Figure~\ref{fig:source_combos}, the single-source bars (red) are now obtained by optimization rather than uniformly random allocation; the multi-source bars are unchanged.}
\label{fig:source_combos_optimized}
\end{figure}

\subsubsection{Convergence Trace on the Synthetic Bandit}
\label{app:syn_convergence}

Figure~\ref{fig:syn_convergence} reports the convergence trace of the MM algorithm on the synthetic clinical bandit, replayed across the source subsets that achieve successively lower scaled MSE values: $\{1,3\}$, $\{1,2,3\}$, and finally $\{1,2,3,4\}$.
Scaled MSE drops sharply within the first MM iteration of each subset and stabilizes within roughly $5$--$6$ iterations, consistent with the monotone-descent property in Proposition~\ref{prop:mm_descent}.
The annotation-count tuples annotated on the trajectory show that the optimizer concentrates the budget on the lowest-bias source (Source~1) at first and progressively allocates a small share to higher-variance, lower-cost sources as additional sources become available.

\begin{figure}[!htbp]
\centering
\includegraphics[width=\linewidth]{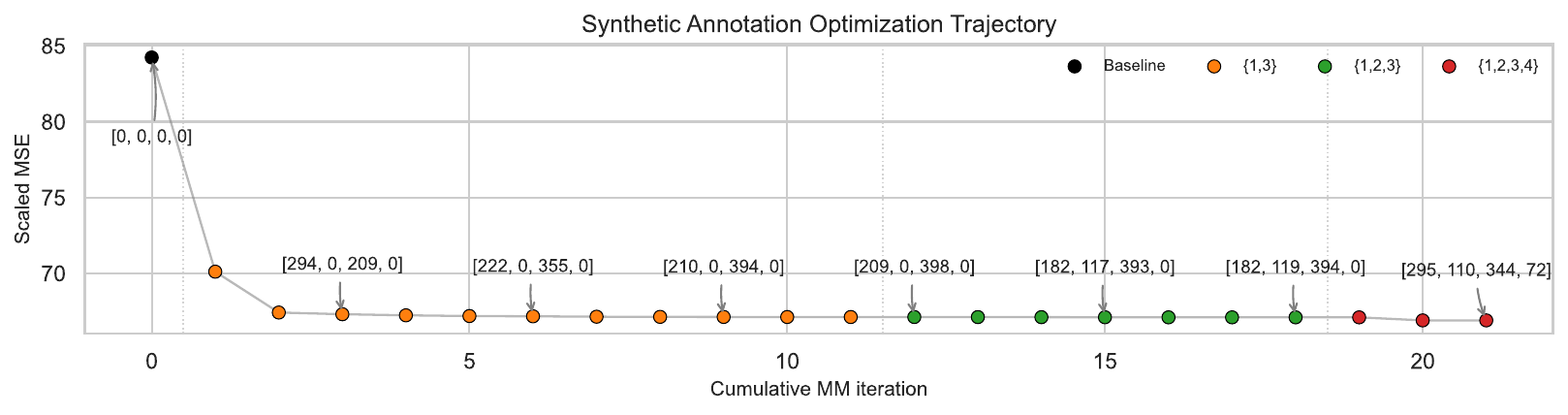}
\caption{Convergence trace of scaled mean squared error (MSE) $\mathcal V(\mathbf{n}^{(t)})$ across majorization-minimization (MM) iterations on the synthetic clinical bandit.
Different colors mark the active source subset; the tuple annotated at selected iterations records the allocation $[n_1,n_2,n_3,n_4]$ across Sources~1--4 in Table~\ref{tab:annotators}.}
\label{fig:syn_convergence}
\end{figure}

\subsection{ASSISTments Semi-Synthetic Environment}
\label{app:assist_exp}

\subsubsection{Dataset, Context Space, and Action Space}
\label{app:assist_data}
The ASSISTments 2012--13 School Data with Affect~\citep{feng2009assistments} contains $6{,}123{,}270$ student--problem interactions from the ASSISTments online middle-school mathematics tutoring platform.
Each row records: (i) student/problem identifiers, (ii) first-response latency \texttt{ms\_first\_response}, (iii) correctness of first attempt, hint count, and attempt count, (iv) problem metadata (\texttt{problem\_type}, \texttt{original} indicator), and (v) four affect-detector confidences (frustration, confusion, concentration, boredom).
We remove interactions with \texttt{ms\_first\_response} outside $[1{,}000,\; 600{,}000]$\,ms ($<1$\,s likely auto-fill, $>10$\,min likely idle), retaining $6{,}117{,}279$ rows ($99.9\%$), then sort by \texttt{(user\_id, start\_time)} so that past statistics use only strictly prior interactions.

\paragraph{Context Space ($|\mathcal{S}|=100$)}
Each interaction is embedded as a 6-dimensional vector
\[
    \phi_{\mathrm{context}} = \bigl[\mathrm{PastCorrect},\;\mathrm{PastHints},\;\mathrm{Frust},\;\mathrm{Conf},\;\mathrm{Conc},\;\mathrm{Bore}\bigr],
\]
where $\mathrm{PastCorrect}$ and $\mathrm{PastHints}$ are expanding-mean correct rate and hint count over the student's strictly prior problems (initialized to $0.5$ and $0$).
Features are standardized; K-means with $k{=}100$, $n_{\mathrm{init}}{=}10$ produces $100$ clusters.
Centroids in the original scale are passed to the LLM as the context descriptor.

\paragraph{Action Space ($|\mathcal{A}|=10$)}
We cross \texttt{problem\_type} (6 values) with \texttt{original} (2 values), then retain the $10$ combinations with at least $100$ observed interactions.
Table~\ref{tab:action_table} lists the resulting actions and their empirical statistics used both in the reward function and in the LLM prompt.

\begin{table}[!htbp]
\centering
\small
\caption{Action space of the ASSISTments environment.}
\label{tab:action_table}
\begin{tabular}{clcrrrrr}
\toprule
\textbf{ID} & \textbf{Problem Type} & \textbf{Original} & \textbf{Median RT} & \textbf{IQR (s)} & \textbf{Correct Rate} & \textbf{Count} \\
\midrule
0 & algebra        & No  & $11.7$\,s & $[5.1, 30.1]$ & $0.381$ & $128{,}221$ \\
1 & algebra        & Yes & $27.0$\,s & $[11.3, 63.3]$ & $0.678$ & $3{,}367{,}980$ \\
2 & choose\_1      & No  & $8.1$\,s  & $[4.3, 17.7]$ & $0.812$ & $127{,}222$ \\
3 & choose\_1      & Yes & $16.1$\,s & $[7.5, 39.0]$ & $0.705$ & $1{,}719{,}882$ \\
4 & choose\_n      & No  & $11.2$\,s & $[5.1, 19.4]$ & $0.600$ & $295$ \\
5 & choose\_n      & Yes & $13.3$\,s & $[7.6, 27.9]$ & $0.788$ & $11{,}300$ \\
6 & fill\_in\_1    & No  & $11.2$\,s & $[4.9, 25.8]$ & $0.335$ & $46{,}590$ \\
7 & fill\_in\_1    & Yes & $28.6$\,s & $[12.3, 66.1]$ & $0.659$ & $695{,}424$ \\
8 & open\_response & Yes & $76.9$\,s & $[27.2, 169.3]$ & $0.588$ & $17{,}629$ \\
9 & rank           & Yes & $51.1$\,s & $[25.4, 104.7]$ & $0.660$ & $2{,}726$ \\
\bottomrule
\end{tabular}
\end{table}

\subsubsection{Reward Function}
\label{sec:appendix_reward}
Following the feature-prediction spirit of APRIL~\citep{mandyam2025april}, we decompose the reward into a data-driven base and an LLM-predictable multiplier.
With $R_{\max}=30$:
\begin{equation}
    R(s, a, t) \;=\; R_{\max}\cdot \underbrace{P_{\mathrm{IRT}}(s,a)}_{\text{data-driven}} \cdot \underbrace{\mathrm{TimeEff}(t, \tilde{t}_a)}_{\text{LLM-predictable}}.
    \label{eq:appendix_reward}
\end{equation}

\paragraph{IRT component.}
For context $s$ with centroid correct rate $p_s$ and action $a$ with overall correct rate $q_a$, set $\theta_s = \mathrm{logit}(p_s)$, $b_a = \mathrm{logit}(1-q_a)$, and $P_{\mathrm{IRT}}(s,a) = \sigma(\theta_s - b_a)$, with $p_s, q_a$ clipped to $[0.02, 0.98]$.
This is a Rasch (1-PL) model depending only on observed data.
Across the $100 \times 10$ grid, $P_{\mathrm{IRT}}$ has range $[0.03, 0.98]$, mean $0.68$, std $0.22$—providing the \emph{primary} reward variation.

\paragraph{Time-efficiency component.}
Let $\tilde{t}_a$ be action $a$'s median response time in seconds.
Define
\[
    \mathrm{TimeEff}(t, \tilde{t}_a) \;=\; \max\!\left\{0.3,\; 1 - \alpha\cdot\bigl|\log_{10}(t/\tilde{t}_a)\bigr|\right\},\quad \alpha=0.4.
\]
This symmetric clamped-linear function is $1.0$ at $t=\tilde{t}_a$, decays linearly in log-space, and floors at $0.3$ for extreme mispredictions.
This constructed discount encodes a preference for response times near the action-specific median, penalizing both unusually short and unusually long responses.

\paragraph{Ground-truth reward.}
For each observed $(s,a)$ pair, the reference $\mu_r(s,a)$ is computed by plugging the empirical median $\tilde{t}_{s,a}$ into Eq.~\eqref{eq:appendix_reward}.
For the $54/1000$ pairs with fewer than $20$ observations, we fit a random forest on context centroids $+$ action one-hot features to predict $\log_{10}\tilde{t}_{s,a}$ ($5$-fold CV $R^2 = 0.89$) and fill the gap.
The final $\mu_r$ grid has range $[0.43, 28.7]$, mean $20.08$, std $10.77$.

\begin{figure}[ht]
\centering
\includegraphics[width=\linewidth]{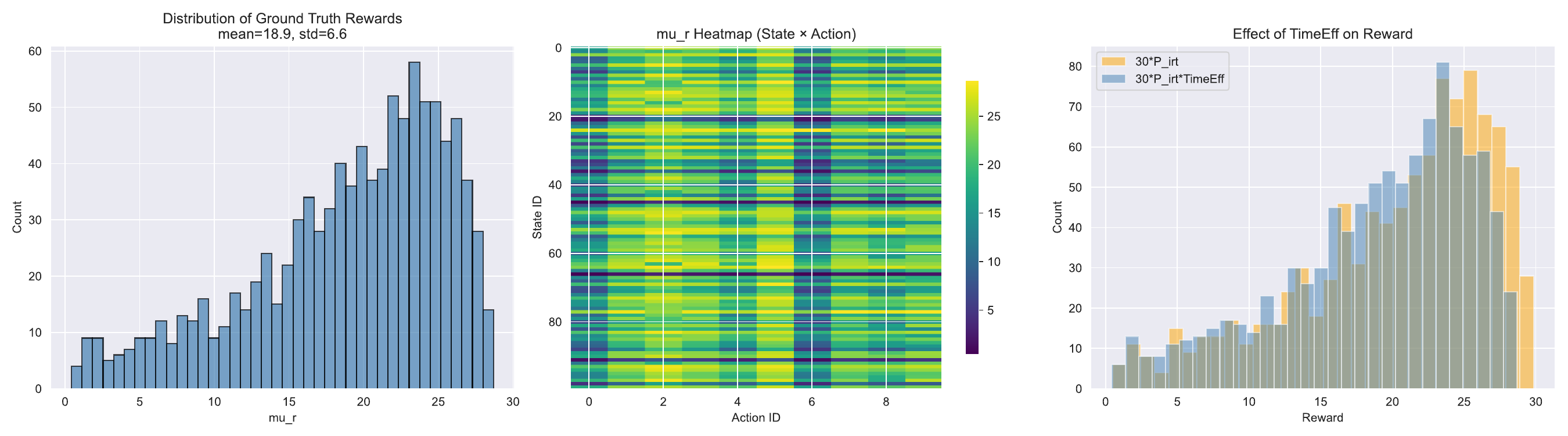}
\caption{Ground-Truth reward \textbf{Left:} distribution of Ground-Truth Rewards \textbf{Middle:} heatmap of factual context-action pair data \textbf{Right:} influence of Time-efficiency discount on the total reward}
\label{fig:assit_reward}
\end{figure}

\subsubsection{LLM Prompt Design}
\label{sec:appendix_prompt}
Each prompt is a \emph{batch} of $10$ $(s,a)$ pairs sharing a common task description, action-reference table, and domain-knowledge preamble, followed by pair-specific student profiles and $8$ bootstrap-sampled historical response times per pair.
All per-pair information (ability-level label, centroid features, action statistics, historical samples) is precomputed once from the dataframe before any API call, keeping per-batch overhead below $1$\,ms and the mean prompt length at $\sim\!3{,}500$ characters.
Figure~\ref{fig:assist_prompt} shows the template; curly-braced fields are populated from precomputed lookups.

\begin{figure}[p]
\centering
\begin{tcolorbox}[
    enhanced,
    colback=violet!4!white,
    colframe=violet!45!white,
    boxrule=0.7pt,
    arc=2.5mm,
    left=11pt, right=11pt, top=7pt, bottom=7pt,
    width=\linewidth,
    fontupper=\footnotesize,
    title={\textsc{LLM Prompt: ASSISTments Response-Time Prediction}},
    coltitle=white,
    colbacktitle=violet!55!black,
    attach boxed title to top left={yshift=-2mm, xshift=8mm},
    boxed title style={boxrule=0pt, arc=1.5mm},
]
You are an expert in educational data mining and intelligent tutoring systems.
Your task: predict the \textbf{response time in seconds} for students solving math problems on the ASSISTments platform.
``Response time'' = seconds from when a problem appears until the student's first action (submitting an answer \textsc{or} requesting a hint).

\vspace{5pt}
\textbf{\textcolor{violet!65!black}{[[ \#\# Action Reference (Problem Types) \#\# ]]}}

{\scriptsize
\begin{tabular}{@{}clcrrc@{}}
\toprule
ID & Problem Type & Orig. & Median RT & IQR (s) & CR \\
\midrule
0 & algebra         & No  & $11.7$s & $[5.1, 30.1]$  & $0.381$ \\
1 & algebra         & Yes & $27.0$s & $[11.3, 63.3]$ & $0.678$ \\
$\vdots$ & $\vdots$ & $\vdots$ & $\vdots$ & $\vdots$ & $\vdots$ \\
9 & rank            & Yes & $51.1$s & $[25.4, 104.7]$ & $0.660$ \\
\bottomrule
\end{tabular}
}

\vspace{5pt}
\textbf{\textcolor{violet!65!black}{[[ \#\# Key Factors Affecting Response Time \#\# ]]}}

\textbf{1.\ Student ability} (past correct rate): higher performers respond faster; strugglers take longer or give up quickly.

\textbf{2.\ Hint-seeking habit}: students who frequently request hints pause longer before their first action.

\textbf{3.\ Emotional context}: frustration/confusion $\to$ longer; boredom $\to$ rush or disengage; concentration $\to$ moderate, productive.

\textbf{4.\ Problem type}: multiple choice fastest; open response and rank slowest; algebra and fill-in in between.

\textbf{5.\ Original vs.\ scaffolding}: scaffolding problems typically $30$--$50\%$ faster than the original.

\textbf{6.\ Interaction effects}: struggling $\times$ hard $\to$ much longer than able $\times$ easy.

\vspace{5pt}
\textbf{\textcolor{violet!65!black}{[[ \#\# Historical Response Time Samples \#\# ]]}}

For each pair below, we provide actual response times from similar students on similar problems.
Use them to calibrate your prediction.

\vspace{5pt}
\textbf{\textcolor{violet!65!black}{[[ \#\# Student--Problem Pairs \#\# ]]}}

\vspace{2pt}
{\scriptsize\ttfamily
--- Pair 1 (context\_id: \{$s$\}, action\_id: \{$a$\}) --- \\
Student: \{ability\_level\}, correct\_rate=\{$p_s$\}, avg\_hints=\{$h_s$\}, \\
\phantom{Student: }frustration=\{$f_s$\}, confusion=\{$c_s$\}, concentration=\{$k_s$\}, boredom=\{$b_s$\} \\
Problem: \{problem\_type\}, original=\{orig\}, typical\_median=\{$\tilde{t}_a$\}s \\
Historical response times from similar cases: [\{$t_1$\}s, \{$t_2$\}s, ..., \{$t_8$\}s]
}

\vspace{2pt}
\ldots\,(9 more pairs in the same format)\,\ldots

\vspace{5pt}
\textbf{\textcolor{violet!65!black}{[[ \#\# Prediction Task \#\# ]]}}

For each pair, predict a single number: the response time in seconds.
Consider all factors and the historical samples carefully.
Output a JSON object conforming to the schema, with one entry per pair containing \texttt{context\_id}, \texttt{action\_id}, and \texttt{predicted\_response\_time\_seconds}.

\end{tcolorbox}
\caption{Prompt template used for all three large language models (LLMs) on the ASSISTments environment.
Bracketed headers \texttt{[[ \#\# \ldots \#\# ]]} follow the section-delimiter style of~\citet{mandyam2025april}; \{\texttt{curly\_braces}\} denote per-pair fields populated from precomputed pandas lookups.
Each batch contains $10$ pairs; $20$ independent queries at temperature $0.8$ are issued per $(s,a)$.}
\label{fig:assist_prompt}
\end{figure}
\clearpage

\subsubsection{LLM Annotation Allocation by MM}
We evaluate four LLM annotation sources: Gemini-2.5-Flash, GPT-4o-mini, GPT-o3-mini, and Claude-Haiku-4.5. 
For each of the $100\times 10$ context-action pairs, each model produces $20$ first-response-time predictions at temperature $0.8$, yielding $20{,}000$ predictions per model. 
We estimate each source's bias as $\hat\varepsilon_k(s,a)=\bar r_k(s,a)-\mu_r(s,a)$, where $\bar r_k(s,a)$ is the mean of the $20$ predicted rewards and $\mu_r(s,a)$ is the reference reward from the simulator. 
We estimate prediction variability by $\hat\sigma_k(s,a)$, the standard deviation of the $20$ predicted rewards.
Table~\ref{tab:llm_quality} summarizes these estimates and the API costs.
GPT-4o-mini has the lowest mean absolute bias and cost, while GPT-o3-mini costs about $78\times$ more without improving the reported quality summaries, consistent with the feature-prediction findings of APRIL~\citep{mandyam2025april}.

\begin{table}[!htbp]
\centering
\caption{LLM annotation quality on ASSISTments. Bias and variability are computed after converting predicted response times to rewards. The cost column reports the total API cost for the $20{,}000$ predictions used to estimate each model's empirical bias and variability; the allocation experiment uses the corresponding per-prediction cost.}
\label{tab:llm_quality}
\vspace{2pt}
\begin{tabular}{lcccc}
\toprule
\textbf{Source} &
$\operatorname{mean}|\hat\varepsilon_k|$ &
$\operatorname{median}|\hat\varepsilon_k|$ &
$\operatorname{mean}\hat\sigma_k$ &
\begin{tabular}[c]{@{}c@{}}Total API cost\\\end{tabular} \\
\midrule
Gemini-2.5-Flash & $1.09$ & $0.77$ & $0.53$ & \$$1.14$ \\
GPT-4o-mini      & $\mathbf{0.83}$ & $0.57$ & $0.44$ & $\mathbf{\$0.48}$ \\
GPT-o3-mini      & $0.96$ & $0.62$ & $0.55$ & \$$37.26$ \\
Claude-Haiku-4.5 & $0.86$ & $\mathbf{0.56}$ & $\mathbf{0.28}$ & \$$8.11$ \\
\bottomrule
\end{tabular}
\end{table}

We generate $6{,}000$ logged samples from the ASSISTments simulator, with pair-specific reward variances
$\sigma_r^2(s,a)\sim\mathrm{Uniform}(30,60)$. 
Based on the source diagnostics in Table~\ref{tab:llm_quality}, we run the MM-based allocation algorithm with the three lower-cost sources, excluding GPT-o3-mini, under budget $\kappa=\$0.50$.
The algorithm decides to purchase $6{,}000$ annotations and reduces the variance objective from $J(\mathbf 0)=88.61$ to $79.06$, a $10.77\%$ reduction. 

Mirroring the synthetic experiment, we evaluate every non-empty LLM source subset (excluding the high-cost GPT-o3-mini, which Table~\ref{tab:llm_quality} shows offers no quality advantage) under two protocols for the single source baseline: a uniformly random allocation and the proposed MM algorithm.

Figure~\ref{fig:llm_source_combos} reports the random-allocation baseline. Single-source subsets yield variance objectives in the range $80.79$--$83.72$, while every multi-source subset drops below~$80$.

\begin{figure}[!htbp]
\centering
\includegraphics[width=\linewidth]{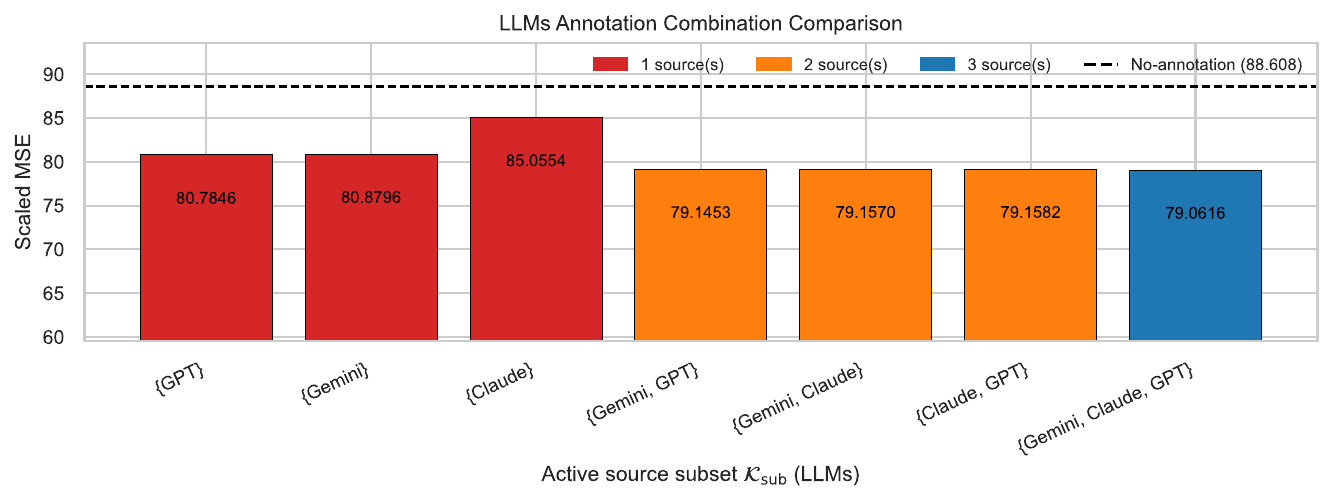}
\caption{Variance objective $J(\mathbf n)$ for each active LLM source subset under uniformly random allocation. Bars are grouped by the number of available sources; the dashed line is the no-annotation baseline $J(\mathbf 0)=88.61$.}
\label{fig:llm_source_combos}
\end{figure}

Figure~\ref{fig:llm_source_combos_optimized} repeats the comparison under consistent MM optimization. The single-source bars improve to $79.32$--$80.82$, narrowing the gap to multi-source results. The three-source configuration $\{\text{Gemini}, \text{Claude}, \text{GPT-4o-mini}\}$ remains the global optimum at $J=79.06$, a $10.77\%$ reduction over $J(\mathbf 0)=88.61$.

\begin{figure}[!htbp]
\centering
\includegraphics[width=\linewidth]{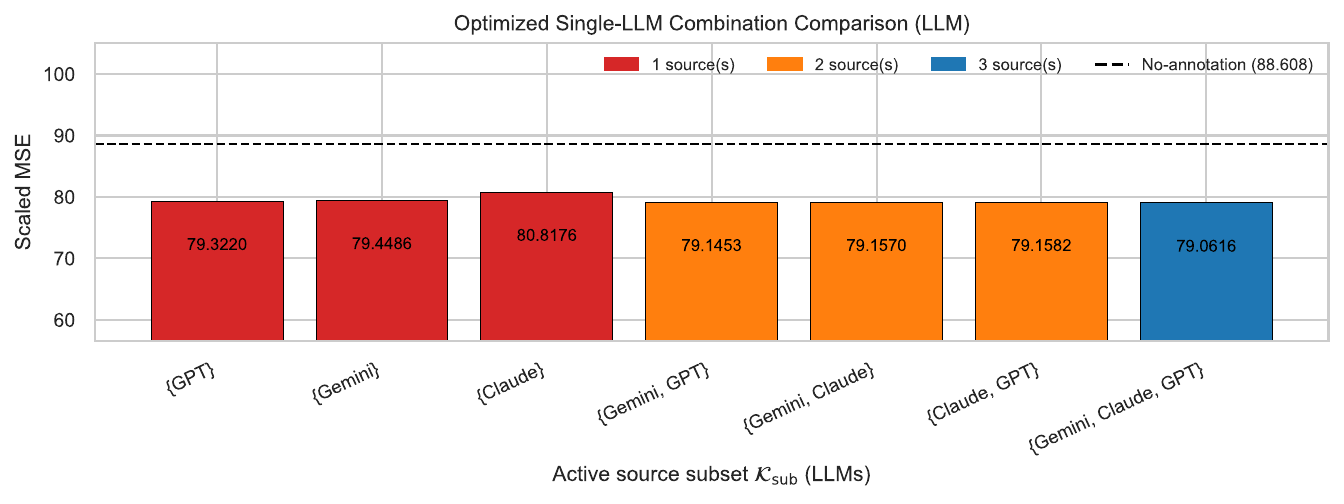}
\caption{Variance objective $J(\mathbf n)$ for each active LLM source subset under the proposed MM algorithm}
\label{fig:llm_source_combos_optimized}
\end{figure}

To illustrate the dynamics of the MM algorithm under realistic LLM annotation profiles, Figure~\ref{fig:llm_convergence} shows the convergence trace on ASSISTments. The objective decreases monotonically from $88.61$ to $79.06$.

\begin{figure}[!htbp]
\centering
\includegraphics[width=\linewidth]{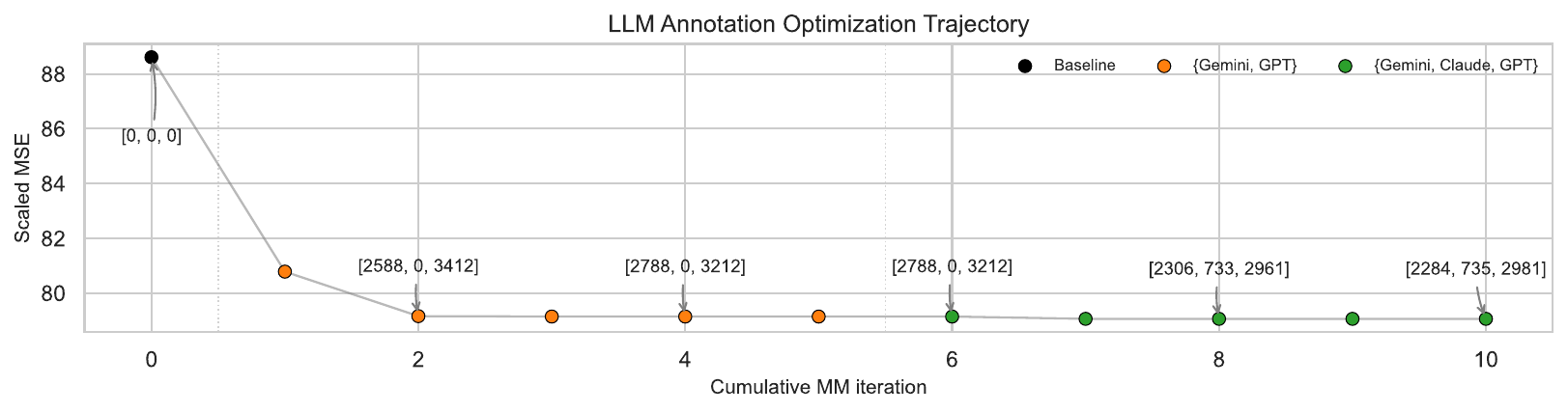}
\caption{Convergence trace of $J(\mathbf{n}^{(t)})$ across MM iterations on ASSISTments. The tuple at each highlighted iteration denotes the optimal allocation across [\texttt{Gemini-2.5-Flash}, \texttt{Claude-Haiku-4.5}, \texttt{GPT-4o-mini}]. The algorithm heavily favors \texttt{GPT-4o-mini} due to its low empirical bias and high cost-efficiency, while utilizing \texttt{Claude-Haiku-4.5} more conservatively because of its higher API cost.}
\label{fig:llm_convergence}
\end{figure}

\subsection{Exact-Solver Benchmarks and Runtime}
\label{app:solver_benchmarks}

To complement the certified small-instance optima in Table~\ref{tab:opt gap}, we warm-started Gurobi on the full Synthetic and ASSISTments instances with the best heuristic solution and imposed a one-hour time limit.
The runs did not certify global optimality within that limit, so we compare the methods against the valid lower bounds available at termination.
Let \(L\) be the scaled-MSE lower bound available at termination and \(z_A=\mathcal V(\mathbf n_A)\) the scaled MSE returned by method \(A\).
The table reports \(100(z_A-L)/z_A\).
The heuristic runtimes exclude the Gurobi computation used to obtain \(L\).

\begin{table}[!htbp]
\begin{center}
\caption{Solver benchmarks and runtimes for the Synthetic and ASSISTments environments.}
\label{tab:benchmarks}
\begin{tabular}{l|lrr}
\hline
Dataset & Method & Gap to Gurobi lower bound & Runtime (s)\\
\hline
Synthetic & Greedy (max-gain) & 0.6863\% & 0.0314\\
Synthetic & Greedy (cost-aware) & 3.9843\% & 0.0562\\
Synthetic & MM & 1.0474\% & 8.2390\\
Synthetic & MM (greedy init.) & 0.2807\% & 5.2657\\
Synthetic & Gurobi incumbent & 0.1258\% & 3600.6957\\
\hline
ASSISTments & Greedy (max-gain) & 0.8216\% & 1.0222\\
ASSISTments & Greedy (cost-aware) & 0.1769\% & 3.5662\\
ASSISTments & MM & 0.0050\% & 857.6975\\
ASSISTments & MM (greedy init.) & 0.0050\% & 521.8667\\
ASSISTments & Gurobi incumbent & 0.0047\% & 3604.0060\\
\hline
\end{tabular}
\end{center}
\end{table}

Table~\ref{tab:opt gap} shows greedy-initialized MM has a \(0.05\%\) mean gap and a \(0.26\%\) maximum gap to the certified optima on the small instances.
On the full instances, Table~\ref{tab:benchmarks} shows that it reaches within \(0.2807\%\) and \(0.0050\%\) of the Gurobi lower bounds in \(5.3\) and \(522\) seconds, respectively.
These comparisons show small empirical gaps to the available exact and lower-bound benchmarks.

\subsection{End-to-End Evaluation with Finite Pilots}
\label{app:end_to_end}

The solver benchmarks assess optimization quality for fixed profiles; we next evaluate final OPE performance when allocations are selected using finite-pilot estimates.
For each of $30$ pilot replications, factual moments are estimated from $20$ pilot rewards per pair and annotation moments from $20$ pilot draws per source--pair.
The pilot is separate from subsequent fitting rewards, acquired annotation realizations, and OPE outcomes.
Fitting-data context--action counts determine the capacities.

Each allocator receives the same pilot estimates and constraints.
For the optimized single-source baseline, we optimize each source separately and use the pilot objective to select the predicted best source.
Thus, this baseline does not select its source using final OPE outcomes.

After fixing an allocation, we generated \(50\) independent reward and annotation realizations from the distributions used in each environment and evaluated DM\(^{+}\)-IS on independent OPE datasets (\(N=1{,}000\) for Synthetic and \(N=6{,}000\) for ASSISTments).
In ASSISTments, the annotation realizations are drawn from empirical source models fitted to the observed LLM predictions.
Within each pilot replication, empirical MSE is the average squared deviation of the \(50\) final value estimates from the environment's true policy value.
Table~\ref{tab:end-to-end} reports the mean across pilot replications of the relative MSE reduction versus the corresponding no-annotation baseline.

\begin{table}[!htbp]
\begin{center}
\caption{Mean relative reduction in empirical mean squared error (MSE) over $30$ pilot replications.
For each pilot, MSE is estimated from $50$ independent final off-policy value estimates.}
\label{tab:end-to-end}
\begin{tabular}{lcc}
\toprule
Method & Synthetic & ASSISTments\\
\midrule
Best optimized single source & 17.86\% & 8.30\%\\
Greedy (max-gain)             & 18.43\% & 7.86\%\\
Greedy (cost-aware)           & 14.39\% & 8.46\%\\
MM                            & 18.67\% & 8.46\%\\
MM (greedy init.)             & \textbf{19.13\%} & \textbf{8.47\%}\\
\bottomrule
\end{tabular}
\end{center}
\end{table}

Table~\ref{tab:end-to-end} shows greedy-initialized MM reduces empirical MSE by \(19.13\%\) on Synthetic and \(8.47\%\) on ASSISTments relative to no annotation.
Thus, the allocation gains persist in these experiments after finite-sample profile estimation and evaluation of the final OPE estimator.

\subsection{Within-pair bootstrap}
\label{app:bootsrap}

As a complementary check, we assess sensitivity to resampling the LLM outputs used to estimate source profiles.
For each of $30$ bootstrap replications, we resample the $20$ observed predictions within each source--pair, re-estimate its moments, and rerun MM.
We evaluate the resulting allocation using reference profiles estimated from all $20$ original predictions per source--pair.

The average reference-profile MSE reduction is $10.765\%$, compared with $10.774\%$ using the original profiles.
The standard deviation is $0.0038$ percentage points, and the mean difference is $0.009$ percentage points.
This measures conditional sensitivity to resampling the observed outputs.
Because evaluation uses the same empirical source reference, it is not held-out validation of the population LLM distribution and does not measure uncertainty in the simulator's factual reward reference.

\section{Broader Impacts and Asset Use}

\subsection{Broader Impacts}
\label{app:broader_impacts}

Budgeted counterfactual annotation can improve the efficiency of policy evaluation before deployment.
In healthcare and education, unreliable annotations or violations of the evaluation assumptions may nevertheless yield inaccurate policy comparisons with consequences for treatment or instruction.
Application-specific validation, including assessment of subgroup outcomes, is therefore important before deployment.

\subsection{Existing Data and Model Services}
\label{app:asset_terms}

The ASSISTments 2012--13 School Data with Affect~\citep{feng2009assistments,wang2015affect} are provided by Neil T. Heffernan and the ASSISTments team at Worcester Polytechnic Institute through the \href{https://sites.google.com/site/assistmentsdata/datasets/2012-13-school-data-with-affect}{official dataset release}.
The \href{https://sites.google.com/site/assistmentsdata/}{ASSISTments data-use conditions} require users to report and delete any inadvertently discovered personally identifiable information and prohibit using that information.
Our experiments use existing records, with no participant recruitment, crowdsourcing, or interventions.

The annotation sources are Google's Gemini-2.5-Flash, OpenAI's GPT-4o-mini, and Anthropic's Claude-Haiku-4.5.
Hosted access is subject to the \href{https://developers.google.com/terms}{Google APIs Terms of Service} and \href{https://ai.google.dev/gemini-api/terms}{Gemini API Additional Terms of Service}, the \href{https://openai.com/policies/services-agreement/}{OpenAI Services Agreement}, and the \href{https://www.anthropic.com/legal/commercial-terms}{Anthropic Commercial Terms of Service}, respectively.

\newpage
\input{checklist.tex}

\end{document}

%% file: checklist.tex
\section*{NeurIPS Paper Checklist}

\begin{enumerate}

\item {\bf Claims}
    \item[] Question: Do the main claims made in the abstract and introduction accurately reflect the paper's contributions and scope?
    \item[] Answer: \answerYes{} 
    \item[] Justification: The claims in the abstract and introduction match the theoretical results proved in Appendices~\ref{sec:appendix_derivations}--\ref{sec:appendix_mm} and the experimental results in Section~\ref{sec:experiments}.

\item {\bf Limitations}
    \item[] Question: Does the paper discuss the limitations of the work performed by the authors?
    \item[] Answer: \answerYes{} 
    \item[] Justification: The Conclusion discusses the contextual-bandit scope, reliance on known or estimated source profiles, and semi-synthetic ASSISTments environment.
    Sensitivity to estimated profiles is analyzed in Appendix~\ref{app:profile_sensitivity} and evaluated in Section~\ref{sec:exp_robustness}.

\item {\bf Theory assumptions and proofs}
    \item[] Question: For each theoretical result, does the paper provide the full set of assumptions and a complete (and correct) proof?
    \item[] Answer: \answerYes{} 
    \item[] Justification: Assumptions accompany the theoretical statements, and their derivations and proofs are provided in Appendices~\ref{sec:appendix_derivations}--\ref{sec:appendix_mm}.

    \item {\bf Experimental result reproducibility}
    \item[] Question: Does the paper fully disclose all the information needed to reproduce the main experimental results of the paper to the extent that it affects the main claims and/or conclusions of the paper (regardless of whether the code and data are provided or not)?
    \item[] Answer: \answerYes{} 
    \item[] Justification: Section~\ref{sec:experiments} and Appendix~\ref{sec:appendix_assist} describe the environments, reward functions, large language model (LLM) prompts, and evaluation protocols; code and reproduction instructions are available at the GitHub repository link in Appendix~\ref{sec:appendix_assist}.
    \item[] Guidelines:
    \begin{itemize}
        \item The answer \answerNA{} means that the paper does not include experiments.
        \item If the paper includes experiments, a \answerNo{} answer to this question will not be perceived well by the reviewers: Making the paper reproducible is important, regardless of whether the code and data are provided or not.
        \item If the contribution is a dataset and\slash or model, the authors should describe the steps taken to make their results reproducible or verifiable.
        \item Depending on the contribution, reproducibility can be accomplished in various ways.
        For example, if the contribution is a novel architecture, describing the architecture fully might suffice, or if the contribution is a specific model and empirical evaluation, it may be necessary to either make it possible for others to replicate the model with the same dataset, or provide access to the model.
        In general, releasing code and data is often one good way to accomplish this, but reproducibility can also be provided via detailed instructions for how to replicate the results, access to a hosted model (e.g., in the case of a large language model), releasing of a model checkpoint, or other means that are appropriate to the research performed.
        \item While NeurIPS does not require releasing code, the conference does require all submissions to provide some reasonable avenue for reproducibility, which may depend on the nature of the contribution.
        For example
        \begin{enumerate}
            \item If the contribution is primarily a new algorithm, the paper should make it clear how to reproduce that algorithm.
            \item If the contribution is primarily a new model architecture, the paper should describe the architecture clearly and fully.
            \item If the contribution is a new model (e.g., a large language model), then there should either be a way to access this model for reproducing the results or a way to reproduce the model (e.g., with an open-source dataset or instructions for how to construct the dataset).
            \item We recognize that reproducibility may be tricky in some cases, in which case authors are welcome to describe the particular way they provide for reproducibility.
            In the case of closed-source models, it may be that access to the model is limited in some way (e.g., to registered users), but it should be possible for other researchers to have some path to reproducing or verifying the results.
        \end{enumerate}
    \end{itemize}

\item {\bf Open access to data and code}
    \item[] Question: Does the paper provide open access to the data and code, with sufficient instructions to faithfully reproduce the main experimental results, as described in supplemental material?
    \item[] Answer: \answerYes{} 
    \item[] Justification: We provide a GitHub repository containing all code, public data sources, data preprocessing scripts, and a detailed README to reproduce the experiments. See details in Appendix~\ref{sec:appendix_assist}.

\item {\bf Experimental setting/details}
    \item[] Question: Does the paper specify all the training and test details (e.g., data splits, hyperparameters, how they were chosen, type of optimizer) necessary to understand the results?
    \item[] Answer: \answerYes{} 
    \item[] Justification: Section~\ref{sec:experiments} and Appendix~\ref{sec:appendix_assist} describe the synthetic and semi-synthetic environments, preprocessing, source-profile estimation, and separation of pilot, fitting, and evaluation samples.

\item {\bf Experiment statistical significance}
    \item[] Question: Does the paper report error bars suitably and correctly defined or other appropriate information about the statistical significance of the experiments?
    \item[] Answer: \answerYes{} 
    \item[] Justification: We report the mean and standard deviation across 100 independent environment initializations(e.g., shaded bands in Figure~\ref{fig:four_regimes}). We also provide bootstraps experiment to test the robustness of the algorithm in Section~\ref{sec:exp_robustness}.

\item {\bf Experiments compute resources}
    \item[] Question: For each experiment, does the paper provide sufficient information on the computer resources (type of compute workers, memory, time of execution) needed to reproduce the experiments?
    \item[] Answer: \answerYes{} 
    \item[] Justification: Appendix~\ref{sec:appendix_assist} identifies the local hardware and hosted LLM services, and Table~\ref{tab:benchmarks} reports full-instance solver runtimes.
    Memory, runtimes for the remaining experiments, and total project compute are not reported.
    
\item {\bf Code of ethics}
    \item[] Question: Does the research conducted in the paper conform, in every respect, with the NeurIPS Code of Ethics \url{https://neurips.cc/public/EthicsGuidelines}?
    \item[] Answer: \answerYes{} 
    \item[] Justification: The research conforms to the NeurIPS Code of Ethics.

\item {\bf Broader impacts}
    \item[] Question: Does the paper discuss both potential positive societal impacts and negative societal impacts of the work performed?
    \item[] Answer: \answerYes{} 
    \item[] Justification: Appendix~\ref{app:broader_impacts} discusses the benefit of more efficient policy evaluation, potential consequences of inaccurate policy comparisons, and the need for application-specific validation.
    
\item {\bf Safeguards}
    \item[] Question: Does the paper describe safeguards that have been put in place for responsible release of data or models that have a high risk for misuse (e.g., pre-trained language models, image generators, or scraped datasets)?
    \item[] Answer: \answerNA{} 
    \item[] Justification: The contribution is an allocation method using existing data and hosted LLMs; we do not release a pretrained generative model or a newly collected student dataset.

\item {\bf Licenses for existing assets}
    \item[] Question: Are the creators or original owners of assets (e.g., code, data, models), used in the paper, properly credited and are the license and terms of use explicitly mentioned and properly respected?
    \item[] Answer: \answerYes{} 
    \item[] Justification: Appendix~\ref{app:asset_terms} identifies the ASSISTments release and creators, links its data-use conditions, and identifies the LLM providers and applicable service terms.

\item {\bf New assets}
    \item[] Question: Are new assets introduced in the paper well documented and is the documentation provided alongside the assets?
    \item[] Answer: \answerYes{} 
    \item[] Justification: The code and detailed instructions are in our repository at [repository URL].

\item {\bf Crowdsourcing and research with human subjects}
    \item[] Question: For crowdsourcing experiments and research with human subjects, does the paper include the full text of instructions given to participants and screenshots, if applicable, as well as details about compensation (if any)?
    \item[] Answer: \answerNA{} 
    \item[] Justification: We use synthetic data and secondary analysis of existing ASSISTments records, with no participant recruitment, crowdsourcing, or interventions (Appendix~\ref{app:asset_terms}).

\item {\bf Institutional review board (IRB) approvals or equivalent for research with human subjects}
    \item[] Question: Does the paper describe potential risks incurred by study participants, whether such risks were disclosed to the subjects, and whether Institutional Review Board (IRB) approvals (or an equivalent approval/review based on the requirements of your country or institution) were obtained?
    \item[] Answer: \answerNA{} 
    \item[] Justification: The experiments use synthetic data and secondary analysis of the public ASSISTments release, with no new participant recruitment or interventions (Appendix~\ref{app:asset_terms}).

\item {\bf Declaration of LLM usage}
    \item[] Question: Does the paper describe the usage of LLMs if it is an important, original, or non-standard component of the core methods in this research?
    Note that if the LLM is used only for writing, editing, or formatting purposes and does \emph{not} impact the core methodology, scientific rigor, or originality of the research, declaration is not required.
    \item[] Answer: \answerYes{} 
    \item[] Justification: Section~\ref{sec:assist_experiment} and Appendix~\ref{app:assist_exp} describe the LLM annotation sources, response-time prediction task, prompts, and conversion of predictions to rewards.

\end{enumerate}